\documentclass[twoside,11pt]{article}
\usepackage{amsmath}
\usepackage{thmtools}
\PassOptionsToPackage{breaklinks}{hyperref}
\usepackage[preprint,hyperref]{jmlr2e}

\renewcommand{\S}{\mathcal{S}}
\newcommand{\B}[1]{\boldsymbol{{#1}}}
\newcommand{\tuple}[1]{\langle{#1}\rangle}
\renewcommand{\P}{\mathbb{P}}

\newcommand{\M}{M}
\newcommand{\X}{\mathcal{X}}
\newcommand{\E}{\mathcal{E}}

\newcommand{\Y}{\mathcal{Y}}
\newcommand{\G}{\mathcal{G}}
\newcommand{\R}{\mathbb{R}}
\newcommand{\BG}{\mathcal{B}}
\newcommand{\V}{\mathcal{V}}

\newcommand{\pa}[2]{\mathrm{pa}_{{#1}}({#2})}
\newcommand{\adj}[2]{\mathrm{adj}_{#1}(#2)}
\newcommand{\an}[2]{\mathrm{an}_{{#1}}({#2})}
\newcommand{\de}[2]{\mathrm{de}_{{#1}}({#2})}
\newcommand{\ch}[2]{\mathrm{ch}_{{#1}}({#2})}
\newcommand{\scc}[2]{\mathrm{sc}_{{#1}}({#2})}
\newcommand{\sigsep}[1]{\overset{\sigma}{\underset{{#1}}{\perp}}}
\newcommand{\dsep}[1]{\overset{d}{\underset{{#1}}{\perp}}}
\newcommand{\indep}{\mathrel{\perp\mspace{-10mu}\perp}}

\newcommand{\ind}[1]{\underset{{#1}}{\indep}}
\newcommand{\given}{\,|\,}

\usepackage{float}

\usepackage{caption}

\usepackage{subcaption}

\usepackage{color}

\usepackage[table]{xcolor}

\usepackage{tabularx}

\usepackage{booktabs}

\usepackage[ruled,vlined]{algorithm2e}

\usepackage{tikz}

\usepackage{tkz-graph}

\usepackage{pgfplots}

\usepackage{pgffor}

\usetikzlibrary{arrows,arrows.meta,fit,automata,matrix,calc,patterns,decorations,decorations.markings,shapes,positioning}

\ShortHeadings{Conditional independences and causal relations implied by sets of equations}{Blom, van Diepen, and Mooij}

\firstpageno{1}

\begin{document}

\title{Conditional independences and causal relations implied by sets of equations}

\author{\name Tineke Blom \email t.blom2@uva.nl \\
       \addr Informatics Institute\\
       University of Amsterdam\\
       P.O. Box 19268, 1000 GG Amsterdam, The Netherlands
       \AND
       \name Mirthe M. van Diepen \email mvdiepen@cs.ru.nl \\
       \addr Institute for Computer and Information Science\\
       Radboud University Nijmegen\\
       PO Box 9102, 6500 HC Nijmegen, The Netherlands
   	   \AND
   	   \name Joris M. Mooij \email j.m.mooij@uva.nl \\
   	   \addr Korteweg-De Vries Institute for Mathematics\\
   	   University of Amsterdam\\
   	   P.O. Box 19268, 1000 GG Amsterdam, The Netherlands
	   }

\editor{Viktor Chernozhukov}

\maketitle

\begin{abstract}
Real-world complex systems are often modelled by sets of equations with endogenous and exogenous variables. What can we say about the causal and probabilistic aspects of variables that appear in these equations without explicitly solving the equations? We make use of Simon's causal ordering algorithm \citep{Simon1953} to construct a \emph{causal ordering graph} and prove that it expresses the effects of soft and perfect interventions on the equations under certain unique solvability assumptions. We further construct a \emph{Markov ordering graph} and prove that it encodes conditional independences in the distribution implied by the equations with independent random exogenous variables, under a similar unique solvability assumption. We discuss how this approach reveals and addresses some of the limitations of existing causal modelling frameworks, such as causal Bayesian networks and structural causal models.
\end{abstract}


\section{Introduction}
\label{sec:introduction}

The discovery of causal relations is a fundamental objective in many scientific endeavours. The process of the scientific method usually involves a hypothesis, such as a causal graph or a set of equations, that explains observed phenomena. In practice, such a graph structure can be learned automatically from conditional independences in observational data via causal discovery algorithms, e.g.\ the well-known PC and FCI algorithms \citep{Spirtes2000, Zhang2008}. The crucial assumption in \emph{causal discovery} is that directed edges in this learned graph express causal relations between variables. However, an immediate concern is whether directed mixed graphs actually can simultaneously encode the causal semantics and the conditional independence constraints of a system.\footnote{See, for example, \citep{Dawid2010} and references therein for a discussion.} We explicitly define soft and perfect interventions on sets of equations and demonstrate that, for some models, a single graph expressing conditional independences between variables via d-separations does not seem to represent the effects of these interventions in an unambiguous way, while graphs that also have vertices representing equations do encode both the dependence and causal structure implied by these models. In particular, we show that the output of the PC algorithm does not have a straightforward causal interpretation when it is applied to data generated by a simple dynamical model with feedback at equilibrium.

It is often said that the ``gold standard'' in causal discovery is controlled experimentation. Indeed, the main principle of the scientific method is to derive predictions from a hypothesis, such as a causal graph or set of equations, that are then verified or rejected through experimentation. We show how, in practice, testable predictions can be derived automatically from sets of equations via the \emph{causal ordering algorithm}, introduced by \citet{Simon1953}. We adapt and extend the algorithm to construct a \emph{directed cluster graph} that we call the \emph{causal ordering graph}. From this, we can construct a directed graph that we call the \emph{Markov ordering graph}. We prove that, under a certain unique solvability assumption, the latter implies conditional independences between variables which can be tested in observational data and the former represents the effects of soft and certain perfect interventions which can be verified through experimentation. We believe that the technique of causal ordering is a useful and scalable tool in our search for and understanding of causal relations.

In this work, we also shed new light on differences between the causal ordering graph and the graph associated with a Structural Causal Model (SCM) (see \citet{Pearl2000, Bongers2020}), which are also commonly referred to as Structural Equation Models (SEMs).\footnote{The latter term has been used by econometricians since the 1950s. Note that, in the past some econometricians have used (cyclic/non-recursive) ``structural models'' without requiring that there is a specified one-to-one correspondence between endogenous variables and equations; see e.g.\ \citet{Basmann1963}. Recent usage is consistent with the implication that there is a specified variable on the left-hand side for each equation as is common in the SCM framework.} Specifically, we demonstrate that the two graphical representations may model different sets of interventions. Furthermore, we show that a stronger Markov property can sometimes be obtained by applying causal ordering to the structural equations of an SCM. By explicitly defining interventions and by distinguishing between the Markov ordering graph and the causal ordering graph we gain new insights about the correct interpretation of results in \citet{Iwasaki1994, Dash2005}. Throughout this work, we discuss an example in \citet{Iwasaki1994} to illustrate our ideas. Here, we use it to highlight the contributions of this paper and to provide an overview of its central concepts.

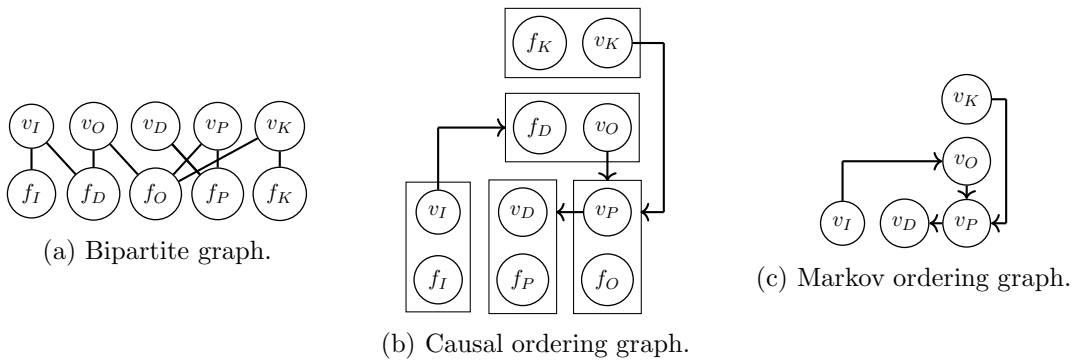
\begin{figure}[th]
\begin{subfigure}{0.33\textwidth}
\centering
\begin{tikzpicture}[scale=0.75,every node/.style={transform shape}] 
\GraphInit[vstyle=Normal]
\SetGraphUnit{1.1}
\SetVertexMath
\Vertex{v_I}
\EA(v_I){v_O}
\EA(v_O){v_D}
\EA(v_D){v_P}
\EA(v_P){v_K}
\SO[unit=1.2](v_I){f_I}
\SO[unit=1.2](v_O){f_D}
\SO[unit=1.2](v_D){f_O}
\SO[unit=1.2](v_P){f_P}
\SO[unit=1.2](v_K){f_K}
\tikzset{EdgeStyle/.style = {-}}
\Edge(f_I)(v_I)
\Edge(f_D)(v_I)
\Edge(f_D)(v_O)
\Edge(f_O)(v_K)
\Edge(f_O)(v_P)
\Edge(f_O)(v_O)
\Edge(f_P)(v_D)
\Edge(f_P)(v_P)
\Edge(f_P)(v_D)
\Edge(f_K)(v_K)
\end{tikzpicture}
\subcaption{Bipartite graph.}
\label{fig:bathtub bipartite graph}
\end{subfigure}%
\begin{subfigure}{0.33\textwidth}
\centering
\begin{tikzpicture}[scale=0.75,every node/.style={transform shape}] 
\GraphInit[vstyle=Normal]
\SetGraphUnit{1.2}
\SetVertexMath
\Vertex{v_I}
\EA[unit=1.5](v_I){v_D}
\EA[unit=1.5](v_D){v_P}
\SO(v_I){f_I}
\SO(v_D){f_P}
\SO(v_P){f_O}
\NO[unit=1.5](v_P){v_O}
\WE(v_O){f_D}
\NO[unit=1.5](v_O){v_K}
\WE(v_K){f_K}
\node[draw=black, fit=(v_I) (f_I), inner sep=0.1cm ]{};
\node[draw=black, fit=(v_D) (f_P), inner sep=0.1cm ]{};
\node[draw=black, fit=(v_P) (f_O), inner sep=0.1cm ]{};
\node[draw=black, fit=(v_O) (f_D), inner sep=0.1cm ]{};
\node[draw=black, fit=(v_K) (f_K), inner sep=0.1cm ]{};
\draw[EdgeStyle, style={-}](v_I) to (0,1.5);
\draw[EdgeStyle, style={->}](0,1.5) to (1.2,1.5);
\draw[EdgeStyle, style={->}](v_O) to (3,0.55);
\draw[EdgeStyle, style={->}](v_P) to (2.085,0);
\draw[EdgeStyle, style={-}](v_K) to (4,3);
\draw[EdgeStyle, style={-}](4,3) to (4,0);
\draw[EdgeStyle, style={->}](4,0) to (3.585,0);
\end{tikzpicture}
\subcaption{Causal ordering graph.}
\label{fig:bathtub directed cluster graph}
\end{subfigure}%
\begin{subfigure}{0.33\textwidth}
\centering
\begin{tikzpicture}[scale=0.75,every node/.style={transform shape}] 
\GraphInit[vstyle=Normal]
\SetGraphUnit{1.1}
\SetVertexMath
\Vertex{v_I}
\EA[unit=1.1](v_I){v_D}
\EA[unit=1.1](v_D){v_P}
\NO[unit=1.1](v_P){v_O}
\NO[unit=1.1](v_O){v_K}
\draw[EdgeStyle, style={-}](v_I) to (0,1.1);
\draw[EdgeStyle, style={->}](0,1.1) to (v_O);
\draw[EdgeStyle, style={->}](v_O) to (v_P);
\draw[EdgeStyle, style={->}](v_P) to (v_D);
\draw[EdgeStyle, style={-}](v_K) to (2.9,2.2);
\draw[EdgeStyle, style={-}](2.9,2.2) to (2.9,0);
\draw[EdgeStyle, style={->}](2.9,0) to (v_P);
\end{tikzpicture}
\subcaption{Markov ordering graph.}
\label{fig:bathtub markov graph}
\end{subfigure}
\caption{Three graphical representations for the bathtub system in equilibrium. The bipartite graph in Figure \ref{fig:bathtub bipartite graph} is a representation of the structure of equations $F=\{f_K,f_I,f_P,f_O,f_D\}$ where the vertices $V=\{v_K,v_I,v_P,v_O,v_D\}$ correspond to endogenous variables and there is an edge $(v-f)$ if and only if the variable $v$ appears in equation $f$. The outcome of the causal ordering algorithm is the directed cluster graph in Figure \ref{fig:bathtub directed cluster graph}, in which rectangles represent a partition of the variable and equation vertices into clusters. The corresponding Markov ordering graph for the variable vertices is given in Figure \ref{fig:bathtub markov graph}.}
\label{fig:bathtub}
\end{figure}

\begin{example}
	\label{ex:bathtub intro}
	Let us revisit a physical model of a filling bathtub in equilibrium that is presented in \citet{Iwasaki1994}. Consider a system where water flows from a faucet into a bathtub at a constant rate $X_{v_I}$ and it flows out of the tub through a drain with diameter $X_{v_K}$. An ensemble of such bathtubs that have faucets and drains with different (unknown) rates and diameters can be modelled by the equations $f_K$ and $f_I$ below:
	\begin{alignat}{3}
	\label{eq:dyn K}
	f_{K}:&\qquad X_{v_K} = U_{w_K},\\
	f_{I}:&\qquad X_{v_I} = U_{w_I},
	\intertext{where $U_{w_K}$ and $U_{w_I}$ are independent random variables both taking value in $\R_{>0}$. When the faucet is turned on the water level $X_{v_D}$ in the bathtub increases as long as the inflow $X_{v_I}$ of the water exceeds the outflow $X_{v_O}$ of water. The differential equation $\dot{X}_{v_D}(t) = U_{w_1}(X_{v_I}(t)-X_{v_O}(t))$ defines the mechanism for the rate of change in $X_{v_D}(t)$, where $U_{w_1}$ is a constant or a random variable taking value in $\R_{>0}$. At equilibrium the rate of change is equal to zero, resulting in the equilibrium equation}
	f_{D}:&\qquad U_{w_1}(X_{v_I}-X_{v_O}) = 0.
	\intertext{As the water level $X_{v_D}$ increases, the pressure $X_{v_P}$ that is exerted by the water increases as well. The mechanism for the change in pressure is defined by the differential equation $\dot{X}_{v_P}(t) = U_{w_2}(g\,U_{w_3}X_{v_D}(t)-X_{v_P}(t))$, where $g$ is the gravitational constant and $U_{w_2},U_{w_3}$ are constants or random variables both taking value in $\R_{>0}$. After equilibration, we obtain}
	f_{P}:&\qquad U_{w_2}(g\,U_{w_3}X_{v_D}-X_{v_P}) =0.
	\intertext{ The higher the pressure $X_{v_P}$ or the bigger the size of the drain $X_{v_K}$, the faster the water flows through the drain. The differential equation $\dot{X}_{v_O}(t) = U_{w_4}(U_{w_5}X_{v_K}X_{v_P}(t)-X_{v_O}(t))$ models the outflow rate of the water, where $U_{w_4},U_{w_5}$ are constants or random variables both taking value in $\R_{>0}$. The equilibrium equation $f_O$ is given by}
	\label{eq:dyn O}
	f_{O}:&\qquad U_{w_4}(U_{w_5}X_{v_K}X_{v_P}-X_{v_O}) = 0.
	\end{alignat}
  	We will study the conditional independences that are implied by equilibrium equations \eqref{eq:dyn K} to \eqref{eq:dyn O}. In Sections \ref{sec:effects soft interventions} and \ref{sec:effects perfect interventions} we will define the notion of soft and perfect interventions on sets of equations as a generalization of soft and perfect interventions on SCMs. The causal properties of sets of equilibrium equations are examined by comparing the equilibrium distribution before and after an intervention. Our approach is related to the comparative statics analysis that is used in economics to study the change in equilibrium distribution after changing exogenous variables or parameters in the model, see also \citet{Simon1988}. In this work, we will additionally consider the effects on the equilibrium distribution of perfect interventions targeting endogenous variables in the equilibrium equations.

	\paragraph{\normalfont{Graphical representations.}} A set of equations can be represented by a bipartite graph. In the case of the filling bathtub, the structure of equilibrium equations \eqref{eq:dyn K} to \eqref{eq:dyn O} is represented by the bipartite graph in Figure \ref{fig:bathtub bipartite graph}. The set $V=\{v_K,v_I,v_P,v_O,v_D\}$ consists of vertices that correspond to variables and the vertices in the set $F=\{f_K,f_I,f_P,f_O,f_D\}$ correspond to equations. There is an edge between a variable vertex $v_i$ and an equation vertex $f_j$ if the variable labelled $v_i$ appears in the equation with label $f_j$. A formal definition of a system of constraints and its associated bipartite graph will be provided in Section \ref{sec:system of constraints}. The causal ordering algorithm, introduced by \citet{Simon1953} and reformulated by us in Section \ref{sec:causal ordering}, takes a \emph{self-contained} bipartite graph as input and returns a \emph{causal ordering graph}. A causal ordering graph is a \emph{directed cluster graph} which consists of variable vertices $v_i$ and equation vertices $f_j$ that are partitioned into clusters. Directed edges go from variable vertices to clusters. For the filling bathtub, the causal ordering graph is given in Figure \ref{fig:bathtub directed cluster graph}. In Section \ref{sec:markov properties} we will show how the \emph{Markov ordering graph} can be constructed from a causal ordering graph. For the equilibrium equations of the filling bathtub, the Markov ordering graph is given in Figure \ref{fig:bathtub markov graph}. The causal ordering algorithm of \citet{Simon1953} can only be applied to bipartite graphs that have the property that they are \emph{self-contained}. In Section \ref{sec:extending the causal ordering algorithm} we introduce an extended causal ordering algorithm that can also be applied to bipartite graphs that are not self-contained.

	\paragraph{\normalfont{Markov property.}} The Markov ordering graph in Figure \ref{fig:bathtub markov graph} encodes conditional independences between the equilibrium solutions $X_{v_K}$, $X_{v_I}$, $X_{v_P}$, $X_{v_O}$, and $X_{v_D}$ of the equilibrium equations. In particular, d-separations between variable vertices in the Markov ordering graph imply conditional independences between the corresponding variables under certain solvability conditions, as we will prove in Theorem \ref{thm:markov property} in Section \ref{sec:markov properties}. In Figure \ref{fig:bathtub markov graph}, the variable vertices $v_I$ and $v_D$ are d-separated by $v_O$. It follows that at equilibrium the inflow rate $X_{v_I}$ and the water level $X_{v_D}$ are independent given the outflow rate $X_{v_O}$. In Sections \ref{sec:extending the causal ordering algorithm} and \ref{sec:generalized directed global markov property} we show how we can use \emph{a perfect matching} for a bipartite graph to construct a directed graph that implies conditional independences between variables via $\sigma$-separations.\footnote{\citet{Forre2017} introduced the notion of $\sigma$-separations to replace $d$-separations in directed graphs that may contain cycles. See Section \ref{sec:cyclic SCMs} for more details.}

	\paragraph{\normalfont{Soft interventions.}} The causal ordering graph in Figure \ref{fig:bathtub directed cluster graph} encodes the effects of \emph{soft} interventions targeting (equilibrium) \emph{equations}. This type of intervention is often also referred to as a mechanism change. We assume that the variables in each cluster can be solved uniquely from the equations in their cluster both before and after the intervention.\footnote{For the underlying dynamical model this assumption means that we assume that the equations of the model define a unique equilibrium to which the system converges and that the system also converges to a unique equilibrium that is defined by the model equations after an intervention on one of the parameters or exogenous variables in the model. For some dynamical systems (e.g., a biochemical reaction network) extra equations are required that describe the dependence of the equilibrium distribution on initial conditions and cannot be modelled in the standard SCM framework \citep{Blom2019}. For these systems, \citet{Blom2019} introduced the more general class of Causal Constraints Models.} A soft intervention has no effect on a variable if there is no directed path from the intervention target to the cluster containing the variable, as we will prove in Theorem \ref{thm:soft interventions} in Section \ref{sec:effects soft interventions}. Consider an experiment where the value of the gravitational constant $g$ is altered (e.g.\ by moving the bathtub to the moon) resulting in an alteration of the equation $f_P$. This is a soft intervention on $f_P$. There is no directed path from $f_P$ to clusters that contain the vertices $\{v_K,v_I,v_P,v_O\}$ in the causal ordering graph in Figure \ref{fig:bathtub directed cluster graph}. Since the conditions of Theorem \ref{thm:soft interventions} are satisfied, the soft intervention on $f_P$ has no effect on $\{X_{v_K},X_{v_I},X_{v_P},X_{v_O}\}$ but it may have an effect on $X_{v_D}$ (depending on the precise functional form of the equations and the values of the parameters).

	\paragraph{\normalfont{Perfect interventions.}} The causal ordering graph in Figure \ref{fig:bathtub directed cluster graph} also encodes the effects of \emph{perfect} interventions on \emph{clusters}, under the assumption that variables can be solved uniquely from the equations of their clusters in the causal ordering graph before and after intervention. We will formally prove this in Theorem \ref{theo:effects of perfect interventions on clusters} in Section \ref{sec:effects perfect interventions}. Consider a perfect intervention on the cluster $\{f_K,v_K\}$ (i.e.\ fixing the diameter $X_{v_K}$ of the drain by altering the equation $f_K$) in Figure \ref{fig:bathtub directed cluster graph}. This intervention generically changes the solution for $\{X_{v_K}, X_{v_P}, X_{v_D}\}$ because $v_K$ is targeted by the intervention and there are directed paths from the cluster of $v_K$ to the clusters of $v_P$ and $v_D$. It has no effect on $\{X_{v_I},X_{v_O}\}$ because there are no directed paths from the cluster of $v_K$ to the clusters of $v_I$ and $v_O$.
\end{example}

\subsection{System of constraints}
\label{sec:system of constraints}

Our formal treatment of sets of equations mirrors the definition of a structural causal model (see also Definition~\ref{def:scm}) in the sense that we separate the model from the endogenous random variables that solve it. We introduce a mathematical object that we call a \emph{system of constraints} to represent equations and their structure as a bipartite graph.

\begin{definition}
	\label{def:system of constraints}
	A \emph{system of constraints} is a tuple $\tuple{\B{\X}, \B{X}_W, \B{\Phi}, \BG=\tuple{V,F,E}}$ where
	\begin{enumerate}
		\item $\B{\X}=\bigotimes_{v\in V}\X_v$, where each $\X_v$ is a standard measurable space and the domain of a variable $X_v$,
		\item $\B{X}_W=(X_w)_{w\in W}$ is a family of independent random variables taking value in $\B{\X}_W$ with $W\subseteq V$ a set of indices corresponding to exogenous variables,\footnote{This means that the nodes $V\setminus W$ correspond to endogenous variables.}
		\item $\B{\Phi}=(\Phi_f)_{f\in F}$ is a family of constraints, each of which is a tuple $\Phi_f=\tuple{\phi_f,c_f,V(f)}$, with:
		\begin{enumerate}
			\item $V(f)\subseteq V$
			\item $c_f$ a constant taking value in a standard measurable space $\Y$,
			\item $\phi_f:\B{\X}_{V(f)}\to \Y$ a measurable function,
		\end{enumerate}	
		\item $\BG=\tuple{V, F, E}$ is a bipartite graph with:
		\begin{enumerate}
			\item $V$ a set of nodes corresponding to variables,
			\item $F$ a set of nodes corresponding to constraints,
			\item $E=\{(f-v): f\in F, v\in V(f)\}$ a set of edges.
		\end{enumerate}
	\end{enumerate}
\end{definition}

Henceforth we will use the terms `variables' and `vertices corresponding to variables' interchangeably. We will also use the terms `constraints', `equations', and `vertices corresponding to constraints' interchangeably. We will often refer to the bipartite graph in a system of constraints as the `associated bipartite graph'. A constraint is formally defined as a triple consisting of a measurable function, a constant, and a subset of the variables. For the sake of convenience we will often write constraints as equations instead. Note that the notation for adjacencies in the associated bipartite graph is equivalent to the notation for the variables that belong to a constraint: $V(f)=\adj{\BG}{f}$. For a set $S_F\subseteq F$, we will let $\adj{\BG}{S_F} = V(S_F)=\cup_{f\in S_F}V(f)$ denote the adjacencies of the vertices $f \in S_F$.

When modelling some system with a system of constraints, we are implicitly assuming that the constraints are \emph{reversible} in the sense that the causal relations between the endogenous variables are flexible and may depend in principle on the entire set of constraints in the system. However, there is an important modelling choice regarding which of the variables to consider as \emph{endogenous} (``internal'' to the system) and which variables to consider as \emph{exogenous} (``external'' to the system). The implicit assumption here is that \emph{the endogenous variables cannot cause the exogenous variables}. This is the (only) causal ``background knowledge'' that is expressed formally by a system of constraints. As \citet{Simon1953} showed, and as we will explicate in later sections, the causal relations between the endogenous variables can then be obtained by applying Simon's causal ordering algorithm.

\begin{example}
Consider two variables: the temperature in a room ($X_1$) and the reading of a thermometer in the same room ($X_2$). One can think of different systems of constraints to model these variables. One possibility is the single constraint ($X_1 - X_2 = 0$) in which both $X_1$ and $X_2$ are considered to be endogenous variables. As it turns out, we will then not be able to draw any conclusion regarding the causal relation between $X_1$ and $X_2$. Another possibility would be to use the same constraint, but now considering $X_1$ to be exogenous and $X_2$ to be endogenous. Then, one will find that $X_1$ causes $X_2$, but not vice versa, which may appear to be a realistic model. Yet another possibility with the same constraint would be to consider $X_2$ to be the exogenous variable and $X_1$ to be endogenous. This model would be considered less realistic in most situations (except perhaps in somewhat unnatural settings where the thermometer would be broken, but its reading would be used by some agent to adjust the heating in order to control the room temperature).

Thus, the constraint $X_1 - X_2 = 0$ on its own does not lead to any conclusions regarding the causal relations between variables $X_1$ and $X_2$; it is only through the additional background knowledge (represented by the distinction between endogenous and exogenous variables) that the causal directionality is fixed. In cases with more than one endogenous variable (like in Example~\ref{ex:bathtub intro}), the causal ordering algorithm can be used to ``propagate'' the causal directionality from exogenous to endogenous variables.
\end{example}

\subsection{Related work and contributions}
\label{sec:introduction:related work}

Graphical models are a popular statistical tool to model probabilistic aspects of complex systems. They represent a set of conditional independences between random variables that correspond to vertices which allows us to learn their graphical structure from data \citep{Lauritzen1996}. These models are often interpreted causally, so that directed edges between vertices are interpreted as direct causal relations between corresponding variables \citep{Pearl2000}. The strong assumptions that are necessary for this viewpoint have been the topic of debate \citep{Dawid2010}. This work contributes to this discussion by revisiting an example in \citet{Iwasaki1994}, and discussing how it seems that, in this case, the presence of vertices representing equations is required to simultaneously express both conditional independences and the effects of interventions in a single graph.

Throughout this work, we discuss the application of the causal ordering algorithm to the equilibrium equations of the bathtub model that we discussed in Example \ref{ex:bathtub intro}. In the literature, feedback processes that have reached equilibrium have been represented by e.g.\ chain graphs \citep{Lauritzen2002} and cyclic directed graphs \citep{Spirtes1995,Mooij2013,Bongers2018}. For the latter it was shown that they imply conditional independences in the equilibrium distribution via the d-separation criterion in the linear or discrete case \citep{Forre2017} but that the directed global Markov property may fail if the underlying model is neither linear nor discrete \citep{Spirtes1995}. The alternative criterion that \citet{Spirtes1995} formulated for the ``collapsed graph'' was recently reformulated in terms of \emph{$\sigma$-separations} and shown to hold in very general settings \citep{Forre2017}. Constraint-based causal discovery algorithms for the cyclic setting under various assumptions are given in \citet{Richardson1996, Forre2018, Strobl2018, Mooij2020a, Mooij2020b}. The causal properties of dynamical systems at equilibrium were previously studied by \citet{Fisher1970, Mooij2013, Hyttinen2012, Lauritzen2002, Mooij2012, Bongers2018}, who consider graphical and causal models that arise from studying the stationary behaviour of dynamical models. For the deterministic case, \citet{Mooij2013} propose to map first-order differential equations to labelled equilibrium equations and then to the structural equations of an SCM. This idea was recently generalized to the stochastic case and higher order differential equations \citep{Bongers2018}. For certain systems, such as the bathtub model in Example \ref{ex:bathtub intro}, this construction may lead to a cyclic SCM with self-cycles \citep{Bongers2018}. The causal and conditional independence properties of cyclic SCMs (possibly with self-cycles) have been studied by \citet{Bongers2020}. In other work assumptions on the underlying dynamical model have been made to avoid the complexities of SCMs with self-cycles. Here, we will consider potential benefits (e.g.\ obtaining a stronger Markov property) of applying the technique of causal ordering to the structural equations of the cyclic SCM for the equilibrium equations of dynamical systems such as the bathtub system.

Our work generalizes the causal ordering algorithm which was introduced by \citet{Simon1953}. Following \citet{Dash2008}, we formally prove that the causal ordering graph that is constructed by the algorithm is unique. One of the novelties of this work is that we also prove that it encodes the effects of soft and certain perfect interventions and, moreover, we show how it can be used to construct a DAG that implies conditional independences via the d-separation criterion. There also exists a different, computationally more efficient, algorithm for causal ordering \citep{Nayak1995, Goncalves2016}. We formally prove that this algorithm is equivalent to the one in \citet{Simon1953}. This approach motivates an alternative representation of the system as a directed graph that may contain cycles. We prove that the \emph{generalized directed global Markov property}, as formulated by \citet{Forre2017}, holds for this graphical representation. Using methods to determine the upper-triangular form of a matrix in \citet{Pothen1990}, we further extend the causal ordering algorithm so that it can be applied to any bipartite graph.

In Section \ref{sec:discussion} we will present a detailed discussion of how our work relates to that of \citet{Iwasaki1994, Bongers2018, Bongers2020, Dash2005}. We show that what \citet{Iwasaki1994} call the ``causal graph'' coincides with the Markov ordering graph in our work. We take a closer look at the intricacies of (possible) causal implications of the Markov ordering graph and notice that it neither represents the effects of soft interventions nor does it have a straightforward interpretation in terms of perfect interventions. Because \citet{Simon1988} assume that a one-to-one correspondence between variables and equations is known in advance, they can use the Markov ordering graph to read off the effects of soft interventions. We argue that the causal ordering graph, and not the Markov ordering graph, should be used to represent causal relations when the matching between variables and equations is not known before-hand. This sheds some new light on the work of \citet{Dash2005} on (causal) structure learning and equilibration in dynamical systems. We further discuss the advantages and disadvantages of our causal ordering approach compared to the SCM framework.

\section{Causal ordering}
\label{sec:causal ordering}

In this section, we adapt the causal ordering algorithm of \citet{Simon1953}, rephrase it in terms of \emph{self-contained} bipartite graphs, and define the output of the algorithm as a \emph{directed cluster graph}.\footnote{The notion of a directed cluster graph corresponds to the box representation of a collapsed graph in \citet{Richardson1996}, Chapter 4.} We then prove that Simon's causal ordering algorithm is well-defined and has a unique output. 

\begin{definition}
\label{def:directed cluster graph}
A \emph{directed cluster graph} is an ordered pair $\tuple{\mathcal{V},\mathcal{E}}$, where $\mathcal{V}$ is a partition $V^{(1)},V^{(2)},\ldots,V^{(n)}$ of a set of vertices $V$ and $\mathcal{E}$ is a set of directed edges $v\to V^{(i)}$ with $v\in V$ and $V^{(i)}\in \mathcal{V}$. For $x\in V$ we let $\mathrm{cl}(x)$ denote the cluster in $\V$ that contains $x$. We say that there is a \emph{directed} path from $x\in V$ to $y\in V$ if either $\mathrm{cl}(x)=\mathrm{cl}(y)$ or there is a sequence of clusters $V_1=\mathrm{cl}(x),V_2,\ldots, V_{k-1}, V_k=\mathrm{cl}(y)$ so that for all $i\in\{1,\ldots,k-1\}$ there is a vertex $z_i\in V_i$ such that $(z_i\to V_{i+1})\in\mathcal{E}$.
\end{definition}

\subsection{Self-contained bipartite graphs}

The causal ordering algorithm in \citet{Simon1953} is presented in terms of a \emph{self-contained} set of equations and variables that appear in them. For bipartite graphs, the notion of \emph{self-containedness} corresponds to the conditions in Definition \ref{def:self-contained}.

\begin{definition}\label{def:self-contained}
	Let $\BG=\tuple{V,F,E}$ be a bipartite graph. A subset $F'\subseteq F$ is said to be \emph{self-contained} if
	\begin{enumerate}
		\item\label{SC1} $|F'|=|\adj{\BG}{F'}|$,
		\item\label{SC2} $|F''|\le|\adj{\BG}{F''}|$ for all $F''\subseteq F'$.\footnote{This condition is also called the Hall Property \citep{Hall1986}.}
	\end{enumerate}

	The bipartite graph $\BG$ is said to be \emph{self-contained} if $|F|=|V|$ and $F$ is self-contained. A non-empty self-contained set $F'\subseteq F$ is said to be a \emph{minimal self-contained set}\footnote{In this case the Strong Hall Property holds, that is $|F''|<|\adj{\BG}{F''}|$ for all $\emptyset \varsubsetneq F''\varsubsetneq F'$ \citep{Hall1986}.} if all its non-empty strict subsets are not self-contained.
\end{definition}

\begin{example}
	\label{ex:self-contained}
	In Figure \ref{fig:self-contained} a bipartite graph is shown with self-contained sets
	\begin{equation*}
	\{f_1\},\{f_1,f_2,f_3,f_4\}, \{f_1,f_2,f_3,f_4,f_5\}
	\end{equation*}
	where $\{f_1\}$ is a minimal self-contained set. Since the set $\{f_1,f_2,f_3,f_4,f_5\}$ is self-contained and $|V|=|F|=5$, we say that this bipartite graph is self-contained.
\end{example}

\begin{figure}[ht]
	\centering
	\begin{tikzpicture}[scale=0.75,every node/.style={transform shape}]
	\GraphInit[vstyle=Normal]
	\SetGraphUnit{1}
	\SetVertexMath
	\Vertex{v_1}
	\EA(v_1){v_2}
	\EA(v_2){v_3}
	\EA(v_3){v_4}
	\EA(v_4){v_5}
	\SO[unit=1.2](v_1){f_1}
	\SO[unit=1.2](v_2){f_2}
	\SO[unit=1.2](v_3){f_3}
	\SO[unit=1.2](v_4){f_4}
	\SO[unit=1.2](v_5){f_5}
	\tikzset{EdgeStyle/.style = {-}}
	\Edge(f_1)(v_1)
	\Edge(f_2)(v_2)
	\Edge(f_3)(v_3)
	\Edge(f_4)(v_4)
	\Edge(f_5)(v_5)
	\Edge(f_2)(v_1)
	\Edge(f_2)(v_3)
	\Edge(f_3)(v_4)
	\Edge(f_4)(v_2)
	\Edge(f_5)(v_4)
	\end{tikzpicture}
	\caption{A self-contained bipartite graph $\BG=\tuple{V,F,E}$ with $V=\{v_1,v_2,v_3,v_4,v_5\}$ and $F=\{f_1,f_2,f_3,f_4,f_5\}$. The sets $\{f_1\}$, $\{f_1,f_2,f_3,f_4\}$, and $\{f_1,f_2,f_3,f_4,f_5\}$ are self-contained, and $\{f_1\}$ is the only minimal self-contained set.}
	\label{fig:self-contained}
\end{figure}
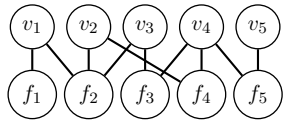

Sets of equations that model systems in the real world often include both \emph{endogenous} and \emph{exogenous} variables. The distinction is that exogenous variables are assumed to be determined outside the system and function as inputs to the model, whereas the endogenous variables are part of the system. The following example illustrates that the associated bipartite graph for a set of equations with both endogenous and exogenous variables is usually not self-contained.

\begin{example}
	\label{ex:system of constraints exogenous}
	Let $V=\{v_1,v_2,w_1,w_2\}$ be an index set for endogenous and exogenous variables $\B{X}=(X_i)_{i\in V}$, $W=\{w_1,w_2\}$ a subset that is an index set for exogenous variables only, and $F=\{f_1,f_2\}$ an index set for equations:
	\begin{alignat*}{3}
	\Phi_{f_1}:&\qquad& X_{v_1}-X_{w_1} &=0,\\
	\Phi_{f_2}:&\qquad& X_{v_2}-X_{v_1}-X_{w_2} &=0.
	\end{alignat*}
	The associated bipartite graph $\BG=\tuple{V,F,E}$ is given in Figure \ref{fig:bipartite graph with exogenous variables}. It has vertices $V$ that correspond to both endogenous variables $X_{v_1}$, $X_{v_2}$ and exogenous variables $X_{w_1}$, $X_{w_2}$. The vertices $F$ correspond to constraints $\Phi_{f_1}$ and $\Phi_{f_2}$. Edges between vertices $v\in V$ and $f\in F$ are present whenever $v\in V(f)$ (i.e.\ when the variable $X_v$ appears in the constraint $\Phi_f$). Since $|V|\neq |F|$, the associated bipartite graph is not self-contained.
\end{example}

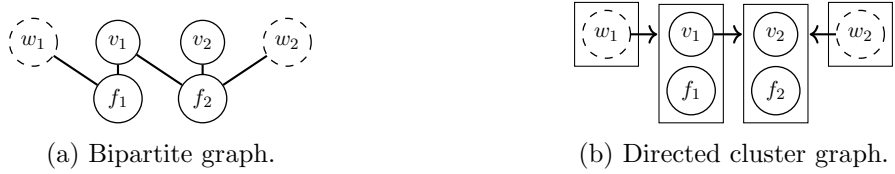
\begin{figure}[ht]
	\begin{subfigure}[b]{.5\textwidth}
		\centering
		\begin{tikzpicture}[scale=0.75,every node/.style={transform shape}]
		\GraphInit[vstyle=Normal]
		\SetGraphUnit{1}
		\SetVertexMath
		\Vertex[style={dashed}]{w_1}
		\EA[unit=1.5](w_1){v_1}
		\EA[unit=1.5](v_1){v_2}
		\EA[unit=1.5,style={dashed}](v_2){w_2}
		\SO(v_1){f_1}
		\SO(v_2){f_2}
		\draw[EdgeStyle, style={-}](w_1) to (f_1);
		\draw[EdgeStyle, style={-}](v_1) to (f_1);
		\draw[EdgeStyle, style={-}](w_2) to (f_2);
		\draw[EdgeStyle, style={-}](v_2) to (f_2);
		\draw[EdgeStyle, style={-}](v_1) to (f_2);
		\end{tikzpicture}
		\caption{Bipartite graph.}
		\label{fig:bipartite graph with exogenous variables}
	\end{subfigure}%
	\begin{subfigure}[b]{.5\textwidth}
		\centering
		\begin{tikzpicture}[scale=0.75,every node/.style={transform shape}]
		\GraphInit[vstyle=Normal]
		\SetGraphUnit{1}
		\SetVertexMath
		\Vertex[style={dashed}]{w_1}
		\EA[unit=1.5](w_1){v_1}
		\EA[unit=1.5](v_1){v_2}
		\EA[unit=1.5,style={dashed}](v_2){w_2}
		\SO(v_1){f_1}
		\SO(v_2){f_2}
		\node[draw=black, fit=(v_1) (f_1), inner sep=0.1cm ]{};
		\node[draw=black, fit=(v_2) (f_2), inner sep=0.1cm ]{};
		\node[draw=black, fit=(w_1), inner sep=0.1cm ]{};
		\node[draw=black, fit=(w_2), inner sep=0.1cm ]{};
		\draw[EdgeStyle, style={->}](w_1) to (0.9,0);
		\draw[EdgeStyle, style={->}](v_1) to (2.4,0);
		\draw[EdgeStyle, style={->}](w_2) to (3.6,0);
		\end{tikzpicture}
		\caption{Directed cluster graph.}
		\label{fig:directed cluster graph exogenous}
	\end{subfigure}%
	\caption{The bipartite graph in Figure \ref{fig:bipartite graph with exogenous variables} is associated with the constraints in Example \ref{ex:system of constraints exogenous}. Exogenous variables are indicated by dashed circles. The directed cluster graph that is obtained by applying Algorithm \ref{alg:causal ordering minimal self-contained exo} is shown in Figure \ref{fig:directed cluster graph exogenous}.}
	\label{fig:exogenous variables}	
\end{figure}

\subsection{Causal ordering algorithm}
\label{sec:causal ordering algorithm}

The causal ordering algorithm, as formulated by \citet{Simon1953}, has as input a self-contained set of equations and as output it has an ordering on clusters of variables that appear in these equations. We reformulate the algorithm in terms of bipartite graphs and minimal self-contained sets. The input of the algorithm is then a self-contained bipartite graph and its output a directed cluster graph that we call the \emph{causal ordering graph}.

The causal ordering algorithm below has been adapted for systems of constraints with exogenous variables. The input is a bipartite graph $\BG=\tuple{V,F,E}$ and a set of vertices $W\subseteq V$ (corresponding to exogenous variables) such that the subgraph $\BG'=\tuple{V',F',E'}$ induced by $(V\setminus W) \cup F$ is self-contained. The algorithm starts out by adding the exogenous vertices as singleton clusters to a cluster set $\V$ during an initialization step. Subsequently, the algorithm searches for a minimal self-contained set $S_F\subseteq F$ in $\BG'$. Together with the set of adjacent variable vertices $S_V=\mathrm{adj}_{\BG'}(S_F)$ a cluster $S_F\cup S_V$ is formed and added to $\mathcal{V}$. For each $v\in V$, an edge $(v\to (S_F\cup S_V))$ is added to $\mathcal{E}$ if $v\notin S_V$ and $v\in\mathrm{adj}_{\BG}(S_F)$. In other words, the cluster has an incoming edge from each variable vertex that is adjacent to the cluster but not in it. These steps are then repeated for the subgraph induced by the vertices $(V'\cup F')\setminus (S_V \cup S_F)$ that are not in the cluster, as long as this is not the null graph. The order in which the self-contained sets are obtained is represented by one of the topological orderings of the clusters in the causal ordering graph $\mathrm{CO}(\BG)=\tuple{\mathcal{V},\mathcal{E}}$.

\SetKwInput{KwData}{Input}
\SetKwInput{KwResult}{Output}

\begin{center}
	\begin{algorithm}[H]
		\DontPrintSemicolon
		\KwData{a set of exogenous vertices $W$, a bipartite graph $\BG=\tuple{V,F,E}$ such that its subgraph induced by $(V\setminus W) \cup F$ is self-contained}
		\KwResult{directed cluster graph $\mathrm{CO}(\BG)=\tuple{\mathcal{V},\mathcal{E}}$}
		$\mathcal{E}\leftarrow \emptyset$ {\color{gray}\tcp*{initialization}}
		$\mathcal{V}\leftarrow \{\{w\}:w\in W\}$ {\color{gray}\tcp*{initialization}}
		$\BG'\leftarrow \tuple{V',F',E'}$ subgraph induced by $(V\setminus W) \cup F$ {\color{gray}\tcp*{initialization}}
		\While{$\BG'$ is not the null graph}{
			$S_F\leftarrow$ a minimal self-contained set of $F'$\;
			$C\leftarrow $ $S_F\cup\mathrm{adj}_{\BG'}(S_F)$ {\color{gray}\tcp*{construct cluster}}
			$\mathcal{V} \leftarrow$ $\mathcal{V}\cup \{C\}$ {\color{gray}\tcp*{add cluster}}
			\For{$v\in \mathrm{adj}_{\BG}(S_F)\setminus \mathrm{adj}_{\BG'}(S_F)$}{
				$\mathcal{E}\leftarrow$ $\mathcal{E}\cup \{(v\to C)\}$ {\color{gray}\tcp*{add edges to cluster}}
			}
			$\BG'\leftarrow$ subgraph of $\BG'$ induced by $(V'\cup F')\setminus C$ {\color{gray}\tcp*{remove cluster}}
		}
		\label{alg:causal ordering minimal self-contained exo}
		\caption{Causal ordering using minimal self-contained sets.}
	\end{algorithm}
\end{center}

Theorem \ref{thm:uniqueness} shows that the output of causal ordering via minimal self-contained sets is well-defined and unique.

\begin{restatable}{theorem}{uniqueness}
	\label{thm:uniqueness}
	The output of Algorithm \ref{alg:causal ordering minimal self-contained exo} is well-defined and unique.
\end{restatable}

The following example shows how the causal ordering algorithm works on the self-contained bipartite graph in Figure \ref{fig:self-contained} and the bipartite graph in Figure \ref{fig:bipartite graph with exogenous variables}.

\begin{example}
	\label{ex:adapted causal ordering with minimal self-contained sets}
	Consider the set of equations in Example \ref{ex:system of constraints exogenous} and its associated bipartite graph in Figure \ref{fig:bipartite graph with exogenous variables}. The subgraph induced by the endogenous variables $v_1,v_2$ and the constraints $f_1,f_2$ is self-contained. We initialize Algorithm \ref{alg:causal ordering minimal self-contained exo} with $\mathcal{E}$ the empty set, $\mathcal{V}=\{ \{w_1\}, \{w_2\}\}$, and $\BG'$ the subgraph induced by $\{v_1,v_2,f_1,f_2\}$. We then first find the minimal self-contained set $\{f_1\}$. Its adjacencies are $\{v_1\}$ in $\BG'$ and $\{v_1,w_1\}$ in $\BG$. We add $\{v_1,f_1\}$ to $\V$ and add the edge $(w_1\to \{v_1,f_1\})$ to $\mathcal{E}$. Finally, we add $\{v_2,f_2\}$ to $\mathcal{V}$ and the edges $(v_1\to \{v_2,f_2\})$ and $(w_2\to\{v_2,f_2\})$ to $\mathcal{E}$. The output of the causal ordering algorithm is the directed cluster graph in Figure \ref{fig:directed cluster graph exogenous}. This reflects how one would solve the system of equations $\Phi_{f_1}$, $\Phi_{f_2}$ with respect to $X_{v_1}$, $X_{v_2}$ in terms of $X_{w_1}$, $X_{w_2}$ by hand.
\end{example}

\section{Extending the causal ordering algorithm}
\label{sec:extending the causal ordering algorithm}

In this section we present an adaptation of an alternative, computationally less expensive, algorithm for causal ordering which uses perfect matchings instead of minimal self-contained sets, similar to the algorithm suggested by \citet{Nayak1995}. \citet{Goncalves2016} proved that Simon's classic algorithm makes use of a subroutine that solves an NP-hard problem, whereas the computational complexity of Nayak's algorithm is bounded by $\mathcal{O}\left(|V|\,|E|\right)$, where $|V|$ is the number of nodes and $|E|$ is the number of edges in the bipartite graph. Here, we provide a proof for the fact that causal ordering via minimal self-contained sets is equivalent to causal ordering via perfect matchings. There are many systems of equations with a unique solution that consist of more equations than there are endogenous variables, most notably in the case of non-linear equations, or in the presence of cycles. In that case the bipartite graph associated with these equations may not be self-contained. In this section, we show how Nayak's algorithm can be extended using maximum matchings so that it can be applied to any bipartite graph.

\subsection{Causal ordering via perfect matchings}
\label{sec:causal ordering via perfect matchings}

Given a bipartite graph $\BG$, the \emph{associated directed graph} can be constructed from a matching $\M$ by \emph{orienting edges}. A directed cluster graph can then be constructed via the operations that \emph{construct clusters} and \emph{merge clusters} in Definition \ref{def:orient, cluster, merge} below.

\begin{definition}
\label{def:orient, cluster, merge}
Let $\BG=\tuple{V,F,E}$ be a bipartite graph and $\M$ a perfect matching for $\BG$.
\begin{enumerate}
	\item \emph{Orient edges}: For each $(v-f)\in E$ the edge set $E_{\mathrm{dir}}$ has an edge $(v\leftarrow f)$ if $(v-f)\in \M$ and an edge $(v\rightarrow f)$ if $(v-f)\notin \M$. $E_{\mathrm{dir}}$ has no additional edges. The \emph{associated directed graph} is $\G(\BG,\M)=\tuple{V\cup F,E_{\mathrm{dir}}}$.
	\item \emph{Construct clusters}: Let $\V'$ be a partition of vertices $V\cup F$ into strongly connected components in $\G(\BG,\M)$. For each $(x\to w)\in E_{\mathrm{dir}}$ the edge set $\E'$ has an edge $(x\to \mathrm{cl}(w))$ if $x\notin \mathrm{cl}(w)$, where $\mathrm{cl}(w)\in\V'$ is the strongly connected component of $w$ in $\G(\BG,\M)$. The edge set $\E'$ has no additional edges. The \emph{associated clustered graph} is $\mathrm{clust}(\G(\BG,\M))=\tuple{\V', \E'}$. 
  \item \emph{Merge clusters}: Let $\V = \{S\cup\M(S): S\in\V'\}$. For each $(x\to S)\in \E'$ with $x\notin \M(S)$ the edge set $\E$ contains an edge $(x\to S\cup M(S))$. The edge set $\E$ has no additional edges. The \emph{associated clustered and merged graph} is $\mathrm{merge}(\mathrm{clust}(\G(\BG,\M)) )=\tuple{\V,\E}$.\footnote{In Theorem \ref{thm:equivalence} we will show that this is the causal ordering graph $\mathrm{CO}(\BG)$.}
\end{enumerate}
\end{definition}

\begin{center}
\begin{algorithm}[H]
\DontPrintSemicolon
\KwIn{a set of exogenous vertices $W$, a bipartite graph $\BG=\tuple{V,F,E}$ such that the subgraph induced by $(V\cup F)\setminus W$ is self-contained}
\KwOut{directed cluster graph $\tuple{\mathcal{V},\mathcal{E}}$}
$\BG'\leftarrow$ subgraph induced by $(V\setminus W)\cup F$ {\color{gray}\tcp*{initialization}}
$\M\leftarrow$ perfect matching for $\BG'$ {\color{gray}\tcp*{initialization}}
$E_{\mathrm{dir}}\leftarrow \emptyset$ {\color{gray}\tcp*{orient edges}}

\For{$(v-f)\in E$ \textbf{\upshape with} $f\in F$}{
	\If{$(v-f)\in \M$}
	{Add $(v\leftarrow f)$ to $E_{\mathrm{dir}}$}
	\Else
	{Add $(v\rightarrow f)$ to $E_{\mathrm{dir}}$}
}

$\V'\leftarrow$ strongly connected components of $\tuple{V\cup F,E_{\mathrm{dir}}}$ {\color{gray}\tcp*{clustering}}
$\E'\leftarrow \emptyset$\;

\For{$(x\to w) \in E_{\mathrm{dir}}$}{
	\For{$S\in\V'$}{
		\If{$w\in S$ \textbf{\upshape and} $x\notin S$}{Add $(x\to S)$ to $\E'$}	
	}
}

$\V,\E\leftarrow \emptyset$ {\color{gray}\tcp*{merge clusters}}

\For{$S\in \V'$}{
	Add $S\cup M(S)$ to $\V$ \;
	\For{$(x\to S)\in\E'$}{
		\If{$x\notin M(S)$}{Add $(x\to S\cup M(S))$ to $\E$}
	}
}
\label{alg:causal ordering perfect matchings}
\caption{Causal ordering via perfect matching.}
\end{algorithm}
\end{center}

For \emph{causal ordering via perfect matching} we require as input a set of exogenous vertices $W$ and a bipartite graph $\BG=\tuple{V,F,E}$, for which the subgraph $\BG'$ induced by the vertices $(V\cup F)\setminus W$ is self-contained. The output is a directed cluster graph. The details can be found in Algorithm \ref{alg:causal ordering perfect matchings}. We see that the algorithm starts out by finding a perfect matching\footnote{Note that a bipartite graph has a perfect matching if and only if it is self-contained \citep{Hall1986}. See also Theorem \ref{thm:hall} and Corollary \ref{cor:perfect matching self contained} in Appendix~\ref{app:equivalence proof}.} for $\BG'$,\footnote{The Hopcraft-Karp-Karzanov algorithm, which runs in $\mathcal{O}(|E|\sqrt{|V\cup F|})$, can be used to find a perfect matching \citep{Hopcroft1973, Karzanov1973}.} which is then used to orient edges in the bipartite graph $\BG$. The algorithm then proceeds by partitioning vertices in the resulting directed graph into strongly connected components to construct the associated clustered graph.\footnote{Tarjan's algorithm, which runs in linear time, can be used to find the strongly connected components in a directed graph \citep{Tarjan1972}.} Finally, the merge operation is applied to construct the causal ordering graph. Theorem \ref{thm:equivalence} below shows that causal ordering via perfect matchings is equivalent to causal ordering via minimal self-contained sets.

\begin{restatable}{theorem}{equivalence}
\label{thm:equivalence}
The output of Algorithm \ref{alg:causal ordering perfect matchings} coincides with the output of Algorithm \ref{alg:causal ordering minimal self-contained exo}.
\end{restatable}

The following example illustrates that the output of causal ordering via perfect matchings does not depend on the choice of perfect matching and coincides with the output of Algorithm \ref{alg:causal ordering minimal self-contained exo}.

\begin{example}
\label{ex:causal ordering perfect matchings}
Consider the bipartite graph $\BG$ in Figure \ref{fig:bipartite exogenous}. The subgraph induced by the vertices $V=\{v_1,\ldots,v_5\}$ and $F=\{f_1,\ldots, f_5\}$ is the self-contained bipartite graph in Figure \ref{fig:self-contained}. We will follow the steps in both Algorithm \ref{alg:causal ordering minimal self-contained exo} and \ref{alg:causal ordering perfect matchings} to construct the causal ordering graph.
	
For causal ordering with minimal self-contained sets we first add the exogenous variables to the cluster set $\V$ as the singleton clusters $\{w_1\}$, $\{w_2\}$, $\{w_3\}$, $\{w_4\}$,  $\{w_5\}$, and $\{w_6\}$. The only minimal self-contained set in the subgraph induced by the vertices $V=\{v_1,\ldots,v_5\}$ and $F=\{f_1,\ldots, f_5\}$ is $\{f_1\}$. Since $f_1$ is adjacent to $v_1$ we add $C_1=\{v_1,f_1\}$ to $\V$. Since $f_1$ is adjacent to $w_1$ in $\BG$ we add $(w_1\to C_1)$ to $\E$. The subgraph $\BG'=\tuple{V',F',E'}$ induced by the remaining nodes $V'=\{v_2,v_3,v_4,v_5\}$ and $F'=\{f_2,f_3,f_4,f_5\}$ has $\{f_2,f_3,f_4\}$ as its only minimal self-contained set. Since the set $\{f_2,f_3,f_4\}$ is adjacent to $\{v_2,v_3,v_4\}$ in $\BG'$, we add $C_2=\{v_2,v_3,v_4,f_2,f_3,f_4\}$ to $\V$. Since $v_1$, $w_2$, $w_3$, $w_4$, and $w_5$ are adjacent to $\{f_2,f_3,f_4\}$ in $\BG$ but not part of $C_2$, we add the edges $(v_1\to C_2)$, $(w_2\to C_2)$, $(w_3\to C_2)$, $(w_4\to C_2)$, and $(w_5\to C_2)$ to $\mathcal{E}$. The subgraph induced by the remaining nodes $v_5$ and $f_5$ has $\{f_5\}$ as its minimal self-contained subset. We add $C_3=\{v_5,f_5\}$ to $\V$ and the edges $(v_4\to C_3)$ and $(w_6\to C_3)$ to $\mathcal{E}$. The directed cluster graph $\mathrm{CO}(\BG)=\tuple{\mathcal{V},\mathcal{E}}$ is given in Figure \ref{fig:cog exogenous}.
	
For causal ordering via perfect matchings, we consider the following two perfect matchings of the self-contained bipartite graph in Figure \ref{fig:self-contained}:
	\begin{align*}
	\M_1 &= \{(v_1-f_1),(v_2-f_2),(v_3-f_3),(v_4-f_4),(v_5-f_5)\},\\
	\M_2 &= \{(v_1-f_1),(v_2-f_4),(v_3-f_2),(v_4-f_3),(v_5-f_5)\}.
	\end{align*}
We use these one-to-one correspondences between endogenous variable vertices and constraint vertices in the orientation step in Definition \ref{def:orient, cluster, merge} to obtain the associated directed graphs $\G(\BG,\M_1)$ and $\G(\BG,\M_2)$  in Figures \ref{fig:directed exogenous} and \ref{fig:directed exogenous2} respectively. Application of the clustering step in Definition \ref{def:orient, cluster, merge} to either $\G(\BG,\M_1)$ or $\G(\BG,\M_2)$ results in the clustered graph $\mathrm{clust}(\G(\BG,\M_2))=\mathrm{clust}(\G(\BG,\M_1))$ in Figure \ref{fig:clustered exogenous}. The final step is to merge clusters in this directed cluster graph. We find that the causal ordering graph $\mathrm{merge}(\mathrm{clust}(\G(\BG,\M_1)))=\mathrm{merge}(\mathrm{clust}(\G(\BG,\M_2)))$ in Figure \ref{fig:cog exogenous} does not depend on the choice of perfect matching, as is implied by Theorem \ref{thm:equivalence}. Note that the output of causal ordering with minimal self-contained sets coincides with the output of causal ordering via perfect matchings.
\end{example}

\begin{figure}[p]
\begin{subfigure}{\textwidth}
	\centering
	\begin{tikzpicture}[scale=0.75,every node/.style={transform shape}]
	\GraphInit[vstyle=Normal]
	\SetGraphUnit{1.2}
	\SetVertexMath
	\Vertex[style={dashed}]{w_1}
	\EA(w_1){v_1}
	\EA[style={dashed}](v_1){w_2}
	\EA(w_2){v_2}
	\EA[style={dashed}](v_2){w_3}
	\EA(w_3){v_3}
	\EA[style={dashed}](v_3){w_4}
	\EA(w_4){v_4}
	\EA[style={dashed}](v_4){w_5}
	\EA(w_5){v_5}
	\EA[style={dashed}](v_5){w_6}
	\SO[unit=1.4](v_1){f_1}
	\SO[unit=1.4](v_2){f_2}
	\SO[unit=1.4](v_3){f_3}
	\SO[unit=1.4](v_4){f_4}
	\SO[unit=1.4](v_5){f_5}
	\tikzset{EdgeStyle/.style = {-}}
	\Edge(f_1)(v_1)
	\Edge(f_2)(v_2)
	\Edge(f_3)(v_3)
	\Edge(f_4)(v_4)
	\Edge(f_5)(v_5)
	\Edge(f_2)(v_1)
	\Edge(f_2)(v_3)
	\Edge(f_3)(v_4)
	\Edge[style={bend left=8}](f_4)(v_2)
	\Edge(f_5)(v_4)
	\Edge(f_1)(w_1)
	\Edge(f_2)(w_2)
	\Edge(f_2)(w_3)
	\Edge(f_3)(w_4)
	\Edge(f_4)(w_5)
	\Edge(f_5)(w_6)
	\end{tikzpicture}
	\caption{Bipartite graph $\BG$ where dashed vertices indicate exogenous variables.}
	\label{fig:bipartite exogenous}
\end{subfigure}%
\vspace{4mm}
\begin{subfigure}{\textwidth}
	\centering
	\begin{tikzpicture}[scale=0.75,every node/.style={transform shape}]
	\GraphInit[vstyle=Normal]
	\SetGraphUnit{1.2}
	\SetVertexMath
	\Vertex[style={dashed}]{w_1}
	\EA(w_1){v_1}
	\EA[style={dashed}](v_1){w_2}
	\EA(w_2){v_2}
	\EA[style={dashed}](v_2){w_3}
	\EA(w_3){v_3}
	\EA[style={dashed}](v_3){w_4}
	\EA(w_4){v_4}
	\EA[style={dashed}](v_4){w_5}
	\EA(w_5){v_5}
	\EA[style={dashed}](v_5){w_6}
	\SO[unit=1.4](v_1){f_1}
	\SO[unit=1.4](v_2){f_2}
	\SO[unit=1.4](v_3){f_3}
	\SO[unit=1.4](v_4){f_4}
	\SO[unit=1.4](v_5){f_5}
	\tikzset{EdgeStyle/.style = {->, blue}}
	\Edge(f_1)(v_1)
	\Edge(f_2)(v_2)
	\Edge(f_3)(v_3)
	\Edge(f_4)(v_4)
	\Edge(f_5)(v_5)
	\tikzset{EdgeStyle/.style = {<-}}
	\Edge(f_2)(v_1)
	\Edge(f_2)(v_3)
	\Edge(f_3)(v_4)
	\Edge[style={bend left=8}](f_4)(v_2)
	\Edge(f_5)(v_4)
	\Edge(f_1)(w_1)
	\Edge(f_2)(w_2)
	\Edge(f_2)(w_3)
	\Edge(f_3)(w_4)
	\Edge(f_4)(w_5)
	\Edge(f_5)(w_6)
	\end{tikzpicture}
	\caption{Associated directed graph $\G(\BG,\M_1)$.}
	\label{fig:directed exogenous}
\end{subfigure}%
\vspace{4mm}
\begin{subfigure}{\textwidth}
	\centering
	\begin{tikzpicture}[scale=0.75,every node/.style={transform shape}]
	\GraphInit[vstyle=Normal]
	\SetGraphUnit{1.2}
	\SetVertexMath
	\Vertex[style={dashed}]{w_1}
	\EA(w_1){v_1}
	\EA[style={dashed}](v_1){w_2}
	\EA(w_2){v_2}
	\EA[style={dashed}](v_2){w_3}
	\EA(w_3){v_3}
	\EA[style={dashed}](v_3){w_4}
	\EA(w_4){v_4}
	\EA[style={dashed}](v_4){w_5}
	\EA(w_5){v_5}
	\EA[style={dashed}](v_5){w_6}
	\SO[unit=1.4](v_1){f_1}
	\SO[unit=1.4](v_2){f_2}
	\SO[unit=1.4](v_3){f_3}
	\SO[unit=1.4](v_4){f_4}
	\SO[unit=1.4](v_5){f_5}
	\tikzset{EdgeStyle/.style = {->, orange}}
	\Edge(f_1)(v_1)
	\Edge(f_2)(v_3)
	\Edge(f_3)(v_4)
	\Edge[style={bend left=8}](f_4)(v_2)
	\Edge(f_5)(v_5)		
	\tikzset{EdgeStyle/.style = {<-}}
	\Edge(f_2)(v_2)
	\Edge(f_3)(v_3)
	\Edge(f_4)(v_4)
	\Edge(f_2)(v_1)
	\Edge(f_5)(v_4)
	\Edge(f_1)(w_1)
	\Edge(f_2)(w_2)
	\Edge(f_2)(w_3)
	\Edge(f_3)(w_4)
	\Edge(f_4)(w_5)
	\Edge(f_5)(w_6)
	\end{tikzpicture}
	\caption{Associated directed graph $\G(\BG,\M_2)$.}
	\label{fig:directed exogenous2}
\end{subfigure}%
\vspace{4mm}
\begin{subfigure}{0.5\textwidth}		
	\centering
	\begin{tikzpicture}[scale=0.75,every node/.style={transform shape}]
	\GraphInit[vstyle=Normal]
	\SetGraphUnit{1.5}
	\SetVertexMath
	\Vertex{v_1}
	\EA[unit=1.5](v_1){v_2}
	\EA[unit=1.0](v_2){v_3}
	\EA[unit=1.0](v_3){v_4}
	\EA[unit=1.5](v_4){v_5}
	\SO[unit=1.35](v_1){f_1}
	\SO[unit=1.35](v_2){f_2}
	\SO[unit=1.35](v_3){f_3}
	\SO[unit=1.35](v_4){f_4}
	\SO[unit=1.35](v_5){f_5}
	\Vertex[style={dashed}, x=-0.75, y=-2.75] {w_1}
	\EA[unit=1.25, style={dashed}](w_1){w_2}
	\EA[unit=1.25, style={dashed}](w_2){w_3}
	\EA[unit=1.25, style={dashed}](w_3){w_4}
	\EA[unit=1.25, style={dashed}](w_4){w_5}
	\EA[unit=1.25, style={dashed}](w_5){w_6}
	\node[draw=black, fit=(w_1), inner sep=0.1cm ]{};
	\node[draw=black, fit=(w_2), inner sep=0.1cm ]{};
	\node[draw=black, fit=(w_3), inner sep=0.1cm ]{};
	\node[draw=black, fit=(w_4), inner sep=0.1cm ]{};
	\node[draw=black, fit=(w_5), inner sep=0.1cm ]{};
	\node[draw=black, fit=(w_6), inner sep=0.1cm ]{};
	\node[draw=black, fit=(v_1), inner sep=0.1cm ]{};
	\node[draw=black, fit=(f_1), inner sep=0.1cm ]{};
	\node[draw=black, fit=(v_2) (f_3) (v_4) (f_2) (f_3) (f_4), inner sep=0.1cm ]{};
	\node[draw=black, fit=(v_5), inner sep=0.1cm ]{};
	\node[draw=black, fit=(f_5), inner sep=0.1cm ]{};
	\draw[EdgeStyle, style={->}](f_5) to (5.0,-0.525);
	\draw[EdgeStyle, style={->}](f_1) to (0,-0.525);
	\draw[EdgeStyle, style={->}](v_1) to (0.975,0);
	\draw[EdgeStyle, style={->}](v_4) to (4.4,-0.765);
	\draw[EdgeStyle, style={->}](w_1) to (-0.3,-1.925);
	\draw[EdgeStyle, style={->}](w_2) to (1.2,-1.925);
	\draw[EdgeStyle, style={->}](w_3) to (1.75,-1.925);
	\draw[EdgeStyle, style={->}](w_4) to (3.0,-1.925);
	\draw[EdgeStyle, style={->}](w_5) to (3.75,-1.925);		
	\draw[EdgeStyle, style={->}](w_6) to (5.25,-1.925);
	\end{tikzpicture}
	\caption{Clustered graph $\mathrm{clust}(\G(\BG,\M_1))$.}
	\label{fig:clustered exogenous}
\end{subfigure}%
\vspace{4mm}
\begin{subfigure}{0.5\textwidth}
	\centering
	\begin{tikzpicture}[scale=0.75,every node/.style={transform shape}]
	\GraphInit[vstyle=Normal]
	\SetGraphUnit{1.5}
	\SetVertexMath
	\Vertex{v_1}
	\EA[unit=1.5](v_1){v_2}
	\EA[unit=1.0](v_2){v_3}
	\EA[unit=1.0](v_3){v_4}
	\EA[unit=1.5](v_4){v_5}
	\SO[unit=1.35](v_1){f_1}
	\SO[unit=1.35](v_2){f_2}
	\SO[unit=1.35](v_3){f_3}
	\SO[unit=1.35](v_4){f_4}
	\SO[unit=1.35](v_5){f_5}
	\Vertex[style={dashed}, x=-0.75, y=-2.75] {w_1}
	\EA[unit=1.25, style={dashed}](w_1){w_2}
	\EA[unit=1.25, style={dashed}](w_2){w_3}
	\EA[unit=1.25, style={dashed}](w_3){w_4}
	\EA[unit=1.25, style={dashed}](w_4){w_5}
	\EA[unit=1.25, style={dashed}](w_5){w_6}
	\node[draw=black, fit=(w_1), inner sep=0.1cm ]{};
	\node[draw=black, fit=(w_2), inner sep=0.1cm ]{};
	\node[draw=black, fit=(w_3), inner sep=0.1cm ]{};
	\node[draw=black, fit=(w_4), inner sep=0.1cm ]{};
	\node[draw=black, fit=(w_5), inner sep=0.1cm ]{};
	\node[draw=black, fit=(w_6), inner sep=0.1cm ]{};
	\node[draw=black, fit=(v_1) (f_1), inner sep=0.1cm ]{};
	\node[draw=black, fit=(v_2) (f_3) (v_4) (f_2) (f_3) (f_4), inner sep=0.1cm ]{};
	\node[draw=black, fit=(v_5) (f_5), inner sep=0.1cm ]{};
	\draw[EdgeStyle, style={->}](v_1) to (0.975,0);
	\draw[EdgeStyle, style={->}](v_4) to (4.4,0);
	\draw[EdgeStyle, style={->}](w_1) to (-0.3,-1.925);
	\draw[EdgeStyle, style={->}](w_2) to (1.2,-1.925);
	\draw[EdgeStyle, style={->}](w_3) to (1.75,-1.925);
	\draw[EdgeStyle, style={->}](w_4) to (3.0,-1.925);
	\draw[EdgeStyle, style={->}](w_5) to (3.75,-1.925);		
	\draw[EdgeStyle, style={->}](w_6) to (5.25,-1.925);
	\end{tikzpicture}
	\caption{Causal ordering graph $\mathrm{CO}(\BG)$.}
	\label{fig:cog exogenous}
\end{subfigure}
\caption{Causal ordering with two different perfect matchings $M_1$ and $M_2$ applied to the bipartite graph in Figure \ref{fig:bipartite exogenous}. The results of subsequently orienting edges, constructing clusters, and merging clusters as in Definition \ref{def:orient, cluster, merge} are given in Figures \ref{fig:directed exogenous} to \ref{fig:cog exogenous}. The edges in $M_1$ that are oriented from variables to equations in Figure \ref{fig:directed exogenous} are indicated with blue edges. Likewise, edges in $M_2$ are indicated with orange edges in Figure \ref{fig:directed exogenous2}. The clustered graph in Figure \ref{fig:clustered exogenous} coincides with $\mathrm{clust}(\G(\BG,\M_2))$ and for the causal ordering graph in Figure \ref{fig:cog exogenous} we have that $\mathrm{CO}(\BG) = \mathrm{merge}(\mathrm{clust}(\G(\BG,\M_1)))=\mathrm{merge}(\mathrm{clust}(\G(\BG,\M_2)))$.}
\label{fig:causal ordering via perfect matchings}
\end{figure}
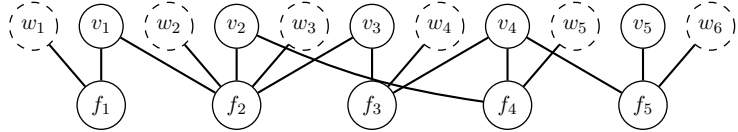
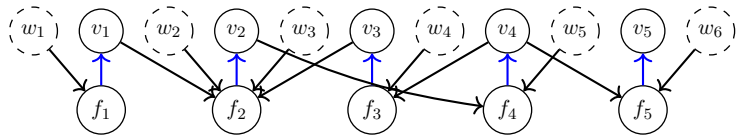
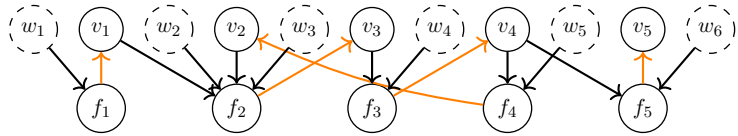
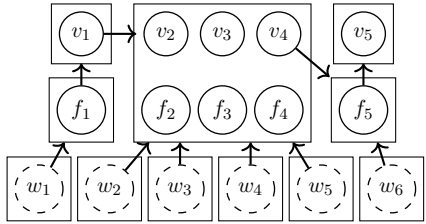
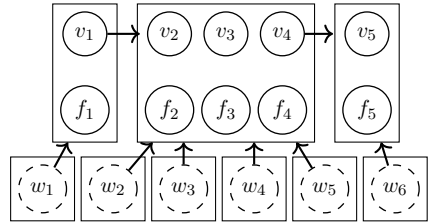

\subsection{Coarse decomposition via maximum matchings}

The extension that we propose relies on the \emph{coarse decomposition} of bipartite graphs in \citet{Pothen1990}, which was originally proposed by \citet{Dulmage1958}. The main idea is that a set of equations (i.e.\ a system of constraints) can be divided into an \emph{incomplete} part that has fewer equations than variables, an \emph{over-complete} part that has more equations than variables, and a part that is self-contained. The \emph{coarse decomposition} in Definition \ref{def:coarse decomposition} below uses the notions of a \emph{maximum matching} and an \emph{alternating path} for a maximum matching. The former is a matching so that there are no matchings with a greater cardinality, while the latter is a sequence of distinct vertices and edges $(v_1,e_1,v_2,e_2,\ldots,e_{n-1},v_n)$ so that edges $e_i$ are alternatingly in and out a maximum matching $M$. Proposition \ref{prop:coarse decomposition unique} by \citet{Pothen1985} shows that the coarse decomposition is unique.\footnote{For completeness, we have included a proof of this theorem in Appendix~\ref{appendix:coarse decomposition proofs}.} In this section we loosely follow the exposition of the coarse decomposition in \citet{Diepen2019} and \citet{Pothen1990}.

\begin{definition}
	\label{def:coarse decomposition}
	Let $M$ be a maximum matching for a bipartite graph $\BG=\tuple{V,F,E}$ and let $V_{\mathrm{un}}$ and $F_{\mathrm{un}}$ denote the unmatched vertices in $V$ and $F$ respectively. The \emph{incomplete set} $T_I\subseteq V\cup F$ and \emph{overcomplete set} $T_O\subseteq V\cup F$ are given by:
	\begin{align*}
	T_I:=\{x\in V\cup F: \text{ there is an alternating path between } x \text{ and some } y\in V_{\mathrm{un}}\},\\
	T_O:=\{x\in V\cup F: \text{ there is an alternating path between } x \text{ and some } y\in F_{\mathrm{un}}\}.
	\end{align*}
	The \emph{complete set} is given by $T_C=V\cup F\setminus (T_I\cup T_O)$. The \emph{coarse decomposition} $\mathrm{CD}(\BG,M)$ is given by $\tuple{T_I,T_C,T_O}$. The \emph{incomplete graph} $\BG_I$ is the subgraph of $\BG$ induced by vertices $T_I$, the \emph{complete graph} $\BG_C$ is the subgraph of $\BG$ induced by vertices $T_C$, and the \emph{overcomplete graph} $\BG_O$ is the subgraph of $\BG$ induced by vertices $T_O$.
\end{definition}

Note that $T_I$ and $T_O$ are necessarily disjoint, for more details see Lemma \ref{lemma:disjoint coarse decomposition} in Appendix~\ref{appendix:coarse decomposition proofs}.

\begin{restatable}{proposition}{coarsedecompositionunique}[\citet{Pothen1985}]
	\label{prop:coarse decomposition unique}
	The coarse decomposition of a bipartite graph $\BG$ is independent of the choice of the maximum matching.
\end{restatable}

There exist fast algorithms that are able to find a maximum matching in a bipartite graph $\BG=\tuple{V,F,E}$, such as the Hopcraft-Karp-Karzanov algorithm, which runs in $\mathcal{O}(|E|\sqrt{|V\cup F|})$ time \citep{Hopcroft1973, Karzanov1973}. In the following example we manually searched for maximum matchings to illustrate the result in Proposition \ref{prop:coarse decomposition unique} that the coarse decomposition is unique.

\begin{example}
	\label{ex:decomposition}	
	Consider the bipartite graph $\BG'$ in Figure \ref{fig:bipartite graph}, which has the following four maximum matchings.
	\begin{align}
	M_1 = \{(v_1 - f_2), (v_2 - f_3), (v_3 - f_4), (v_4 - f_5)\}, \\
	M_2 = \{(v_1 - f_1), (v_2 - f_3), (v_3 - f_4), (v_5 - f_5)\}, \\
	M_3 = \{(v_1 - f_2), (v_2 - f_3), (v_3 - f_4), (v_5 - f_5)\}, \\
	M_4 = \{(v_1 - f_1), (v_2 - f_3), (v_3 - f_4), (v_4 - f_5)\}.
	\end{align}
	By Proposition \ref{prop:coarse decomposition unique} we know that the coarse decomposition $\mathrm{CD}(\BG',M)$, with $M\in\{M_1,M_2,M_3,M_4\}$, does not depend on the choice of maximum matching. The coarse decomposition is displayed in Figure \ref{fig:coarse decomposition}. It is a straightforward exercise to verify that applying Definition \ref{def:coarse decomposition} to each of the maximum matchings results in the same coarse decomposition. Note that if the vertices $\{f_1,\ldots f_5\}$ are associated with equations, and the vertices $\{v_1,\ldots,v_5\}$ are associated with variables, then the incomplete graph $\BG_I$ has fewer equations than variables, whereas the over-complete graph $\BG_O$ has more equations than variables. The complete graph $\BG_C$ is self-contained. 
\end{example}

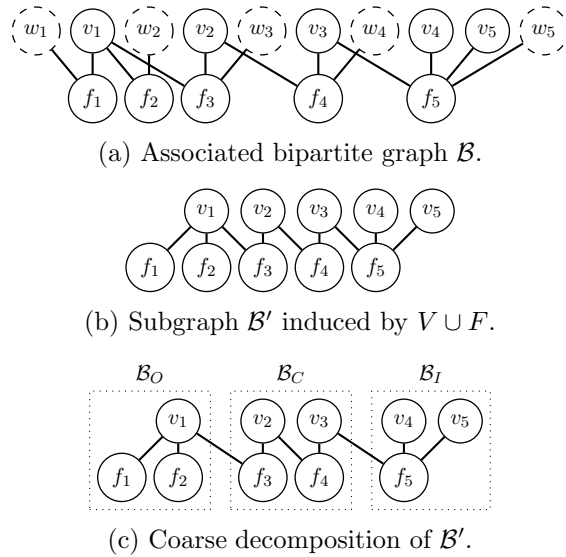
\begin{figure}[ht]
\begin{subfigure}{\textwidth}
	\centering
	\begin{tikzpicture}[scale=0.75,every node/.style={transform shape}]
	\GraphInit[vstyle=Normal]
	\SetGraphUnit{1}
	\Vertex[L=$w_1$, style={dashed}] {w1}
	\Vertex[L=$v_1$,x=1,y=0] {v1}
	\Vertex[L=$w_2$,x=2,y=0, style={dashed}] {w2}
	\Vertex[L=$v_2$,x=3,y=0] {v2}
	\Vertex[L=$w_3$,x=4,y=0, style={dashed}] {w3}
	\Vertex[L=$v_3$,x=5,y=0] {v3}
	\Vertex[L=$w_4$,x=6,y=0, style={dashed}] {w4}
	\Vertex[L=$v_4$,x=7,y=0] {v4}
	\Vertex[L=$v_5$,x=8,y=0] {v5}
	\Vertex[L=$w_5$,x=9,y=0, style={dashed}] {w5}
	\SetVertexMath
	\SO[unit=1.2](v1){f_1}
	\SO[unit=1.2](w2){f_2}
	\SO[unit=1.2](v2){f_3}
	\SO[unit=1.2](v3){f_4}
	\SO[unit=1.2](v4){f_5}
	\tikzset{EdgeStyle/.style = {-}}
	\Edge(f_1)(v1)
	\Edge(f_1)(w1)
	\Edge(f_2)(v1)
	\Edge(f_2)(w2)
	\Edge(f_3)(v1)
	\Edge(f_3)(v2)
	\Edge(f_3)(w3)
	\Edge(f_4)(v2)
	\Edge(f_4)(v3)
	\Edge(f_4)(w4)
	\Edge(f_5)(v3)
	\Edge(f_5)(v4)
	\Edge(f_5)(v5)
	\Edge(f_5)(w5)
	\end{tikzpicture}
	\caption{Associated bipartite graph $\BG$.}
	\label{fig:associated bipartite graph}
\end{subfigure}%
\vspace{2mm}	
\begin{subfigure}{\textwidth}
	\centering
	\begin{tikzpicture}[scale=0.75,every node/.style={transform shape}]
	\GraphInit[vstyle=Normal]
	\SetGraphUnit{1}
	\Vertex[L=$v_1$,x=0,y=0] {v1}
	\Vertex[L=$v_2$,x=1,y=0] {v2}
	\Vertex[L=$v_3$,x=2,y=0] {v3}
	\Vertex[L=$v_4$,x=3,y=0] {v4}
	\Vertex[L=$v_5$,x=4,y=0] {v5}
	\Vertex[L=$f_1$,x=-1,y=-1] {f_1}
	\SetVertexMath
	\SO[unit=1](v1){f_2}
	\SO[unit=1](v2){f_3}
	\SO[unit=1](v3){f_4}
	\SO[unit=1](v4){f_5}
	\tikzset{EdgeStyle/.style = {-}}
	\Edge(f_1)(v1)
	\Edge(f_2)(v1)
	\Edge(f_3)(v1)
	\Edge(f_3)(v2)
	\Edge(f_4)(v2)
	\Edge(f_4)(v3)
	\Edge(f_5)(v3)
	\Edge(f_5)(v4)
	\Edge(f_5)(v5)
	\end{tikzpicture}
	\caption{Subgraph $\BG'$ induced by $V\cup F$.}
	\label{fig:bipartite graph}
\end{subfigure}%
\vspace{2mm}	
\begin{subfigure}{\textwidth}
	\centering
	\begin{tikzpicture}[scale=0.75,every node/.style={transform shape}]
	\GraphInit[vstyle=Normal]
	\SetGraphUnit{1}
	\Vertex[L=$v_1$,x=0,y=0] {v1}
	\Vertex[L=$v_2$,x=1.5,y=0] {v2}
	\Vertex[L=$v_3$,x=2.5,y=0] {v3}
	\Vertex[L=$v_4$,x=4.0,y=0] {v4}
	\Vertex[L=$v_5$,x=5.0,y=0] {v5}
	\Vertex[L=$f_1$,x=-1,y=-1] {f_1}
	\SetVertexMath
	\SO[unit=1](v1){f_2}
	\SO[unit=1](v2){f_3}
	\SO[unit=1](v3){f_4}
	\SO[unit=1](v4){f_5}
	\node[draw=black, style={dotted}, fit=(v1) (f_1) (f_2), inner sep=0.1cm, label=$\BG_O$]{};
	\node[draw=black, style={dotted}, fit=(v2) (v3) (f_3) (f_4), inner sep=0.1cm, label=$\BG_C$]{};
	\node[draw=black, style={dotted}, fit=(v4) (v5) (f_5), inner sep=0.1cm, label=$\BG_I$]{};
	\tikzset{EdgeStyle/.style = {-}}
	\Edge(f_1)(v1)
	\Edge(f_2)(v1)
	\Edge(f_3)(v1)
	\Edge(f_3)(v2)
	\Edge(f_4)(v2)
	\Edge(f_4)(v3)
	\Edge(f_5)(v3)
	\Edge(f_5)(v4)
	\Edge(f_5)(v5)
	\end{tikzpicture}
	\caption{Coarse decomposition of $\BG'$.}
	\label{fig:coarse decomposition}
\end{subfigure}%
\caption{The bipartite graph $\BG$ associated with the system of equations in Example \ref{ex:not self-contained system of equations} is given in Figure \ref{fig:associated bipartite graph}. Its subgraph $\BG'$ induced by $V=\{v_1,\ldots,v_5\}$ and $F=\{f_1,\ldots,f_5\}$ in Figure \ref{fig:bipartite graph} is not self-contained. The coarse decomposition of $\BG'$ is given in Figure \ref{fig:coarse decomposition}.}
\label{fig:decomposition}	
\end{figure}

\subsection{Causal ordering via coarse decomposition}
\label{sec:causal ordering via coarse decomposition}

Here we present the extended causal ordering algorithm. It relies on the unique coarse decomposition of a bipartite graph into its incomplete, complete, and over-complete parts. Lemma \ref{lemma:complete graph self-contained}, due to \citet{Pothen1985}, shows that the complete graph has a perfect matching. Together, Lemma \ref{lemma:complete graph self-contained} and Lemma \ref{lemma:impossible edges} justify the steps in Algorithm \ref{alg:causal ordering coarse decomposition} to construct a causal ordering graph. The proofs are provided in Appendix~\ref{appendix:coarse decomposition proofs}.

\begin{restatable}{lemma}{completegraphselfcontained}[\citet{Pothen1985}]
\label{lemma:complete graph self-contained}	
Let $\BG$ be a bipartite graph with coarse decomposition $\tuple{T_I,T_C,T_O}$. The subgraph $\BG_C$ of $\BG$ induced by vertices in $T_C$ has a perfect matching and is self-contained.	
\end{restatable}

\begin{restatable}{lemma}{impossibleedges}[\citet{Pothen1985}]
\label{lemma:impossible edges}
Let $\BG=\tuple{V,F,E}$ be a bipartite graph with a maximum matching $M$. Let $\mathrm{CD}(\BG,M)=\tuple{T_I,T_C,T_O}$ be the associated coarse decomposition. No edge joins a vertex in $T_I\cap V$ with a vertex in $(T_C\cup T_O)\cap F$ and no edge joins a vertex in $T_C\cap V$ with a vertex in $T_O\cap F$.
\end{restatable}

Algorithm \ref{alg:causal ordering coarse decomposition} takes a set of exogenous vertices $W\subseteq V$ and a bipartite graph $\BG=\tuple{V, F, E}$ as input. The output is a causal ordering graph $\tuple{\V,\E}$. The algorithm first uses a maximum matching $M$ for the subgraph $\BG'$ of $\BG$ induced by $(V\setminus W)\cup F$ to construct the coarse decomposition $\tuple{T_I,T_C,T_O}$ of $\BG'$. Since the complete graph $\BG_C$ is self-contained (by Lemma \ref{lemma:complete graph self-contained}) the causal ordering algorithm for self-contained bipartite graphs can be applied to obtain the directed cluster graph $\mathrm{CO}(\BG_C)=\tuple{\V_C,\E_C}$. The cluster set $\V$ consists of the clusters in $\V_C$ and the connected components in $\BG_I$ and $\BG_O$. The edge set $\E$ contains all edges in $\E_C$. For edges between vertices $v\in T_O\cap V$ and $f\in T_C\cap F$ in $\BG$ an edge $(v\to \mathrm{cl}_{\V}(f))$ is added to $\E$.\footnote{Note that $\mathrm{cl}_{\V}(x)$ denotes the cluster in the partition $\V$ that contains the vertex $x$.} Similarly, for edges between vertices $v\in (T_O\cup T_C)\cap V$ and $f\in T_I\cap F$ an edge $(v\to \mathrm{cl}_{\V}(f))$ is also added to $\E$. By Lemma \ref{lemma:impossible edges} there are no other edges between the incomplete, complete, and over-complete graphs. Finally, edges from exogenous vertices $W$ are added to the causal ordering graph. The details can be found in Algorithm \ref{alg:causal ordering coarse decomposition}.

\begin{center}
\begin{algorithm}[H]
	\DontPrintSemicolon
	\KwIn{a set of exogenous vertices $W$, a bipartite graph $\BG=\tuple{V\cup W,F,E}$.}
	\KwOut{directed cluster graph $\tuple{\mathcal{V},\mathcal{E}}$}
  $\BG'\leftarrow$ subgraph of $\BG$ induced by $(V\setminus W)\cup F$\;
	$\M\leftarrow$ maximum matching for $\BG'$\;
	$\tuple{T_I,T_C,T_O} \leftarrow \mathrm{CD}(\BG',\M)$ {\color{gray}\tcp*{coarse decomposition}}
	$\BG_C \leftarrow$ subgraph of $\BG'$ induced by $T_C$ \;
	$\BG_I \leftarrow$ subgraph of $\BG'$ induced by $T_I$ \;
	$\BG_O \leftarrow$ subgraph of $\BG'$ induced by $T_O$ \;
	$\tuple{\V_C,\E_C} \leftarrow $ causal ordering graph for $\BG_C$ {\color{gray}\tcp*{construct clusters}}
	$\V_I \leftarrow$ partition of $T_I$ into connected components in $\BG_I$ \;
	$\V_O \leftarrow$ partition of $T_O$ into connected components in $\BG_O$ \;
	$\V\leftarrow \V_I\cup \V_C\cup \V_O \cup \{\{w\} : w\in W\}$ \;
	$\E \leftarrow \E_C$  {\color{gray}\tcp*{find edges}}
	\For{$(v-f)\in E$} {
		\If{$v\in (T_O\cup T_C)\cap V$ \textbf{\upshape and} $f\in T_I\cap F$}{
			Add $(v\to\mathrm{cl}_{\V}(f))$ to $\E$	
		}
		\ElseIf{$v\in T_O\cap V$ \textbf{\upshape and} $f\in T_C\cap F$}{	
			Add $(v\to\mathrm{cl}_{\V}(f))$ to $\E$
		}
	}
	\For{$w\in W$}{ 	
		add $(w \to \mathrm{cl}_{\V}(f))$ to $\E$ for all $f\in \adj{\BG}{w}$ {\color{gray}\tcp*{exogenous vertices}}
	}
	\label{alg:causal ordering coarse decomposition}
	\caption{Causal ordering via coarse decomposition.}
\end{algorithm}
\end{center}

\begin{corollary}
\label{coro:uniqueness extended causal ordering algorithm}
The output of Algorithm \ref{alg:causal ordering coarse decomposition} is well-defined and unique.
\end{corollary}
\begin{proof}
This follows directly from Theorem \ref{thm:uniqueness} and Proposition \ref{prop:coarse decomposition unique}.
\end{proof}

Corollary \ref{coro:uniqueness extended causal ordering algorithm} shows that the output of causal ordering via coarse decomposition does not depend on the choice of the maximum matching (i.e.\ the output is unique). The following example provides a manual demonstration of the causal ordering algorithm via the coarse decomposition.

\begin{example}
\label{ex:not self-contained system of equations}
We apply the causal ordering algorithm via coarse decomposition (i.e.\ Algorithm \ref{alg:causal ordering coarse decomposition}) to the bipartite graph in Figure \ref{fig:associated bipartite graph}. Its subgraph induced by endogenous variables and equations is the bipartite graph in Figure \ref{fig:bipartite graph} and its coarse decomposition is given in Figure \ref{fig:coarse decomposition}.  Since, by Lemma \ref{lemma:complete graph self-contained}, $\BG_C$ is self-contained we can apply the causal ordering algorithm (Algorithm \ref{alg:causal ordering minimal self-contained exo}) to the complete subgraph resulting in the directed cluster graph $\mathrm{CO}(\BG_C)=\tuple{\V_C,\E_C}$ where $\V_C=\{\{v_2,f_3\}, \{v_3,f_4\}\}$ and $\E_C=\{(v_2\to \{v_3,f_4\})\}$. The cluster set is then given by $\V=\V_C\cup \{\{v_4,v_5,f_5\}\} \cup \{ \{v_1,f_1,f_2\} \}$. We then add singleton clusters $\{w_1\}$, $\{w_2\}$, $\{w_3\}$, $\{w_4\}$, $\{w_5\}$ for each exogenous vertex. Next we add the edges $\E_C$, $(v_1\to\{\{v_2,f_3\})$ and $(v_3\to \{v_4,v_5,f_5\})$ to the edge set $\E$. Finally, we add edges $(w_1\to \{v_1,f_1,f_2\})$, $(w_2\to \{v_1,f_1,f_2\})$, $(w_3\to \{v_2, f_3\})$, $(w_4\to \{v_3,f_4\})$ and $(w_5\to \{v_4,v_5,f_5\})$ to the edge set $\E$. The resulting causal ordering graph $\mathrm{CO}(\BG)=\tuple{\V,\E}$ is given in Figure \ref{fig:causal ordering graph}.
\end{example}

\begin{figure}[ht]
\centering
\begin{tikzpicture}[scale=0.75,every node/.style={transform shape}]
\GraphInit[vstyle=Normal]
\SetGraphUnit{1}
\Vertex[L={$w_1$}, x=-2.5,y=-0.5, style={dashed}] {w_1}
\Vertex[L=$v_1$,x=0,y=0] {v1}
\Vertex[L=$v_2$,x=1.5,y=0] {v2}
\Vertex[L=$v_3$,x=3.0,y=0] {v3}
\Vertex[L=$v_4$,x=4.5,y=0] {v4}
\Vertex[L=$v_5$,x=5.5,y=0] {v5}
\Vertex[L=$w_5$,x=7.0,y=-0.5, style={dashed}] {w_5}
\Vertex[L=$f_1$,x=-1,y=-1] {f_1}
\Vertex[L=$w_2$, x=-0.5, y=-2.5, style={dashed}] {w_2}
\Vertex[L=$w_3$, x=1.5, y=-2.5, style={dashed}] {w_3}
\Vertex[L=$w_4$, x=3.0, y=-2.5, style={dashed}] {w_4}
\SetVertexMath
\SO[unit=1](v1){f_2}
\SO[unit=1](v2){f_3}
\SO[unit=1](v3){f_4}
\SO[unit=1](v4){f_5}
\node[draw=black, fit=(v1) (f_1) (f_2), inner sep=0.1cm]{};
\node[draw=black, fit=(v2) (f_3), inner sep=0.1cm]{};
\node[draw=black, fit=(v3) (f_4), inner sep=0.1cm]{};
\node[draw=black, fit=(v4) (v5) (f_5), inner sep=0.1cm]{};
\node[draw=black, fit=(w_1), inner sep=0.1cm]{};
\node[draw=black, fit=(w_2), inner sep=0.1cm]{};
\node[draw=black, fit=(w_3), inner sep=0.1cm]{};
\node[draw=black, fit=(w_4), inner sep=0.1cm]{};
\node[draw=black, fit=(w_5), inner sep=0.1cm]{};
\draw[EdgeStyle, style={->}](w_1) to (-1.6,-0.5);
\draw[EdgeStyle, style={->}](v1) to (0.9,0.0);
\draw[EdgeStyle, style={->}](v2) to (2.4,0.0);
\draw[EdgeStyle, style={->}](v3) to (3.9,0.0);
\draw[EdgeStyle, style={->}](w_2) to (-0.5,-1.6);
\draw[EdgeStyle, style={->}](w_3) to (1.5,-1.6);
\draw[EdgeStyle, style={->}](w_4) to (3.0,-1.6);
\draw[EdgeStyle, style={->}](w_5) to (6.1,-0.5);
\end{tikzpicture}
\caption{Causal ordering graph for the bipartite graph in Figure \ref{fig:associated bipartite graph}.}
\label{fig:causal ordering graph}
\end{figure}
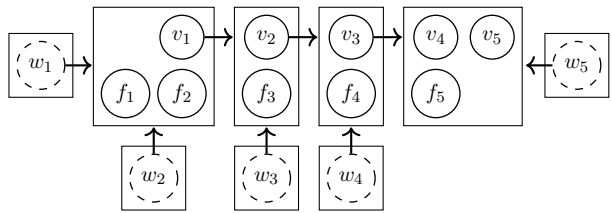

\section{Markov ordering graph}
\label{sec:markov properties}

First we consider (unique) solvability assumptions for systems of constraints. We will then construct the \emph{Markov ordering graph} and prove that it implies conditional independences between variables that appear in constraints. We also apply our method to the model for the filling bathtub in Example \ref{ex:bathtub intro}. Finally, we present a novel result regarding the generalized directed global Markov property for solutions of systems of constraints and an \emph{associated directed graph}. 

\subsection{Solvability for systems of constraints}
\label{sec:solvability for systems of constraints}

In this section, we consider (unique) solutions of systems of constraints with exogenous random variables, and give a sufficient condition under which the output of the causal ordering algorithm can be interpreted as the order in which sets of (endogenous) variables can be solved in a set of equations (i.e.\ constraints).

\begin{definition}
\label{def:solutions of systems of constraints}
We say that a measurable mapping $\B{g} : \B{\X}_W \mapsto \B{\X}_{V\setminus W}$ that maps
values of the exogenous variables to values of the endogenous variables is a 
\emph{solution to a system of constraints $\tuple{\B{\X},\B{X}_W,\B{\Phi}, \BG}$} if
\begin{align*}
  \phi_f(\B{g}_{V(f) \setminus W}(\B{X}_W),\B{X}_{V(f) \cap W})=c_f, \quad \forall\, f\in F, \quad \P_{\B{X}_W}\text{-a.s.}
\end{align*}
We say that the system of constraints is \emph{uniquely} solvable (or ``has a unique solution'') if all its solutions are $\P_{\B{X}_W}$-a.s.\ equal.
\end{definition}

The system of constraints in the example below is solvable but not uniquely solvable. The example illustrates that the dependence or independence between solution components (i.e.\ endogenous variables) is not the same for all solutions. 

\begin{example}
	Consider a system of constraints $\tuple{\B{\X},\B{X}_W,\B{\Phi}, \BG}$ with $\B{\X}=\R^4$ and independent exogenous random variables $\B{X}_W=(X_w)_{w\in \{w_1,w_2\}}$ taking value in $\R^2$. Suppose that $\B{\Phi}$ consists of the constraints 
	\begin{align}
	\Phi_{f_1}&=\tuple{X_{V(f_1)} \mapsto X_{v_1}-X_{w_1}, 0 ,\{v_1,w_1\}}, \\
	\Phi_{f_2}&=\tuple{X_{V(f_2)} \mapsto X_{v_2}^2 - |X_{w_2}|, 0, \{v_2,w_2\}}.
	\end{align}
  This system of constraints has solutions with different distributions. One solution is given by $(X_{v_1}^*,X_{v_2}^*) = (X_{w_1}, \sqrt{|X_{w_2}|})$ and another solution is $(X'_{v_1}, X'_{v_2}) = (X_{w_1}, \mathrm{sgn}(X_{w_1})\sqrt{|X_{w_2}|})$. Note that the solution components $X^*_{v_1}$ and $X^*_{v_2}$ are independent, whereas the solution components $X'_{v_1}$ and $X'_{v_2}$ may be dependent.
\end{example}

Underspecified (and overspecified) systems of constraints can be avoided by the requirement that it is \emph{uniquely solvable}. In Definition \ref{def:unique solvability} below we give a sufficient condition under which a unique solution can be obtained by solving variables in clusters from equations in these clusters.

\begin{definition}
	\label{def:solvability}
	A system of constraints $\mathcal{M}=\tuple{\B{\X}, \B{X}_W,\B{\Phi}, \BG}$ is \emph{solvable w.r.t.\ constraints $S_F\subseteq F$ and endogenous variables $S_V\subseteq V(S_F)\setminus W$} if there exists a measurable function $\B{g}_{S_V}:\B{\X}_{V(S_F)\setminus S_V}\to \B{\X}_{S_V}$ s.t.\ $\P_{\B{X}_W}$-a.s., for all $\B{x}_{V(S_F)\setminus W}\in \B{\X}_{V(S_F)\setminus W}$:
	\begin{align*}
    \phi_f(\B{x}_{V(f) \setminus W},\B{X}_{V(f) \cap W})=c_f, \;\; \forall\, f\in S_F \;\impliedby\; \B{x}_{S_V}=\B{g}_{S_V}(\B{x}_{V(S_F)\setminus (S_V \cup W)},\B{X}_{V(S_F) \cap W}).
	\end{align*}
	\label{def:unique solvability}
  $\mathcal{M}$ is \emph{uniquely solvable w.r.t.\ constraints $S_F$ and endogenous variables $S_V$} if the converse implication also holds.
\end{definition}

The following condition suffices for the existence of a unique solution that can be obtained by solving for variables from equations in their cluster along a topological ordering of the clusters in the causal ordering graph. This weakens the assumptions made in \citet{Simon1953} who requires both unique solvability w.r.t.\ every subset of equations (and the endogenous variables that appear in them) and self-containedness of the bipartite graph.

\begin{definition}
	\label{def:solvability wrt co graph}
	We say that $\mathcal{M}$ is \emph{uniquely solvable w.r.t.\ the causal ordering graph} $\mathrm{CO}(\BG)=\tuple{\mathcal{V},\mathcal{E}}$ if it is uniquely solvable w.r.t.\ $S\cap F$ and $S\cap V$ for all $S\in\mathcal{V}$ with $S\cap W=\emptyset$.
\end{definition}

For systems of constraints for cyclic models or with non-linear equations, for which the incomplete subgraph is not the empty graph, the condition of unique solvability with respect to the causal ordering graph is too strong. This is illustrated by Example \ref{ex:solvability incomplete graph} below.

\begin{example}
\label{ex:solvability incomplete graph}
Let $V= \{v_1,\ldots v_5\}$ be an index set for endogenous variables $X_{v_1},\ldots,X_{v_5}$ taking value in $\mathbb{R}$, $W=\{w_1,\ldots,w_5\}$ an index set for independent exogenous random variables $U_{w_1},\ldots,U_{w_5}$ taking value in $\mathbb{R}$, and $p_1,p_2$ parameters with values in $\mathbb{R}$. Consider the following non-linear system of constraints:
\begin{align}
\Phi_{f_1}:&& X_{v_1}^2 - U_{w_1} &= 0, \\
\Phi_{f_2}:&& \mathrm{sgn}(X_{v_1}) - \mathrm{sgn}(U_{w_2}) &= 0, \\
\Phi_{f_3}:&& X_{v_2} - p_1 X_{v_1} - U_{w_3} &= 0, \\
\Phi_{f_4}:&& X_{v_3} - p_2 X_{v_2} - U_{w_4} &= 0, \\
\Phi_{f_5}:&& X_{v_3} + X_{v_4} + X_{v_5} - U_{w_5} &= 0.
\end{align}
The associated bipartite graph $\BG$ is given in Figure \ref{fig:associated bipartite graph} and the corresponding causal ordering graph is given in Figure \ref{fig:causal ordering graph}. It is easy to check that the system of constraints is uniquely solvable with respect to the clusters $\{v_1,f_1,f_2\}$, $\{v_2,f_3\}$, and $\{v_3,f_4\}$ in the causal ordering graph. Equation $f_5$ does not provide a unique solution for the variables $v_4$ and $v_5$ and hence the system is not uniquely solvable with respect to the cluster $\{v_4,v_5,f_5\}$.
\end{example}

Generally speaking, systems of constraints are not uniquely solvable with respect to the clusters in the incomplete set of vertices in the associated bipartite graph. In order to derive a Markov property for the complete and overcomplete sets of vertices in the associated bipartite graph, we use the condition in Definition \ref{def:solvability overcomplete} below, which is slightly weaker than the one in Definition \ref{def:solvability}. Since self-contained bipartite graphs do not have an incomplete part there is no difference between the two conditions in that case.

\begin{definition}
\label{def:solvability overcomplete}
Let $\mathcal{M}=\tuple{\B{\X}, \B{X}_W,\B{\Phi}, \BG}$ be a system of constraints. Denote its coarse decomposition by $\mathrm{CD}(\BG)=\tuple{T_I,T_C,T_O}$ and its causal ordering graph by $\mathrm{CO}(\BG) = \tuple{\mathcal{V},\mathcal{E}}$. We say that $\mathcal{M}$ is \emph{maximally uniquely solvable} if it is 
  \begin{enumerate}
    \item uniquely solvable w.r.t.\ $S \cap F$ and $S \cap V$ for all $S \in \mathcal{V}$ with $S\cap W=\emptyset$ and $S\cap T_I=\emptyset$, and
  	\item solvable with respect to $T_I \cap F$ and $(T_I \cap V) \setminus W$.
  \end{enumerate}
\end{definition}
This condition suffices to guarantee the existence of a solution, and that it is unique on the (over)complete part $(T_O \cup T_C) \cap V \setminus W$.

\subsection{Directed global Markov property via causal ordering}
\label{sec:directed global markov property via causal ordering}

The \emph{Markov ordering graph} is constructed from a causal ordering graph by \emph{declustering} and then marginalizing out the vertices that correspond to constraints.

\begin{definition}
\label{def:declustering}
Let $\G=\tuple{\mathcal{V},\mathcal{E}}$ be a directed cluster graph. The \emph{declustered} graph is given by $D(\G)=\tuple{V,E}$ with $V=\cup_{S\in\mathcal{V}} S$ and $(v\to w)\in E$ if and only if $(v\to \mathrm{cl}(w))\in \mathcal{E}$. For a system of constraints $\mathcal{M}=\tuple{\B{\X}, \B{X}_W, \B{\Phi}, \BG}$ with $\BG=\tuple{V,F,E}$, we say that $\mathrm{MO}(\BG)=D(\mathrm{CO}(\BG))_{\mathrm{mar}(F)}$ is the \emph{Markov ordering graph}.
\end{definition}

Under the assumption that systems of constraints are uniquely solvable with respect to the (over)complete part of their causal ordering graph, Theorem \ref{thm:markov property} relates d-separations between vertices in the Markov ordering graph to conditional independences between the corresponding components of a solution of the system of constraints.

\begin{restatable}{theorem}{markovproperty}
\label{thm:markov property}
Let $\B{X}^*=\B{h}(\B{X}_W)$ with $\B{h} : \B{\X}_W \to \B{\X}_{V\setminus W}$ be a solution of a system of constraints $\mathcal{M}=\tuple{\B{\X}, \B{X}_W, \B{\Phi}, \BG}$ with coarse decomposition $\mathrm{CD}(\BG)=\tuple{T_I,T_C,T_O}$. Let $\mathrm{MO}_{\mathrm{CO}}(\BG)$ denote the subgraph of the Markov ordering graph induced by $T_C\cup T_O$ and let $\B{X}_{\mathrm{CO}}^{*}$ denote the corresponding solution components. If $\mathcal{M}$ is maximally uniquely solvable then the pair $(\mathrm{MO}_{\mathrm{CO}}(\BG), \P_{\B{X}_{\mathrm{CO}}^{*}})$ satisfies the directed global Markov property (see Definition~\ref{def:markov property}).
\end{restatable}

In particular, when the incomplete and overcomplete sets are empty (i.e.\ when $T_I=\emptyset$ and $T_C=\emptyset$) and the system is uniquely solvable with respect to the causal ordering graph, Theorem \ref{thm:markov property} tells us that the pair $(\mathrm{MO}(\BG), \mathbb{P}_{\B{X}^*})$ satisfies the directed global Markov property.

\begin{example}
\label{ex:markov ordering incomplete}
Consider the system of constraints in Example \ref{ex:solvability incomplete graph}. The Markov ordering graph for the associated bipartite graph in Figure \ref{fig:associated bipartite graph} can be constructed from the causal ordering graph in Figure \ref{fig:causal ordering graph} and is given in Figure \ref{fig:coarse decomposition:markov ordering graph}. One can check that the system of constraints is uniquely solvable with respect to the clusters in the complete and overcomplete sets. The Markov ordering graph can be used to read off conditional independences from d-separations between vertices that are not in the incomplete part. For example, since $v_1$ is d-separated from $v_3$ given $v_2$, we deduce that $X_{v_1}\indep X_{v_3}\given X_{v_2}$, for any solution of the constraints.
\end{example}

\begin{figure}[ht]
\centering
\begin{subfigure}[b]{0.5\textwidth}%
\begin{tikzpicture}[scale=0.75,every node/.style={transform shape}]
\GraphInit[vstyle=Normal]
\SetGraphUnit{1.5}
\Vertex[L=$w_1$, style={dashed}] {w_1}
\SetVertexMath
\EA[unit=1.5](w_1){v_1}
\EA[unit=1.5](v_1){v_2}
\EA[unit=1.5](v_2){v_3}
\EA[unit=1.5](v_3){v_4}
\SO[unit=1.2](v_4){v_5}
\EA[unit=1.5, style={dashed}](v_4){w_5}
\SO[unit=1.2, style={dashed}](v_1){w_2}
\SO[unit=1.2, style={dashed}](v_2){w_3}
\SO[unit=1.2, style={dashed}](v_3){w_4}
\node[draw=black, style={dotted}, fit=(v_4) (v_5), inner sep=0.2cm, label=$T_I\cap V$]{};
\tikzset{EdgeStyle/.style = {->}}
\Edge(w_1)(v_1)
\Edge(w_2)(v_1)
\Edge(w_3)(v_2)
\Edge(w_4)(v_3)
\Edge(v_2)(v_3)
\Edge(v_1)(v_2)
\Edge(v_3)(v_4)
\Edge(v_3)(v_5)
\Edge(w_5)(v_4)
\Edge(w_5)(v_5)
\end{tikzpicture}
\caption{$\mathrm{MO}(\BG)$.}
\end{subfigure}
\begin{subfigure}[b]{0.4\textwidth}
\begin{tikzpicture}[scale=0.75,every node/.style={transform shape}]
\GraphInit[vstyle=Normal]
\SetGraphUnit{1.5}
\Vertex[L=$w_1$, style={dashed}] {w_1}
\SetVertexMath
\EA[unit=1.5](w_1){v_1}
\EA[unit=1.5](v_1){v_2}
\EA[unit=1.5](v_2){v_3}
\EA[unit=1.5, style={dashed}](v_3){w_5}
\SO[unit=1.2, style={dashed}](v_1){w_2}
\SO[unit=1.2, style={dashed}](v_2){w_3}
\SO[unit=1.2, style={dashed}](v_3){w_4}
\tikzset{EdgeStyle/.style = {->}}
\Edge(w_1)(v_1)
\Edge(w_2)(v_1)
\Edge(w_3)(v_2)
\Edge(w_4)(v_3)
\Edge(v_2)(v_3)
\Edge(v_1)(v_2)
\end{tikzpicture}
\caption{$\mathrm{MO}_{\mathrm{CO}}(\BG)$.}
\end{subfigure}
\caption{(a) The Markov ordering graph associated with the system of constraints in Example \ref{ex:solvability incomplete graph}. It can be constructed from the causal ordering graph in Figure \ref{fig:causal ordering graph}. The vertices in the incomplete graph are indicated by the dashed rectangle. (b) Its subgraph induced by $T_C \cap T_O$. Theorem \ref{thm:markov property} shows that d-separations in $\mathrm{MO}_{\mathrm{CO}}(\BG)$ imply conditional independences.}
\label{fig:coarse decomposition:markov ordering graph}
\end{figure}
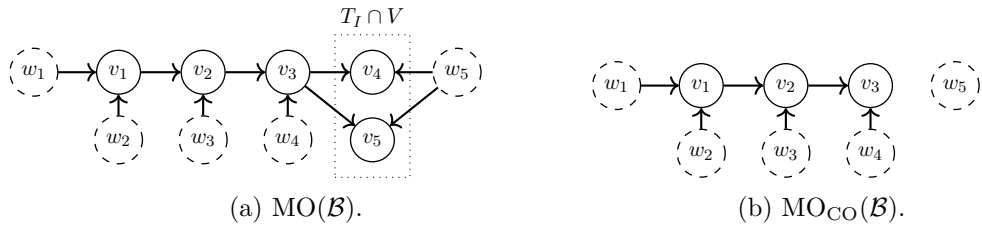

\subsection{Application to the filling bathtub}
\label{sec:application to bathtub}

In Example \ref{ex:bathtub intro} we informally described an equilibrium model for a filling bathtub. The endogenous variables of the system are the diameter $X_{v_K}$ of the drain, the rate $X_{v_I}$ at which water flows from the faucet, the water pressure $X_{v_P}$, the rate $X_{v_O}$ at which the water goes through the drain and the water level $X_{v_D}$. The model is formally represented by a system of constraints $\mathcal{M} = \tuple{ \B{\X}, \B{X}_W, \Phi, \BG}$ where:

\begin{enumerate}
	\item $\B{\X}=\mathbb{R}_{>0}^{12}$ is a product of standard measurable spaces corresponding to the domain of variables that are indexed by $\{v_K,v_I,v_P,v_O,v_D,w_K,w_I,w_1,\ldots, w_5\}$,
	\item $\B{X}_W = \{U_{w_I},U_{w_K},U_{w_1},\ldots,U_{w_5}\}$ is a family of independent exogenous random variables,
	\item $\B{\Phi}$ is a family of constraints:
	{\small
		\begin{align*}
      \Phi_{f_K} &= \langle X_{V(f_K)} \mapsto X_{v_K}-U_{w_K}, && 0, && V(f_K) = \{v_K,w_K\}\rangle, \\
      \Phi_{f_I} &= \langle X_{V(f_I)} \mapsto X_{v_I}-U_{w_I}, && 0, && V(f_I) = \{v_I,w_I\}\rangle, \\
      \Phi_{f_P} &= \langle  X_{V(f_P)} \mapsto U_{w_1}(g U_{w_2} X_{v_D} - X_{v_P}), && 0 , && V(f_P) = \{v_D,v_P,w_1,w_2\}\rangle, \\
      \Phi_{f_O} &= \langle  X_{V(f_O)} \mapsto U_{w_3}(U_{w_4} X_{v_K} X_{v_P} - X_{v_O}), && 0 , && V(f_O) = \{v_K, v_P, v_O, w_3, w_4\}\rangle, \\
      \Phi_{f_D} &= \langle  X_{V(f_D)} \mapsto U_{w_5} (X_{v_I}-X_{v_O}), && 0, && V(f_D) = \{v_I,v_O,v_5\}\rangle,		
		\end{align*}
	}
	\item The associated bipartite graph $\BG=\tuple{V,F,E}$ is as in Figure \ref{fig:bathtub full bipartite}. The vertices $F=\{f_K,f_I,f_P,f_O,f_D\}$ correspond to constraints and the vertices $V\setminus W = \{v_K,v_I,v_P,v_O,$ $v_D\}$ and $W=\{w_I, w_K, w_{1}, \ldots, w_{5}\}$ correspond to endogenous and exogenous variables respectively. Note that the subgraph induced by the endogenous vertices $V\setminus W$ is the self-contained bipartite graph presented in Figure \ref{fig:bathtub bipartite graph}.
\end{enumerate}

\begin{figure}[ht]
\centering
\begin{tikzpicture}[scale=0.75,every node/.style={transform shape}]
\GraphInit[vstyle=Normal]
\SetGraphUnit{1}
\Vertex[L=$w_{I}$, style={dashed}] {wI}
\Vertex[L=$v_I$,x=1,y=0] {v_I}
\Vertex[L=$w_{5}$,x=2,y=0, style={dashed}] {w5}
\Vertex[L=$v_O$,x=3,y=0] {v_O}
\Vertex[L=$w_{4}$,x=4,y=0, style={dashed}] {w4}
\Vertex[L=$v_D$,x=5,y=0] {v_D}
\Vertex[L=$w_{3}$,x=6,y=0, style={dashed}] {w3}
\Vertex[L=$w_{2}$,x=7,y=0, style={dashed}] {w2}
\Vertex[L=$v_P$,x=8,y=0] {v_P}
\Vertex[L=$w_{1}$,x=9,y=0, style={dashed}] {w1}
\Vertex[L=$v_K$,x=10,y=0] {v_K}
\Vertex[L=$w_{K}$,x=11,y=0, style={dashed}] {wK}
\Vertex[L=$f_O$, x=5.5, y=-1.5] {f_O}
\SetVertexMath
\SO[unit=1.5](v_I){f_I}
\SO[unit=1.5](v_O){f_D}
\SO[unit=1.5](v_P){f_P}
\SO[unit=1.5](v_K){f_K}
\tikzset{EdgeStyle/.style = {-}}
\Edge(f_I)(v_I)
\Edge(f_D)(v_I)
\Edge(f_D)(v_O)
\Edge[style={bend right=10}](f_O)(v_K)
\Edge(f_O)(v_P)
\Edge(f_O)(v_O)
\Edge(f_P)(v_D)
\Edge(f_P)(v_P)
\Edge(f_K)(v_K)
\Edge(wI)(f_I)
\Edge(w5)(f_D)
\Edge(w4)(f_O)
\Edge(w3)(f_O)
\Edge(w2)(f_P)
\Edge(w1)(f_P)
\Edge(wK)(f_K)
\end{tikzpicture}
\caption{The bipartite graph associated with the equilibrium equations of the bathtub system.}
\label{fig:bathtub full bipartite}
\end{figure}
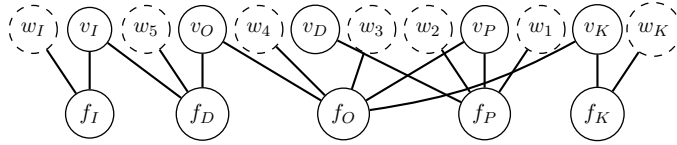

\paragraph{Solvability with respect to the causal ordering graph:}

Applying Algorithm \ref{alg:causal ordering minimal self-contained exo} to the bipartite graph results in the causal ordering graph $\mathrm{CO}(\BG)$ in Figure \ref{fig:bathtub full causal ordering}. Since the bipartite graph induced by the endogenous variables and equations is self-contained, there is no incomplete or overcomplete subgraph. The assumption of maximal unique solvability in Theorem \ref{thm:markov property} then reduces to the assumption of unique solvability with respect to the causal ordering graph. Through explicit calculations, it is easy to verify that $\mathcal{M}$ is (maximally) uniquely solvable with respect to $\mathrm{CO}(\BG)$, whenever $g\neq 0$:
\begin{enumerate}
	\item For the cluster $\{f_K,v_K\}$ we have that $X_{v_K}-U_{w_K} = 0 \iff X_{v_K} = U_{w_K}$.
	\item For the cluster $\{f_I,v_I\}$ we have that $X_{v_I}-U_{w_I} = 0 \iff X_{v_I}=U_{w_I}$.
	\item For $\{f_O,v_P\}$ we have that $U_{w_3}(U_{w_4} X_{v_K} X_{v_P} - X_{v_O})=0 \iff X_{v_P}=\frac{X_{v_O}}{U_{w_3} X_{v_K}}$.
	\item For $\{f_D,v_O\}$ we have that $U_{w_5} (X_{v_I}-X_{v_O}) \iff X_{v_O} = X_{v_I}$.
	\item For $\{f_P,v_D\}$ we have that $U_{w_1} (g U_{w_2} X_{v_D} - X_{v_P}) \iff X_{v_D} = \frac{X_{v_P}}{g U_{w_2}}$.
\end{enumerate}

In practice, we do not always need to manually check the assumption of unique solvability with respect to the causal ordering graph. For example, in linear systems of equations of the form $\B{A} \B{X} = \B{Y}$, we may use the fact that this assumption is satisfied when the matrix of coefficients $\B{A}$ is invertible. More generally, global implicit function theorems give conditions under which (non-linear) systems of equations have a unique solution \citep{Krantz2013}.\footnote{In particular, Hadamard's global implicit function theorem in \citet{Krantz2013} states the following \citep{Hadamard1906}. Let $\B{f}:\mathbb{R}^n\mapsto \mathbb{R}^n$ be a $C^2$ mapping. Suppose that $\B{f}(\B{0})=\B{0}$ and that the Jacobian determinant is non-zero at each point. Further suppose that whenever $K\subseteq \mathbb{R}^n$ is compact then $\B{f}^{-1}(K)$ is compact (i.e.\ $\B{f}$ is proper). Then $\B{f}$ is one-to-one and onto. In the literature, several conditions have been formulated yielding global inverse theorems in different or more general settings, see for example \citet{Idczak2016, Gutu2017}.} We consider detailed analysis of conditions under which (maximal) unique solvability is guaranteed to be outside the scope of this paper. Note that, under the assumption of (maximal) unique solvability, the conditional independences can be read off from the Markov ordering graph without the requirement of calculating explicit solutions.

\begin{figure}[ht]
\centering
\begin{tikzpicture}[scale=0.75,every node/.style={transform shape}]
\GraphInit[vstyle=Normal]
\SetGraphUnit{1}
\Vertex[L=$w_{I}$, style={dashed}, x=-1.5, y=0] {wI}
\Vertex[L=$w_{1}$, style={dashed}, x=1.5, y=0.2] {w1}
\Vertex[L=$w_{2}$, style={dashed}, x=1.5, y=-1.3] {w2}
\Vertex[L=$w_{3}$, style={dashed}, x=6.7, y=0.2] {w3}
\Vertex[L=$w_{4}$, style={dashed}, x=6.7, y=-1.3] {w4}
\Vertex[L=$w_{5}$, style={dashed}, x=6.7, y=1.7] {w5}
\Vertex[L=$w_{K}$, style={dashed}, x=6.7, y=-2.8] {wK}
\node[draw=black, fit=(wI), inner sep=0.1cm ]{};
\node[draw=black, fit=(w1), inner sep=0.1cm ]{};
\node[draw=black, fit=(w2), inner sep=0.1cm ]{};
\node[draw=black, fit=(w3), inner sep=0.1cm ]{};
\node[draw=black, fit=(w4), inner sep=0.1cm ]{};
\node[draw=black, fit=(w5), inner sep=0.1cm ]{};
\node[draw=black, fit=(wK), inner sep=0.1cm ]{};
\draw[EdgeStyle, style={->}] (wI) to (-0.6,0);
\draw[EdgeStyle, style={->}] (w1) to (2.6,0.2);
\draw[EdgeStyle, style={->}] (w2) to (2.6,-1.3);
\draw[EdgeStyle, style={->}] (w3) to (5.6,0.2);
\draw[EdgeStyle, style={->}] (w4) to (5.6,-1.3);
\draw[EdgeStyle, style={->}] (w5) to (5.6,1.7);
\draw[EdgeStyle, style={->}] (wK) to (5.6,-2.8);	
\SetVertexMath
\Vertex{v_I}
\EA[unit=3.2](v_I){v_D}
\EA[unit=1.8](v_D){v_P}
\SO(v_I){f_I}
\SO(v_D){f_P}
\SO(v_P){f_O}
\NO[unit=1.5](v_P){v_O}
\WE(v_O){f_D}
\SO[unit=2.5](v_P){v_K}
\WE(v_K){f_K}
\node[draw=black, fit=(v_I) (f_I), inner sep=0.1cm ]{};
\node[draw=black, fit=(v_D) (f_P), inner sep=0.1cm ]{};
\node[draw=black, fit=(v_P) (f_O), inner sep=0.1cm ]{};
\node[draw=black, fit=(v_O) (f_D), inner sep=0.1cm ]{};
\node[draw=black, fit=(v_K) (f_K), inner sep=0.1cm ]{};	
\draw[EdgeStyle, style={-}](v_I) to (0,1.5);
\draw[EdgeStyle, style={->}](0,1.5) to (3.4,1.5);
\draw[EdgeStyle, style={->}](v_O) to (5,0.55);
\draw[EdgeStyle, style={->}](v_P) to (3.8,0);
\draw[EdgeStyle, style={->}](v_K) to (5,-1.6);
\end{tikzpicture}
\caption{The causal ordering graph for the equilibrium equations of the bathtub system.}
\label{fig:bathtub full causal ordering}
\end{figure}
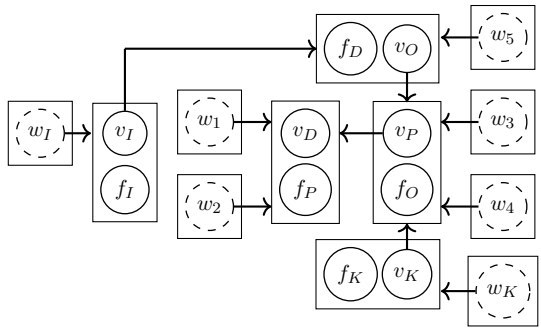

\paragraph{Markov ordering graph:}

Application of declustering and marginalization of vertices in $F$, as in Definition \ref{def:declustering}, to the causal ordering graph in Figure \ref{fig:bathtub full causal ordering} results in the Markov ordering graph in Figure \ref{fig:bathtub full markov ordering}. Since $\mathcal{M}$ is uniquely solvable with respect to $\mathrm{CO}(\BG)$, Theorem \ref{thm:markov property} tells us that the pair $(\mathrm{MO}(\BG), \mathbb{P}_{\B{X}^*})$ satisfies the directed global Markov property, where $\B{X}^*$ is a solution of $\mathcal{M}$.\footnote{Recall that there is no incomplete and overcomplete part of the bipartite graph. Therefore we have that $\mathrm{MO}_{\mathrm{CO}}(\BG) = \mathrm{MO}(\BG)$.}

\begin{figure}[ht]
\begin{subfigure}[b]{0.5\textwidth}
	\centering
	\begin{tikzpicture}[scale=0.75,every node/.style={transform shape}]
	\GraphInit[vstyle=Normal]
	\SetGraphUnit{1}
	\Vertex[L=$w_{I}$, style={dashed}, x=-1.2, y=0] {wI}
	\Vertex[L=$w_{1}$, style={dashed}, x= 1.2, y=0.6] {w1}
	\Vertex[L=$w_{2}$, style={dashed}, x=1.2, y=-0.6] {w2}
	\Vertex[L=$w_{3}$, style={dashed}, x=4.8, y=0.6] {w3}
	\Vertex[L=$w_{4}$, style={dashed}, x=4.8, y=-0.6] {w4}
	\Vertex[L=$w_{5}$, style={dashed}, x=4.8, y=1.8] {w5}
	\Vertex[L=$w_{K}$, style={dashed}, x=4.8, y=-1.8] {wK}
	\SetVertexMath
	\Vertex{v_I}
	\EA[unit=2.4](v_I){v_D}
	\EA[unit=1.2](v_D){v_P}
	\NO[unit=1.2](v_P){v_O}
	\SO[unit=1.2](v_P){v_K}
	\draw[EdgeStyle, style={-}](v_I) to (0,1.2);
	\draw[EdgeStyle, style={->}](0,1.2) to (v_O);
	\draw[EdgeStyle, style={->}](v_O) to (v_P);
	\draw[EdgeStyle, style={->}](v_P) to (v_D);
	\draw[EdgeStyle, style={->}](v_K) to (v_P);
	\draw[EdgeStyle, style={->}](wI) to (v_I);
	\draw[EdgeStyle, style={->}](w1) to (v_D);
	\draw[EdgeStyle, style={->}](w2) to (v_D);
	\draw[EdgeStyle, style={->}](w3) to (v_P);
	\draw[EdgeStyle, style={->}](w4) to (v_P);
	\draw[EdgeStyle, style={->}](w5) to (v_O);
	\draw[EdgeStyle, style={->}](wK) to (v_K);
	\end{tikzpicture}
	\caption{Markov ordering graph.}
	\label{fig:bathtub full markov ordering}
\end{subfigure}%
\begin{subfigure}[b]{0.5\textwidth}
	\centering
	\begin{tikzpicture}[scale=0.75,every node/.style={transform shape}]
	\GraphInit[vstyle=Normal]
	\SetGraphUnit{1}
	\SetVertexMath
	\Vertex{v_I}
	\EA[unit=1.2](v_I){v_D}
	\EA[unit=1.2](v_D){v_P}
	\NO[unit=1.2](v_P){v_O}
	\NO[unit=1.2](v_O){v_K}
	\draw[EdgeStyle, style={->}](v_I) to (v_D);
	\draw[EdgeStyle, style={->}](v_D) to (v_P);
	\draw[EdgeStyle, style={->}](v_P) to (v_O);
	\draw[EdgeStyle, style={->}](v_O) to (v_D);
	\draw[EdgeStyle, style={->}](v_K) to (v_O);
	\draw[EdgeStyle, style={->,loop above}](v_D) to (v_D);
	\end{tikzpicture}
	\caption{Graph of the SCM.}
	\label{fig:bathtub SCM}
\end{subfigure}
\caption{The Markov ordering graph for the equilibrium equations of the filling bathtub system, obtained by applying Definition \ref{def:declustering} to the causal ordering graph in Figure \ref{fig:bathtub full causal ordering} is given in Figure \ref{fig:bathtub full markov ordering}. The d-separations in the Markov ordering graph imply conditional independences between corresponding endogenous variables. Most of these conditional independences cannot be read off from the graph for the SCM of the bathtub system in Figure \ref{fig:bathtub SCM}, except for $X_{v_I}\indep X_{v_K}$.}
\label{fig:scm compared with mo}
\end{figure}
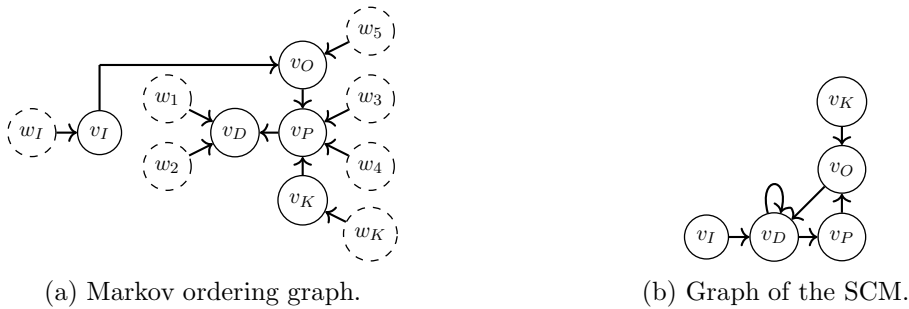

\paragraph{Encoded conditional independences:}

Since the assumption of unique solvability with respect to the causal ordering graph holds for this particular example, we can read off conditional independences between endogenous variables from the Markov ordering graph. More precisely, the d-separations in $\mathrm{MO}(\BG)$ between vertices in $V\setminus W$ imply conditional independences between the corresponding endogenous variables. For example:
\begin{align*}
v_K\dsep{\mathrm{MO}(\BG)} v_O &\;\implies\; X_{v_K}\indep X_{v_O}, \\
v_K\dsep{\mathrm{MO}(\BG)} v_D \given v_P &\;\implies\; X_{v_K}\indep X_{v_D} \given X_{v_P}, \\
v_I\dsep{\mathrm{MO}(\BG)} v_P \given v_O &\;\implies\; X_{v_I}\indep X_{v_P}\given X_{v_O}, \\
v_O\dsep{\mathrm{MO}(\BG)} v_D \given v_P &\;\implies\; X_{v_O}\indep X_{v_D} \given X_{v_P}.
\end{align*}
For $g>0$, every solution to the system of constraints has the same distribution, and this distribution is d-faithful to the Markov ordering graph. When $g=0$, the system of constraints only has a solution if $U_{w_I} = 0$ almost surely; in that case the corresponding distribution is not d-faithful w.r.t.\ the Markov ordering graph in Figure~\ref{fig:bathtub full markov ordering}.

\paragraph{Comparison to SCM representation:} 

The (random) differential equations that describe the system of a bathtub can be \emph{equilibrated to an SCM} that has a self-cycle. \citet{Bongers2020} show that the model has the following structural equations:
\begin{align*}
X_{v_K} &= U_{w_K}, \\
X_{v_I} &= U_{w_I}, \\
X_{v_P} &= g U_{w_3} X_{v_D}, \\
X_{v_O} &= U_{w_5} X_{v_K} X_{v_P}, \\
X_{v_D} &= X_{v_D} + U_{w_1} (X_{v_I} - X_{v_O}).
\end{align*}

The graph of this SCM is depicted in Figure~\ref{fig:bathtub SCM}. Because the SCM is uniquely solvable w.r.t.\ the strongly connected components $\{v_I\}, \{v_D, v_P, v_O\}$ and $\{v_K\}$, the $\sigma$-separations in this graph imply conditional independences \citep[Theorem 6.3 in][]{Bongers2020}. Most of the conditional independences implied by the Markov ordering graph cannot be read off from the graph of this SCM in Figure \ref{fig:bathtub SCM} via the $\sigma$-separation criterion, except for $X_{v_I}\indep X_{v_K}$. Clearly the distribution of a solution to the system of constraints is not faithful to the graph of the SCM and causal ordering on the equilibrium equations provides a stronger Markov property than equilibration to an SCM.

An important difference between SCMs and systems of constraints is that while the former require a particular one-to-one correspondence between endogenous variables and structural equations, the latter do not require a similar correspondence between endogenous variables and constraints. Interestingly, in the case of the bathtub model, a one-to-one correspondence between variables and constraints is obtained automatically by the causal ordering algorithm. In general, the bipartite graph of a set of structural equations is self-contained and perfect matchings connect each variable to an equation. If the SCM is acyclic then the associated bipartite graph has a unique perfect matching that retrieves the correspondence between variables and equations in the SCM. We further discuss applications of the technique of causal ordering to structural equations in Section \ref{sec:discussion:relation to other causal models}.

\subsection{Generalized directed global Markov property}
\label{sec:generalized directed global markov property}

For systems of constraints with no over- or incomplete parts, the associated directed graph that is constructed in the causal ordering algorithm via perfect matchings also yields a Markov property. Theorem \ref{theo:generalized markov property} below shows that for systems that are uniquely solvable with respect to the causal ordering graph, the $\sigma$-separations between variable vertices in the directed graph $\G(\BG,\M)_{\mathrm{mar}(F)}$ imply conditional independences between the corresponding solution components.

\begin{restatable}{theorem}{gdgmp}
\label{theo:generalized markov property}
Let $\B{X}^*=\B{g}(\B{X}_W)$ be a solution of a system of constraints $\tuple{\B{\X}, \B{X}_W, \B{\Phi}, \BG}$, where the subgraph of $\BG=\tuple{V,F,E}$ induced by $(V\cup F)\setminus W$ has a perfect matching $\M$. If for each strongly connected component $S$ in $\G(\BG,\M)$ with $S\cap W=\emptyset$, the system $\mathcal{M}$ is uniquely solvable w.r.t.\ $S_V=(S\cup \M(S))\cap V$ and $S_F=(S\cup \M(S))\cap F$ then the pair $(\G(\BG,\M)_{\mathrm{mar}(F)}, \P_{\B{X}^*})$ satisfies the generalized directed global Markov property (Definition~\ref{def:markov property}).
\end{restatable}

\begin{example}
Consider a system of constraints $\mathcal{M}=\tuple{\B{\X}, \B{X}_W, \B{\Phi}, \BG}$ with $W=\{w_1,\ldots,w_6\}$, $V\setminus W=\{v_1,\ldots,v_5\}$, $F=\{f_1,\ldots,f_5\}$, and $\BG=\tuple{V,F,E}$ as in Figure \ref{fig:bipartite exogenous}. Suppose that $\B{\X}=\R^{11}$ and $\B{\Phi}$ consists of constraints:
\begin{alignat*}{3}
\Phi_{f_1}:&\qquad&\,\, X_{v_1} - X_{w_1} &= 0,\\
\Phi_{f_2}:&\qquad&\,\, X_{v_2} - X_{v_1} + X_{v_3} + X_{w_2} - X_{w_3} &=0,\\
\Phi_{f_3}:&\qquad&\,\, X_{w_4} - X_{v_3} + X_{v_4} &=0,\\
\Phi_{f_4}:&\qquad&\,\, X_{w_5} + X_{v_2} - X_{v_4} &=0,\\
\Phi_{f_5}:&\qquad&\,\, X_{w_6} - X_{v_4} + X_{v_5} &=0.
\end{alignat*}
It is easy to check that this linear system of equations can be uniquely solved in the order prescribed by the causal ordering graph $\mathrm{CO}(\BG)$ in Figure \ref{fig:cog exogenous}. Therefore, according to Theorem \ref{thm:markov property} the d-separations among endogenous variables in the corresponding Markov ordering graph $\mathrm{MO}(\BG)$ imply conditional independences between the corresponding endogenous variables. It follows that d-separations in the  Markov ordering graph $\mathrm{MO}(\BG)_{\mathrm{mar}(W)}$ for the endogenous variables in Figure \ref{fig:markov ordering graph} imply conditional independences between the corresponding variables. For example, we see that $v_1$ and $v_5$ are d-separated by $v_4$ and deduce that for a solution $\B{X}^*$ to the system of constraints it holds that $X^*_{v_1} \indep X^*_{v_5} \given X^*_{v_4}$. One may note that d-separations in $\mathrm{MO}(\BG)_{\mathrm{mar}(W)}$ coincide with $\sigma$-separations in both associated directed graphs $\G(\BG,\M_1)_{\mathrm{mar}(F\cup W)}$ and $\G(\BG,\M_2)_{\mathrm{mar}(F\cup W)}$, which are depicted in Figures \ref{fig:cyclic1} and \ref{fig:cyclic2} respectively. It can be seen from the proof of Theorem \ref{theo:generalized markov property} in Appendix~\ref{app:acyclification proof} that this result holds in general. It follows from Theorem \ref{theo:generalized markov property} that the $\sigma$-separations in $\G(\BG,\M_1)_{\mathrm{mar}(F\cup W)}$ and $\G(\BG,\M_2)_{\mathrm{mar}(F\cup W)}$ imply conditional independences between the corresponding variables. For example, we see that $v_1$ and $v_5$ are $\sigma$-separated by $v_4$ in both graphs, and hence $X^*_{v_1} \indep X^*_{v_5} \given X^*_{v_4}$ for a solution $\B{X}^*$.
\end{example}

\begin{figure}[ht]
\begin{subfigure}[b]{.5\textwidth}
	\centering
	\begin{tikzpicture}[scale=0.75,every node/.style={transform shape}]
	\GraphInit[vstyle=Normal]
	\Vertex[L=$v_1$, x=0.0, y=0.0] {v1}
	\Vertex[L=$w_1$, x=0.0, y=1.2, style=dashed] {w1}
	\Vertex[L=$v_5$, x=0.0, y=-1.2] {v5}
	\Vertex[L=$v_2$, x=1.2, y=1.2] {v2}
	\Vertex[L=$v_3$, x=1.2, y=0.0] {v3}
	\Vertex[L=$v_4$, x=1.2, y=-1.2] {v4}
	\Vertex[L=$w_2$, x=2.7, y=1.5, style=dashed] {w2}
	\Vertex[L=$w_3$, x=2.7, y=0.5, style=dashed] {w3}
	\Vertex[L=$w_4$, x=2.7, y=-0.5, style=dashed] {w4}
	\Vertex[L=$w_5$, x=2.7, y=-1.5, style=dashed] {w5}
	\Vertex[L=$w_5$, x=-1.2, y=-1.2, style=dashed] {w6}
	\tikzset{EdgeStyle/.style = {->, thick}}
	\draw[EdgeStyle] (w1) to (v1);
	\draw[EdgeStyle] (w2) to (v2);
	\draw[EdgeStyle] (w3) to (v3);
	\draw[EdgeStyle] (w4) to (v4);
	
	\draw[EdgeStyle] (w5) to (v2);
	\draw[EdgeStyle] (w5) to (v3);
	\draw[EdgeStyle] (w5) to (v4);
	\draw[EdgeStyle] (w6) to (v5);
	
	\draw[EdgeStyle] (v1) to (v2);
	\draw[EdgeStyle] (v1) to (v3);
	\draw[EdgeStyle] (v1) to (v4);
	\draw[EdgeStyle] (v4) to (v5);
	
	\draw[EdgeStyle] (w2) to (v3);
	\draw[EdgeStyle] (w2) to (v4);
	
	\draw[EdgeStyle] (w3) to (v2);
	\draw[EdgeStyle] (w3) to (v4);

	\draw[EdgeStyle] (w4) to (v2);
	\draw[EdgeStyle] (w4) to (v3);
	\end{tikzpicture}
	\caption{$\mathrm{MO}(\BG)$.}
	\label{fig:markov ordering graph augmented}
\end{subfigure}%
\begin{subfigure}[b]{.5\textwidth}
	\centering
	\begin{tikzpicture}[scale=0.75,every node/.style={transform shape}]
	\GraphInit[vstyle=Normal]
	\SetGraphUnit{1.0}
	\SetVertexMath
	\Vertex{v_1}
	\EA[unit=1.2](v_1){v_2}
	\NOEA(v_2){v_3}
	\SOEA(v_3){v_4}
	\NOEA(v_4){v_5}
	\tikzset{EdgeStyle/.style = {->}}
	\Edge(v_1)(v_2)
	\tikzset{EdgeStyle/.style = {->, bend left}}
	\Edge(v_1)(v_3)
	\tikzset{EdgeStyle/.style = {->, bend right=30}}
	\Edge(v_1)(v_4)
	\tikzset{EdgeStyle/.style = {->}}
	\Edge(v_4)(v_5)
	\tikzset{EdgeStyle/.style = {<->}}
	\Edge(v_2)(v_4)
	\Edge(v_2)(v_3)
	\Edge(v_3)(v_4)
	\end{tikzpicture}
	\caption{$\mathrm{MO}(\BG)_{\mathrm{mar}(W)}$.}
	\label{fig:markov ordering graph}
\end{subfigure}
\vskip3mm
\begin{subfigure}[b]{.5\textwidth}
	\centering
	\begin{tikzpicture}[scale=0.75,every node/.style={transform shape}]
	\GraphInit[vstyle=Normal]
	\SetGraphUnit{1.0}
	\SetVertexMath
	\Vertex{v_1}
	\EA[unit=1.2](v_1){v_2}
	\NOEA(v_2){v_3}
	\SOEA(v_3){v_4}
	\NOEA(v_4){v_5}
	\tikzset{EdgeStyle/.style = {->}}
	\Edge(v_1)(v_2)
	\tikzset{EdgeStyle/.style = {->}}
	\Edge(v_3)(v_2)
	\tikzset{EdgeStyle/.style = {->}}
	\Edge(v_2)(v_4)
	\tikzset{EdgeStyle/.style = {->}}
	\Edge(v_4)(v_3)
	\tikzset{EdgeStyle/.style = {->}}
	\Edge(v_4)(v_5)
	\end{tikzpicture}
	\caption{$\G(\BG,\M_1)_{\mathrm{mar}(F\cup W)}$.}
	\label{fig:cyclic1}
\end{subfigure}%
\begin{subfigure}[b]{.5\textwidth}
	\centering
	\begin{tikzpicture}[scale=0.75,every node/.style={transform shape}]
	\GraphInit[vstyle=Normal]
	\SetGraphUnit{1.0}
	\SetVertexMath
	\Vertex{v_1}
	\EA[unit=1.2](v_1){v_2}
	\NOEA(v_2){v_3}
	\SOEA(v_3){v_4}
	\NOEA(v_4){v_5}
	\tikzset{EdgeStyle/.style = {->, bend left}}
	\Edge(v_1)(v_3)
	\tikzset{EdgeStyle/.style = {->}}
	\Edge(v_3)(v_4)
	\tikzset{EdgeStyle/.style = {->}}
	\Edge(v_4)(v_2)
	\tikzset{EdgeStyle/.style = {->}}
	\Edge(v_2)(v_3)
	\tikzset{EdgeStyle/.style = {->}}
	\Edge(v_4)(v_5)
	\end{tikzpicture}
	\caption{$\G(\BG,\M_2)_{\mathrm{mar}(F\cup W)}$.}
	\label{fig:cyclic2}
\end{subfigure}
\caption{The Markov ordering graph of the causal ordering graph in Figure \ref{fig:cog exogenous} is given in Figure \ref{fig:markov ordering graph augmented}. Marginalization of the exogenous vertices $W$ results in the directed mixed graph in  Figure \ref{fig:markov ordering graph}. The directed graphs in Figures \ref{fig:cyclic1} and \ref{fig:cyclic2} are obtained by marginalizing out the constraint vertices $F$ and exogenous vertices $W$ from the directed graphs $\G(\BG,\M_1)$ and $\G(\BG,\M_2)$ in Figures \ref{fig:directed exogenous} and \ref{fig:directed exogenous2} respectively. Note that d-separations in the Markov ordering graph correspond to $\sigma$-separations in the associated directed graphs in Figures \ref{fig:directed exogenous} and \ref{fig:directed exogenous2}.}
\label{fig:markov ordering graph and cycles}
\end{figure}
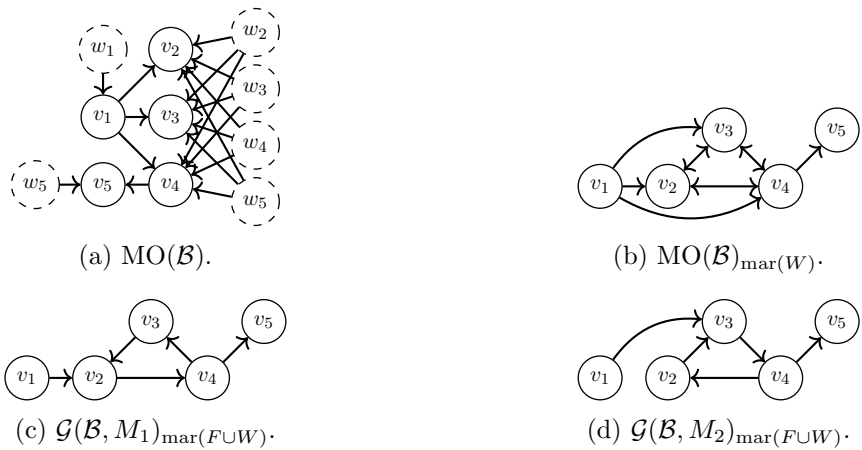

\section{Causal implications of sets of equations}
\label{sec:causal implications for sets of equations}

Nowadays, it is common to relate causation directly to the effects of manipulation \citep{Woodward2003, Pearl2000}. In the context of sets of equations there are many ways to model manipulations on these equations. Assuming that the manipulations correspond to feasible actions in the real world that is modelled by the equations, the effects of manipulations correspond to causal relations. In order to derive causal implications from systems of constraints, we explicitly define two types of manipulation. We consider the notions of both \emph{soft} and \emph{perfect} interventions on sets of equations.\footnote{Our definitions in the context of systems of constraints may deviate from conventional definitions of interventions on SCMs. In an SCM, each variable is associated with a single structural equation. The notion of a perfect intervention on an SCM does not carry over for systems of constraints because there is no imposed one-to-one correspondence between equations and variables.} We prove that the causal ordering graph represents the effects of both \emph{soft interventions on equations} and \emph{perfect interventions on clusters} in the causal ordering graph. We also show that these interventions commute with causal ordering.

\subsection{The effects of soft interventions}
\label{sec:effects soft interventions}

A \emph{soft intervention}, also known as a ``mechanism change'', acts on a constraint. It replaces the targeted constraint by a constraint in which the same variables appear as in the original one. This type of intervention does not change the bipartite graph that represents the structure of the constraints.

\begin{definition}
\label{def:soft intervention}
Let $\mathcal{M}=\tuple{\B{\X}, \B{X}_W, \B{\Phi}, \BG}$ be a system of constraints, $\Phi_f=\tuple{\phi_f,c_f,V(f)} \in \B{\Phi}$ a constraint, $c'_f$ a constant taking value in a measurable space $\B{\Y}$, and $\phi'_f:\B{\X}_{V(f)}\to \B{\Y}$ a measurable function. A \emph{soft intervention $\mathrm{si}(f,\phi'_f, c'_f)$ targeting $\Phi_f$} results in the intervened system $\mathcal{M}_{\mathrm{si}(f,\phi'_f,c'_f)}=\tuple{\B{\X},\B{X}_W,\B{\Phi}_{\mathrm{si}(f, \phi'_f, c'_f)},\BG}$ where $\B{\Phi}_{\mathrm{si}(f, \phi'_f, c'_f)}=\left(\B{\Phi}\setminus\{\Phi_f\}\right)\cup\{\Phi'_f\}$ with $\Phi'_f=\tuple{\phi'_f,c'_f,V(f)}$.
\end{definition}

For systems of constraints that are maximally uniquely solvable w.r.t.\ the causal ordering graph, both before and after a soft intervention, Theorem \ref{thm:soft interventions} shows that such a soft intervention does not have an effect on variables that cannot be reached by a directed path from that constraint in the causal ordering graph, while it may have an effect on other variables.\footnote{Our result generalizes Theorem 6.1 in \citet{Simon1953} for linear self-contained systems of equations. The proof of our theorem is similar.}

\begin{restatable}{theorem}{softinterventions}
\label{thm:soft interventions}
  Let $\mathcal{M}=\tuple{\B{\X}, \B{X}_W, \B{\Phi}, \BG}$ be a system of constraints with coarse decomposition $\mathrm{CD}(\BG)=\tuple{T_I,T_C,T_O}$. Suppose that $\mathcal{M}$ is maximally uniquely solvable w.r.t.\ the causal ordering graph $\mathrm{CO}(\BG)$ and let $\B{X}^*=\B{g}(\B{X}_W)$ be a solution of $\mathcal{M}$. Let $f\in (T_C\cup T_O)\cap F$ and assume that the intervened system $\mathcal{M}_{\mathrm{si}(f, \phi'_f,c'_f)}$ is also maximally uniquely solvable w.r.t.\ $\mathrm{CO}(\BG)$. Let $\B{X}' = \B{h}(\B{X}_W)$ be a solution of $\mathcal{M}_{\mathrm{si}(f, \phi'_f,c'_f)}$. If there is no directed path from $f$ to $v\in (T_C\cup T_O) \cap V$ in $\mathrm{CO}(\BG)$ then $X^*_v=X'_v$ almost surely. On the other hand, if there is a directed path from $f$ to $v$ in $\mathrm{CO}(\BG)$ then $X^*_v$ may have a different distribution than $X'_v$, depending on the details of the model $\mathcal{M}$.
\end{restatable}

Example \ref{ex:bathtub soft interventions} shows that the presence of a directed path in the causal ordering graph for the equilibrium equations of the bathtub system implies a causal effect for almost all parameter values. This illustrates that non-effects and \emph{generic} effects can be read off from the causal ordering graph.\footnote{If a directed path from an equation vertex $f$ to a variable vertex $v$ implies that an intervention on $f$ changes the distribution of the solution component $X_v$ for almost all values (w.r.t.\ Lebesgue measure) of the parameters, then we say that there is a \emph{generic} causal effect of $f$ on $v$.}

\begin{example}
\label{ex:bathtub soft interventions}
Recall the system of constraints for the filling bathtub in Section \ref{sec:application to bathtub}. Think of an experiment where the gravitational constant $g$ is changed so that it takes on a different value $g'$ without altering the other equations that describe the bathtub system. Such an experiment is, at least in theory, feasible. For example, it can be accomplished by accelerating the bathtub system or by moving the bathtub system to another planet. We can model the effect on the equilibrium distribution in such an experiment by a soft intervention targeting $f_P$ that replaces the constraint $\Phi_{f_P}$ by
\begin{align}
\tuple{X_{V(f_P)}\mapsto U_{w_1}(g'U_{w_2} X_{v_D} - X_{v_P}),\,\, 0,\,\, V(f_P) = \{v_D,v_P,w_1,w_2\}}.
\end{align}
Which variables are and which are not affected by this soft intervention? We can read off the effects of this soft intervention from the causal ordering graph in Figure \ref{fig:bathtub full causal ordering}. There is no directed path from $f_P$ to $v_K, v_I, v_P$ or $v_O$. Therefore, perhaps surprisingly, Theorem \ref{thm:soft interventions} tells us that the soft intervention targeting $f_P$ neither has an effect on the pressure $X_{v_P}$ at equilibrium nor on the outflow rate $X_{v_O}$ at equilibrium. Since there is a directed path from $f_P$ to $v_D$, the water level $X_{v_D}$ at equilibrium may be different after a soft intervention on $f_P$. If the gravitational constant $g$ is equal to zero, then the system of constraints for the bathtub is not maximally uniquely solvable w.r.t.\ the causal ordering graph (except if $U_{w_I} = 0$ almost surely). For all other values of the parameter $g$ the generic effects and non-effects of soft interventions on other constraints of the bathtub system can be read off from the causal ordering graph and are presented in Table \ref{tab:soft interventions cog}.
\end{example}

\begin{table}
	\caption{The effects of soft interventions on constraints in the causal ordering graph for the bathtub system in Figure \ref{fig:bathtub full causal ordering}.}
	\label{tab:soft interventions cog}
	\begin{center}
		\begin{tabular}{l l l}
			\toprule
			target & generic effect & non-effect \\
			\midrule
			$f_K$ & $X_{v_K}$, $X_{v_P}$, $X_{v_D}$ & $X_{v_I}$, $X_{v_O}$ \\
			$f_I$ & $X_{v_I}$, $X_{v_P}$, $X_{v_O}$, $X_{v_D}$ & $X_{v_K}$ \\
			$f_P$ & $X_{v_D}$ & $X_{v_K}$, $X_{v_I}$, $X_{v_P}$, $X_{v_O}$ \\
			$f_O$ & $X_{v_P}$, $X_{v_D}$ & $X_{v_K}$, $X_{v_I}$, $X_{v_O}$ \\
			$f_D$ & $X_{v_P}$, $X_{v_O}$, $X_{v_D}$ & $X_{v_K}$, $X_{v_I}$ \\
			\bottomrule
		\end{tabular}
	\end{center}
\end{table}

\subsection{The effects of perfect interventions}
\label{sec:effects perfect interventions}

A \emph{perfect intervention} acts on a variable and a constraint. Definition \ref{def:perfect intervention} shows that it replaces the targeted constraint by a constraint that sets the targeted variable equal to a constant. Note that this definition of perfect interventions is very general and allows interventions for which the intervened system of constraints is not maximally uniquely solvable w.r.t.\ the causal ordering graph. In this work, we will only consider the subset of perfect interventions that target clusters in the causal ordering graph, for which the intervened system is also maximally uniquely solvable w.r.t.\ the causal ordering graph. We consider an analysis of necessary conditions on interventions for the intervened system to be consistent beyond the scope of this work.

\begin{definition}
\label{def:perfect intervention}
Let $\mathcal{M}=\tuple{\B{\X}, \B{X}_W, \B{\Phi}, \BG=\tuple{V,F,E} }$ be a system of constraints and let $\xi_v\in\X_v$. A \emph{perfect intervention $\mathrm{do}(f,v,\xi_v)$ targeting the variable $v\in V\setminus W$ and the constraint $f\in F$} results in the intervened system $\mathcal{M}_{\mathrm{do}(f,v,\xi_v)} = \tuple{\B{\X}, \B{X}_W, \B{\Phi}_{\mathrm{do}(f,v,\xi_v)}, \BG_{\mathrm{do}(f,v)}}$ where
\begin{enumerate}
	\item $\B{\Phi}_{\mathrm{do}(f,v,\xi_v)} = \left(\B{\Phi}\setminus \Phi_f\right) \cup \{\Phi'_f\}$ with $\Phi'_f=\tuple{X_v \mapsto X_v, \xi_v, \{v\}}$,
	\item $\BG_{\mathrm{do}(f,v)}=\tuple{V,F,E'}$ with $E'=\{(i-j)\in E: i,j \neq f \}\cup \{(v-f)\}$.
\end{enumerate}
\end{definition}

Perfect interventions on a set of variable-constraint pairs $\{(f_1,v_1), \ldots, (f_n,v_n)\}$ in a system of constraints are denoted by $\mathrm{do}(S_F,S_V,\B{\xi}_{S_V})$ where $S_F=\tuple{f_1,\ldots, f_n}$ and $S_V=\tuple{v_1,\ldots,v_n}$ are tuples. For a bipartite graph $\BG$ so that its subgraph induced by $(V\cup F)\setminus W$ is self-contained, Lemma \ref{lemma:self-contained after intervention} shows that the subgraph of the intervened bipartite graph $\BG_{\mathrm{do}(S_F,S_V)}$ induced by $(V\cup F)\setminus W$ is also self-contained when $S=(S_F\cup S_V)$ is a cluster in $\mathrm{CO}(\BG)$ with $S\cap W=\emptyset$.

\begin{restatable}{lemma}{selfcontainedafterintervention}
\label{lemma:self-contained after intervention}
Let $\BG=\tuple{V,F,E}$ be a bipartite graph and $W\subseteq V$, so that the subgraph of $\BG$ induced by $(V\cup F)\setminus W$ is self-contained. Consider an intervention $\mathrm{do}(S_V,S_F)$ on a cluster $S= S_F\cup S_V$ with $S\cap W=\emptyset$ in the causal ordering graph $\mathrm{CO}(\BG)$. The subgraph of $\BG_{\mathrm{do}(S_F,S_V)}$ induced by $(V\cup F)\setminus W$ is self-contained.
\end{restatable}

Theorem \ref{theo:effects of perfect interventions on clusters} shows how the causal ordering graph can be used to read off the (generic) effects and non-effects of \emph{perfect interventions on clusters} in the complete and overcomplete sets of the associated bipartite graph under the assumption of unique solvability with respect to the complete and overcomplete sets in the causal ordering graph.

\begin{restatable}{theorem}{perfectinterventions}
\label{theo:effects of perfect interventions on clusters}
Let $\mathcal{M}=\tuple{\B{\X}, \B{X}_W, \B{\Phi}, \BG=\tuple{V,F,E}}$ with coarse decomposition $\mathrm{CD}(\BG)=\tuple{T_I,T_C,T_O}$. Assume that $\mathcal{M}$ is maximally uniquely solvable w.r.t.\ $\mathrm{CO}(\BG)=\tuple{\V,\E}$ and let $\B{X}^*$ be a solution of $\mathcal{M}$. Let $S_F\subseteq (T_C\cup T_O)\cap F$ and $S_V\subseteq (T_C\cup T_O)\cap (V\setminus W)$ be such that $(S_F\cup S_V)\in\V$. Consider the intervened system $\mathcal{M}_{\mathrm{do}(S_F,S_V,\B{\xi}_{S_V})}$ with coarse decomposition $\mathrm{CD}(\BG_{\mathrm{do}(S_F,S_V)})=\tuple{T'_I,T'_C,T'_O}$. Let $\B{X}'$ be a solution of $\mathcal{M}_{\mathrm{do}(S_F,S_V,\B{\xi}_{S_V})}$. If there is no directed path from any $x\in S_V$ to $v\in (T_C\cup T_O)\cap V$ in $\mathrm{CO}(\BG)$ then $X^*_v=X'_v$ almost surely. On the other hand, if there is $x\in S_V$ such that there is a directed path from $x$ to $v$ in $\mathrm{CO}(\BG)$ then $X^*_v$ may have a different distribution than $X'_v$.
\end{restatable}

One way to determine whether a perfect intervention has an effect on a certain variable is to explicitly solve the system of constraints before and after the intervention and check which solution components are altered. In particular, when the distribution of a solution component is different for almost all parameter values, then we say that there is a generic effect. This way, we can establish the generic effects of a perfect intervention without solving the equations by relying on a solvability assumption. Example \ref{ex:bathtub perfect interventions} illustrates this notion of perfect intervention on the system of constraints for the filling bathtub that we first introduced in Example \ref{ex:bathtub intro} and shows how the generic effects and non-effects of perfect interventions on clusters can be read off from the causal ordering graph.

\begin{table}[!htb]
	\caption{Solutions for system of constraints describing the bathtub system in Section \ref{sec:application to bathtub} without interventions (i.e.\ the observed system) and after perfect interventions $\mathrm{do}(f_P,v_D,\xi_D)$, $\mathrm{do}(f_D,v_O,\xi_O)$, and $\mathrm{do}(f_D,v_D,\xi_D)$.}
	\label{tab:bathtub solutions}
	\vspace{-6pt}
	\begin{center}
		\begin{tabular}{l  l  l  l  l }
			\toprule
			& observed & $\mathrm{do}(f_P,v_D,\xi_D)$ & $\mathrm{do}(f_D,v_O,\xi_O)$ & $\mathrm{do}(f_D,v_D,\xi_D)$ \\
			\midrule
			$X^*_{v_K}$ & $U_{w_K}$ & $U_{w_K}$ & $U_{w_K}$ & $U_{w_K}$\\
			$X^*_{v_I}$ & $U_{w_I}$ & $U_{w_I}$ & $U_{w_I}$ & $U_{w_I}$ \\
			$X^*_{v_P}$ & $\frac{U_{w_I}}{(U_{w_4}U_{w_K})}$ & $\frac{U_{w_I}}{(U_{w_4}U_{w_K})}$ & $\frac{\xi_O}{ (U_{w_4}U_{w_K})}$ & $gU_{w_2}\xi_D$ \\
			$X^*_{v_O}$ & $U_{w_I}$ & $U_{w_I}$ & $\xi_O$ & $U_{w_4}U_{w_K}g U_{w_2} \xi_D$ \\
			$X^*_{v_D}$ & $\frac{U_{w_I}}{(U_{w_4}U_{w_K}g U_{w_2})}$ & $\xi_D$ & $\frac{\xi_O}{(U_{w_4}U_{w_K}g U_{w_2})}$ & $\xi_D$\\
			\bottomrule
		\end{tabular}
	\end{center}
	%
	\caption{The effects of perfect interventions on clusters of variables and constraints in the causal ordering graph for the bathtub system in Figure \ref{fig:bathtub full causal ordering} obtained by Theorem \ref{theo:effects of perfect interventions on clusters}. Since $\{f_D,v_D\}$ is not a cluster in the causal ordering graph, the effects of this intervention cannot be read off from the causal ordering graph.}
	\label{tab:perfect interventions cog}
	\vspace{-6pt}
	\begin{center}
		\begin{tabular}{l l l}
			\toprule
			target & generic effect & non-effect \\
			\midrule
			$f_K, v_K$ & $X_{v_K}$, $X_{v_P}$, $X_{v_D}$ & $X_{v_I}$, $X_{v_O}$ \\
			$f_I, v_I$ & $X_{v_I}$, $X_{v_P}$, $X_{v_O}$, $X_{v_D}$ & $X_{v_K}$ \\
			$f_P, v_D$ & $X_{v_D}$ & $X_{v_K}$, $X_{v_I}$, $X_{v_P}$, $X_{v_O}$ \\
			$f_O, v_P$ & $X_{v_P}$, $X_{v_D}$ & $X_{v_K}$, $X_{v_I}$, $X_{v_O}$ \\
			$f_D, v_O$ & $X_{v_P}$, $X_{v_O}$, $X_{v_D}$ & $X_{v_K}$, $X_{v_I}$ \\
			$f_P, f_D, f_O, v_P, v_D, v_O$ & $X_{v_P}$, $X_{v_O}$, $X_{v_D}$ & $X_{v_K}$, $X_{v_I}$\\
			\bottomrule
		\end{tabular}
	\end{center}
\end{table}

\begin{example}
\label{ex:bathtub perfect interventions}
Recall the system of constraints $\mathcal{M}$ for the filling bathtub at equilibrium in Section \ref{sec:application to bathtub} and consider the perfect interventions $\mathrm{do}(f_P,v_D,\xi_D)$, $\mathrm{do}(f_D,v_O,\xi_O)$, and $\mathrm{do}(f_D,v_D,\xi_D)$. These interventions model experiments that can, at least in principle, be conducted in practice:
\begin{enumerate}
	\item The intervention $\mathrm{do}(f_P,v_D,\xi_D)$ replaces the constraint $f_P$ by a constraint that sets the water level $X_{v_D}$ equal to a constant and leaves all other constraints unaffected. This could correspond to an experimental set-up where the constant $g$ in the constraint $\Phi_{f_P}$ is controlled by accelerating and decelerating the bathtub system precisely in such a way that the water level $X_{v_D}$ is forced to take on a constant value $\xi_D$ both in time and across the ensemble of bathtubs. We observe the system once it has reached equilibrium.
	\item The interventions $\mathrm{do}(f_D,v_O,\xi_O)$ and $\mathrm{do}(f_D,v_D,\xi_D)$, may correspond to an experiment where a hose is added to the system that can remove or add water precisely in such a way that either the outflow rate $X_{v_O}$ or the water level $X_{v_D}$ is kept at a constant level both in time and across the ensemble of bathtubs. The system is observed when it has reached equilibrium.
\end{enumerate}
Note that the cluster $\{f_D,v_D\}$ is not a cluster in the causal ordering graph in Figure \ref{fig:bathtub full causal ordering}. However, the system of constraints $\mathcal{M}_{\mathrm{do}(f_D,v_D,\xi_D)}$ is maximally uniquely solvable with respect to the causal ordering graph $\mathrm{CO}(\BG_{\mathrm{do}(f_D,v_D)})$, and therefore the effects of the intervention are well-defined.\footnote{The intervention on $\{f_D,v_D\}$ is interesting because it removes the constraint that the water flowing through the faucet $X_{v_I}$ must be equal to the water flowing through the drain $X_{v_O}$. This can be accomplished by adding a hose to the system through which additional water can flow in and out of the bathtub to ensure that $X_{v_D}$ remains at a constant level. Notice that, in this example, the \emph{total} inflow and \emph{total} outflow of water remain equal, while the inflow \emph{through the faucet} and the outflow \emph{through the drain} may differ.} By explicit calculation we obtain the (unique) solutions in Table \ref{tab:bathtub solutions} for the observed and intervened bathtub systems. By comparing with the solutions in the observed column we read off that the perfect intervention $\mathrm{do}(f_P,v_D,\xi_D)$ does not change the solution for the variables $X_{v_K}, X_{v_I},X_{v_P},X_{v_O}$, but it generically does change the solution for $X_{v_D}$. We further find that $\mathrm{do}(f_D,v_D,\xi_D)$ and $\mathrm{do}(f_D,v_O,\xi_D)$ affect the solution for the variables $X_{v_P},X_{v_O},X_{v_D}$ but not of $X_{v_K}$ and $X_{v_I}$.

The causal ordering graph $\mathrm{CO}(\BG)=\tuple{\V,\E}$ for the bathtub system is given in Figure \ref{fig:bathtub full causal ordering}. It has clusters $\mathcal{V}=\{\{f_K,v_K\}, \{f_I,v_I\}, \{f_P, v_D\}, \{f_O, v_P\}, \{f_D, v_O\} \}$. Under the assumption that the (intervened) system is maximally uniquely solvable w.r.t.\ its causal ordering graph, we can apply Theorem \ref{theo:effects of perfect interventions on clusters} and read off the effects and non-effects of perfect interventions on clusters, which are presented in Table \ref{tab:perfect interventions cog}. This illustrates the fact that we can establish the generic effects and non-effects of the perfect interventions $\mathrm{do}(f_P,v_D,\xi_D)$ and $\mathrm{do}(f_D,v_O,\xi_O)$, which act on clusters in the causal ordering graph, without explicitly solving the system of equations. We will discuss differences between causal implications of the causal ordering graph and the graph of the SCM in Figure \ref{fig:bathtub SCM} in Section \ref{sec:discussion}.
\end{example}

\subsection{Interventions commute with causal ordering}
\label{sec:interventions commute with causal ordering}

Given a system of constraints we can obtain the causal ordering graph after a perfect intervention on one of its clusters in the original causal ordering graph by running the causal ordering algorithm on the bipartite graph in the intervened system of constraints. In this section we will define an operation of ``perfect intervention'' directly on the clusters in a causal ordering graph and show that the causal ordering graph that is obtained after a perfect intervention coincides with the causal ordering graph of the intervened system (i.e.\ perfect interventions on clusters in the causal ordering graph commute with the causal ordering algorithm). Roughly speaking, a perfect intervention on a cluster in a directed cluster graph removes all incoming edges to that cluster and separates all variable vertices and constraint vertices in the targeted cluster into separate clusters in a specified way.

\begin{definition}
	\label{def:perfect intervention on a directed cluster graph}
	Let $\BG=\tuple{V,F,E}$ be a bipartite graph and $W$ a set of exogenous variables. Let $\mathrm{CO}(\BG)=\tuple{\mathcal{V},\mathcal{E}}$ be the corresponding causal ordering graph and consider $S\in\mathcal{V}$ with $S\cap W = \emptyset$. Let $S_F=\tuple{f_i: i=1,\ldots, n}$ and $S_V = \tuple{v_i: i=1,\ldots, n}$ with $n=|S\cap V|=|S\cap F|$ be tuples consisting of all the vertices in $S\cap F$ and $S\cap V$ respectively. A \emph{perfect intervention $\mathrm{do}(S_F, S_V)$ on a cluster $\{S_F,S_V\}$} results in the directed cluster graph $\mathrm{CO}(\BG)_{\mathrm{do}(S_F,S_V)}=\tuple{\V',\E'}$ where\footnote{A perfect intervention $\mathrm{do}(S_F,S_V,\B{\xi}_V)$ replaces constraints $\Phi_{f_i}$ with causal constraints $\Phi'_{f_i}=\tuple{X_{v_i}\mapsto X_{v_i}, \xi_{v_i}, \{v_i\}}$. Notice that the labels $f_i$ of the constraints are unaltered, and therefore only the edges in the bipartite graph and causal ordering graph change after an intervention, as well as the clusters in the causal ordering graph, while the labels of vertices are preserved.}
	\begin{enumerate}
		\item $\mathcal{V}'= \left(\V\setminus\{S\}\right)\cup\{\{v_i,f_i\}: i=1,\ldots,n\}$,\nopagebreak
		\item $\mathcal{E}'=\{(x\to T)\in\mathcal{E}: T\neq S\}$.
	\end{enumerate}
	\vspace{-6pt}
\end{definition}

A soft intervention on a system of constraints has no effect on the bipartite graphical structure of the constraints and the variables that appear in them. Since the bipartite graph of the system is the same before and after soft interventions, it trivially follows that soft interventions commute with causal ordering. The following proposition shows that perfect interventions on clusters also commute with causal ordering.

\begin{restatable}{proposition}{commute}
\label{prop:perfect interventions on clusters commute}
Let $\BG=\tuple{V,F,E}$ be a bipartite graph and $W$ a set of exogenous variables. Let $\mathrm{CO}(\BG)=\tuple{\mathcal{V},\mathcal{E}}$ be the corresponding causal ordering graph. Let $S_F\subseteq F$ and $S_V\subseteq V\setminus W$ be such that $(S_F\cup S_V)\in\V$. Then:
\begin{align*}
\mathrm{CO}(\BG_{\mathrm{do}(S_F,S_V)}) = \mathrm{CO}(\BG)_{\mathrm{do}(S_F,S_V)}.
\end{align*}
\end{restatable}

The bipartite graph in Figure \ref{fig:intervened causal ordering} (left upper graph) has the causal ordering graph depicted in Figure \ref{fig:intervened causal ordering} (right upper graph). The perfect intervention $\mathrm{do}(S_F,S_V)$ with $S_F=\tuple{f_2,f_3}$ and $S_V=\{v_2,v_3\}$ on this causal ordering graph results in the directed cluster graph in Figure \ref{fig:intervened causal ordering} (right lower graph). Since perfect interventions on clusters commute with causal ordering this graph can also be obtained by applying the causal ordering algorithm to the intervened bipartite graph in Figure \ref{fig:intervened causal ordering} (left lower graph). Proposition \ref{prop:perfect interventions on clusters commute} shows that perfect interventions on the graphical level can be used to draw conclusions about dependencies and causal implications of the underlying intervened system of constraints. We will use this result in Section \ref{sec:equilibration in dynamical causal models} to elucidate the commutation properties of equilibration in dynamical models and interventions in \citet{Dash2005} and \citet{Bongers2018}.

\begin{figure}[ht]
  \centering
  \begin{tikzpicture}[scale=0.75,every node/.style={transform shape}]
  \begin{scope}
    \GraphInit[vstyle=Normal]
    \SetGraphUnit{1.2}
    \SetVertexMath
    \Vertex{v_1}
    \EA(v_1){v_2}
    \EA(v_2){v_3}
    \EA(v_3){v_4}
    \SO[unit=1](v_1){f_1}
    \SO[unit=1](v_2){f_2}
    \SO[unit=1](v_3){f_3}
    \SO[unit=1](v_4){f_4}
    \tikzset{EdgeStyle/.style = {-}}
    \Edge(f_1)(v_1)
    \Edge(f_2)(v_2)
    \Edge(f_3)(v_3)
    \Edge(f_4)(v_4)
    \Edge(f_2)(v_1)
    \Edge(f_2)(v_3)
    \Edge(f_3)(v_2)
    \Edge(f_4)(v_3)
    \node at (2,1) {(a) Bipartite graph $\BG$.};
  \end{scope}
  \begin{scope}[xshift=10cm,yshift=0cm]
    \GraphInit[vstyle=Normal]
    \SetGraphUnit{1.5}
    \SetVertexMath
    \Vertex{v_1}
    \EA(v_1){v_2}
    \EA(v_2){v_3}
    \EA(v_3){v_4}
    \SO[unit=1](v_1){f_1}
    \SO[unit=1](v_2){f_2}
    \SO[unit=1](v_3){f_3}
    \SO[unit=1](v_4){f_4}
    \node[draw=black, fit=(v_1) (f_1), inner sep=0.1cm ]{};
    \node[draw=black, fit=(v_2) (f_2) (v_3) (f_3), inner sep=0.1cm ]{};
    \node[draw=black, fit=(v_4) (f_4), inner sep=0.1cm ]{};
    \draw[EdgeStyle, style={->}](v_3) to (3.9,0);
    \draw[EdgeStyle, style={->}](v_1) to (0.9,0);
    \node at (2,1) {(b) Causal ordering graph $\mathrm{CO}(\BG)$.};
  \end{scope}
  \begin{scope}[xshift=0cm,yshift=-5cm]
    \GraphInit[vstyle=Normal]
    \SetGraphUnit{1.2}
    \SetVertexMath
    \Vertex{v_1}
    \EA(v_1){v_2}
    \EA(v_2){v_3}
    \EA(v_3){v_4}
    \SO[unit=1](v_1){f_1}
    \SO[unit=1](v_2){f_2}
    \SO[unit=1](v_3){f_3}
    \SO[unit=1](v_4){f_4}
    \tikzset{EdgeStyle/.style = {-}}
    \Edge(f_1)(v_1)
    \Edge(f_2)(v_2)
    \Edge(f_3)(v_3)
    \Edge(f_4)(v_4)
    \Edge(f_4)(v_3)
    \node at (2,-2) {(c) Intervened bipartite graph $\BG_{\mathrm{do}(S_F,S_V)}$.};
  \end{scope}
  \begin{scope}[xshift=10cm,yshift=-5cm]
    \GraphInit[vstyle=Normal]
    \SetGraphUnit{1.5}
    \SetVertexMath
    \Vertex{v_1}
    \EA[unit=1.5](v_1){v_2}
    \EA(v_2){v_3}
    \EA(v_3){v_4}
    \SO[unit=1](v_1){f_1}
    \SO[unit=1](v_2){f_2}
    \SO[unit=1](v_3){f_3}
    \SO[unit=1](v_4){f_4}
    \node[draw=black, fit=(v_1) (f_1), inner sep=0.1cm ]{};
    \node[draw=black, fit=(v_2) (f_2), inner sep=0.1cm ]{};
    \node[draw=black, fit=(v_3) (f_3), inner sep=0.1cm ]{};
    \node[draw=black, fit=(v_4) (f_4), inner sep=0.1cm ]{};
    \draw[EdgeStyle, style={->}](v_3) to (3.9,0);
    \node at (2,-2) {(d) $\mathrm{CO}(\BG_{\mathrm{do}(S_F,S_V)}) = \mathrm{CO}(\BG)_{\mathrm{do}(S_F,S_V)}$.};
  \end{scope}
  \draw[|->,thick] (5,-0.5) -- node[anchor=south] {Causal Ordering} (8.5,-0.5);
  \draw[|->,thick] (5,-5.5) -- node[anchor=south] {Causal Ordering} (8.5,-5.5);
  \draw[|->,thick] (2,-2) -- node[anchor=east] {$\mathrm{do}(S_F,S_V)$} (2,-4);
  \draw[|->,thick] (12,-2) -- node[anchor=west] {$\mathrm{do}(S_F,S_V)$} (12,-4);
  \end{tikzpicture}
 \caption{The intervention $\mathrm{do}(S_F,S_V)$ with ordered sets $S_F=\tuple{f_2,f_3}$ and $S_V=\tuple{v_2,v_3}$ commutes with causal ordering. Application of causal ordering and the intervention to the bipartite graph (a) results in the causal ordering graph (b) and the intervened bipartite graph (c) respectively. The directed cluster graph (d) can be obtained either by applying causal ordering to the intervened bipartite graph or by intervening on the causal ordering graph.}
 \label{fig:intervened causal ordering}
\end{figure}
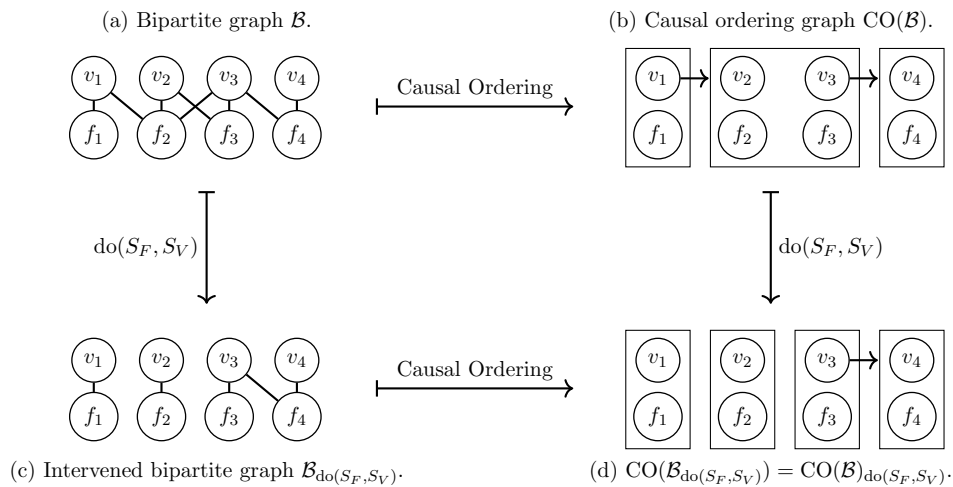

\section{Discussion}
\label{sec:discussion}

In this section we give a detailed account of how our work relates to some of the existing literature on causal ordering and causal modelling.

\subsection{``The causal graph'': A misnomer?}
\label{sec:the causal graph a misnomer}

Our work extends the work of \citet{Simon1953} who introduced the causal ordering algorithm. We extensively discussed the example of a bathtub that first appeared in \citet{Iwasaki1994}, in which the authors refer to the Markov ordering graph as ``the causal graph'' and claim that this graph represents the effects of ``manipulations''. We observe here that the Markov ordering graph in Figure \ref{fig:bathtub full markov ordering} does not have an unambiguous causal interpretation, contrary to claims in the literature. In this work we have formalized soft and perfect interventions, which are two common types of manipulation. This allows us to show that the Markov ordering graph, unlike the causal ordering graph, neither represents the effects of soft interventions nor does it have a straightforward interpretation in terms of perfect interventions. \citet{Iwasaki1994} do not clarify what the correct causal interpretation of the Markov ordering graph should be and therefore we believe that the term ``causal graph'' is a misnomer from a contemporary perspective on interventions and causality.

\paragraph{Markov ordering.} To support this claim, we consider the bathtub system in \citet{Iwasaki1994} that we presented in Example \ref{ex:bathtub intro}. The structure of the equations and the endogenous variables that appear in them can be represented by the bipartite graph in Figure \ref{fig:bathtub bipartite observed}. The corresponding Markov ordering graph in Figure \ref{fig:bathtub markov ordering observed} corresponds to the graph that \citet{Iwasaki1994} call the ``causal graph'' for the bathtub system. Note that \citet{Iwasaki1994} do not make a distinction between variable vertices and equation vertices like we do. Their ``causal graph'' therefore has vertices $K,I,P,O,D$ instead of $v_K,v_I,v_P,v_O,v_D$. An aspect that is not discussed at all by \citet{Iwasaki1994}, is that the Markov ordering graph implies conditional independences between components of solutions of equations.\footnote{\citet{Iwasaki1994} consider deterministic systems of equations and therefore it would not have made sense to consider Markov properties. In earlier work, the vanishing partial correlations implied by linear systems with three variables and normal errors were studied by \citet{Simon1954}.}

\paragraph{Soft interventions.} We first consider the representation of soft interventions. Table \ref{tab:soft interventions cog} shows that a soft intervention on $f_D$ has a generic effect on the solution for the variables $v_P,v_O$, and $v_D$. This soft intervention cannot be read off from the Markov ordering graph in Figure \ref{fig:bathtub markov ordering observed} because there is no vertex $f_D$. Since \citet{Iwasaki1994} make no distinction between variable vertices and equation vertices, a manipulation on $D$ should perhaps be interpreted as a soft intervention on the vertex $D$ in the Markov ordering graph in Figure \ref{fig:bathtub markov ordering observed} instead. However, the graphical structure would lead us to erroneously conclude that the soft intervention on $D$ \emph{only} has an effect on the variable $D$. In earlier work, \citet{Simon1988} assumed that a matching between variable and equation vertices is known in advance, allowing them to read off effects of soft interventions. We conclude that the Markov ordering graph, by itself, does not represent the effects of soft interventions on equations in general.

\paragraph{Perfect interventions.} In Example \ref{ex:bathtub perfect interventions} we found that a perfect intervention $\mathrm{do}(f_D,v_D,\xi_D)$ has an effect on the solution of the variables $v_P, v_O$ and $v_D$. If we would interpret this manipulation as a perfect intervention on $D$ in the Markov ordering graph in Figure \ref{fig:bathtub markov ordering observed} then we would mistakenly find that this intervention only affects the variable $D$. Since \citet{Iwasaki1994} do not make a distinction between variable vertices and equation vertices we could also interpret a manipulation on $D$ as the perfect intervention $\mathrm{do}(f_P,v_D,\xi_D)$ or $\mathrm{do}(f_D,v_O,\xi_O)$. From Table \ref{tab:bathtub solutions} we see that these perfect interventions would change the solution of the variables $\{v_D\}$ and $\{v_P, v_O, v_D\}$ respectively. Only the perfect intervention $\mathrm{do}(f_P,v_D,\xi_D)$ which targets the cluster containing $v_D$ corresponds to a perfect intervention on $D$ in the Markov ordering graph in Figure \ref{fig:bathtub markov ordering observed}. Since it is not clear from the Markov ordering graph what type of experiment a perfect intervention on one of its vertices should correspond to, we conclude that the Markov ordering graph cannot be used to read off the effects of perfect interventions.

\paragraph{Causal ordering graph.} The causal ordering graph for the bathtub system is given in Figure \ref{fig:bathtub directed cluster graph}. We proved that the causal ordering graph, contrary to the Markov ordering graph, represents the effects of soft interventions on equations and perfect interventions on clusters (see Theorems \ref{thm:markov property} and \ref{theo:effects of perfect interventions on clusters}). To derive causal implications from sets of equations we therefore propose to use the notion of the causal ordering graph instead. The distinction between variable vertices and equations vertices is also made by \citet{Simon1953} who shows how, for linear systems of equations, the principles of causal ordering can be used to qualitatively assess the effects of soft interventions on equations. A different, but closely related, notion of the causal ordering graph is used by \citet{Hautier2004} in the context of control systems modelling.

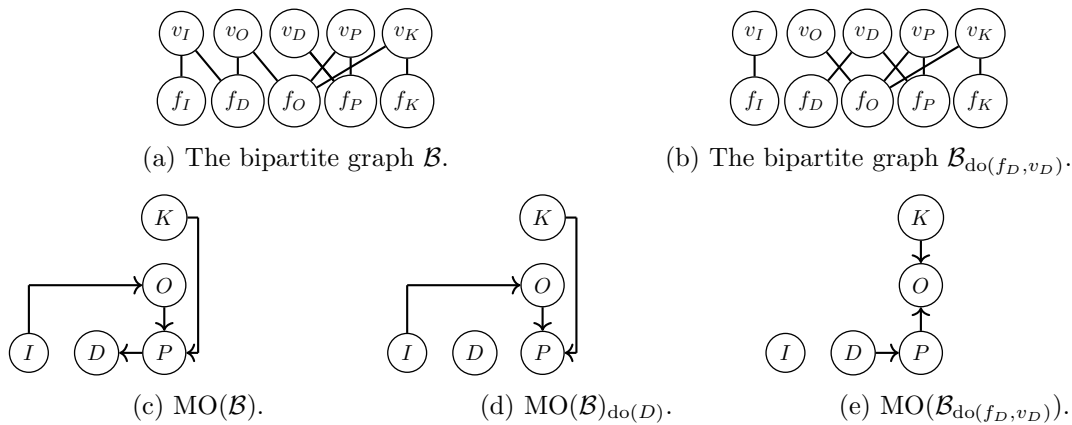
\begin{figure}[ht]
	\begin{subfigure}{0.5\textwidth}
		\centering
		\begin{tikzpicture}[scale=0.75,every node/.style={transform shape}]
		\GraphInit[vstyle=Normal]
		\SetGraphUnit{1}
		\SetVertexMath
		\Vertex{v_I}
		\EA(v_I){v_O}
		\EA(v_O){v_D}
		\EA(v_D){v_P}
		\EA(v_P){v_K}
		\SO[unit=1.2](v_I){f_I}
		\SO[unit=1.2](v_O){f_D}
		\SO[unit=1.2](v_D){f_O}
		\SO[unit=1.2](v_P){f_P}
		\SO[unit=1.2](v_K){f_K}
		\tikzset{EdgeStyle/.style = {-}}
		\Edge(f_I)(v_I)
		\Edge(f_D)(v_I)
		\Edge(f_D)(v_O)
		\Edge(f_O)(v_K)
		\Edge(f_O)(v_P)
		\Edge(f_O)(v_O)
		\Edge(f_P)(v_D)
		\Edge(f_P)(v_P)
		\Edge(f_P)(v_D)
		\Edge(f_K)(v_K)
		\end{tikzpicture}
		\caption{The bipartite graph $\BG$.}
		\label{fig:bathtub bipartite observed}
		\vspace{6pt}
	\end{subfigure}%
	\begin{subfigure}{0.5\textwidth}
		\centering
		\begin{tikzpicture}[scale=0.75,every node/.style={transform shape}]
		\GraphInit[vstyle=Normal]
		\SetGraphUnit{1}
		\SetVertexMath
		\Vertex{v_I}
		\EA(v_I){v_O}
		\EA(v_O){v_D}
		\EA(v_D){v_P}
		\EA(v_P){v_K}
		\SO[unit=1.2](v_I){f_I}
		\SO[unit=1.2](v_O){f_D}
		\SO[unit=1.2](v_D){f_O}
		\SO[unit=1.2](v_P){f_P}
		\SO[unit=1.2](v_K){f_K}
		\tikzset{EdgeStyle/.style = {-}}
		\Edge(f_I)(v_I)
		\Edge(f_D)(v_D)
		\Edge(f_O)(v_K)
		\Edge(f_O)(v_P)
		\Edge(f_O)(v_O)
		\Edge(f_P)(v_D)
		\Edge(f_P)(v_P)
		\Edge(f_K)(v_K)
		\end{tikzpicture}
		\caption{The bipartite graph $\BG_{\mathrm{do}(f_D,v_D)}$.}
		\label{fig:bathtub bipartite intervened}
		\vspace{6pt}
	\end{subfigure}
	\begin{subfigure}{0.33\textwidth}
		\begin{tikzpicture}[scale=0.75,every node/.style={transform shape}]
		\GraphInit[vstyle=Normal]
		\SetGraphUnit{1}
		\SetVertexMath
		\Vertex{I}
		\EA[unit=1.2](I){D}
		\EA[unit=1.2](D){P}
		\NO[unit=1.2](P){O}
		\NO[unit=1.2](O){K}
		\draw[EdgeStyle, style={-}](I) to (0,1.2);
		\draw[EdgeStyle, style={->}](0,1.2) to (O);
		\draw[EdgeStyle, style={->}](O) to (P);
		\draw[EdgeStyle, style={->}](P) to (D);
		\draw[EdgeStyle, style={-}](K) to (3.0,2.4);
		\draw[EdgeStyle, style={-}](3.0,2.4) to (3.0,0);
		\draw[EdgeStyle, style={->}](3.0,0) to (P);
		\end{tikzpicture}
		\caption{$\mathrm{MO}(\BG)$.}
		\label{fig:bathtub markov ordering observed}
	\end{subfigure}%
	\begin{subfigure}{0.33\textwidth}
		\begin{tikzpicture}[scale=0.75,every node/.style={transform shape}]
		\GraphInit[vstyle=Normal]
		\SetGraphUnit{1}
		\SetVertexMath
		\Vertex{I}
		\EA[unit=1.2](I){D}
		\EA[unit=1.2](D){P}
		\NO[unit=1.2](P){O}
		\NO[unit=1.2](O){K}
		\draw[EdgeStyle, style={-}](I) to (0,1.2);
		\draw[EdgeStyle, style={->}](0,1.2) to (O);
		\draw[EdgeStyle, style={->}](O) to (P);
		\draw[EdgeStyle, style={-}](K) to (3.0,2.4);
		\draw[EdgeStyle, style={-}](3.0,2.4) to (3.0,0);
		\draw[EdgeStyle, style={->}](3.0,0) to (P);
		\end{tikzpicture}
		\caption{$\mathrm{MO}(\BG)_{\mathrm{do}(D)}$.}
		\label{fig:bathtub markov ordering observed intervened}
	\end{subfigure}%
	\begin{subfigure}{0.33\textwidth}
		\begin{tikzpicture}[scale=0.75,every node/.style={transform shape}]
		\GraphInit[vstyle=Normal]
		\SetGraphUnit{1}
		\SetVertexMath
		\Vertex{I}
		\EA[unit=1.2](I){D}
		\EA[unit=1.2](D){P}
		\NO[unit=1.2](P){O}
		\NO[unit=1.2](O){K}
		\draw[EdgeStyle, style={->}](D) to (P);
		\draw[EdgeStyle, style={->}](P) to (O);
		\draw[EdgeStyle, style={->}](K) to (O);
		\end{tikzpicture}
		\caption{$\mathrm{MO}(\BG_{\mathrm{do}(f_D,v_D)})$.}
		\label{fig:bathtub markov ordering intervened}
	\end{subfigure}
	\caption{The bipartite graph for the bathtub system without exogenous variables is given in Figure \ref{fig:bathtub bipartite observed}. The intervened bipartite graph is given in Figure \ref{fig:bathtub bipartite intervened}. The Markov ordering graphs for the observed and intervened bathtub system are given in Figures \ref{fig:bathtub markov ordering observed} and \ref{fig:bathtub markov ordering intervened} respectively. Figure \ref{fig:bathtub markov ordering observed intervened} shows the graph that we obtain by intervening on the Markov ordering graph. Note that this does not correspond with the Markov ordering graph of the intervened bathtub system.}
	\label{fig:bathtub markov ordering does not commute}
\end{figure}

\subsection{Relation to other causal models}
\label{sec:discussion:relation to other causal models}

The results in this work are easily applicable to other modelling frameworks, such as the popular SCM framework \citep{Pearl2000,Bongers2020}. Application of causal ordering to the structural equations of an SCM with self-cycles may result in a different ordering than the one implied by the SCM. In particular, causal ordering may lead to a stronger Markov property and a representation of effects of a different set of (perfect) interventions. Even though the causal ordering graph itself may not allow us to read off non-effects of arbitrary perfect interventions, one can still obtain those by first intervening on the bipartite graph, then applying the causal ordering algorithm, and finally reading off the descendants of the intervention targets (under appropriate maximal unique solvability conditions).

\paragraph{Structural Causal Models.} In an SCM, each endogenous variable is on the left-hand side of exactly one structural equation and perfect interventions always act on a structural equation and its corresponding variable. In comparison, a system of constraints consists of symmetric equations and the asymmetric relations between variables are derived automatically by the causal ordering algorithm. Consider, for example, the following structural equations:
\begin{alignat}{3}
X_1 &= U_1 \\
X_2 &= a X_1 + U_2,
\intertext{where $X_1,X_2$ are endogenous variables, $U_1, U_2$ are exogenous random variables, and $a$ is a constant. The ordering $X_1\to X_2$ can also be obtained by causal ordering of the following set of equations:}
X_1 - U_1 &= 0,\\
X_2 - a X_1 - U_2 &= 0.
\end{alignat}
Note that any set of structural equations implies a self-contained set of equations.\footnote{In a set of structural equations each variable is matched to a single equation. Since the set of equations has a perfect matching it is self-contained by Hall's marriage theorem (see Theorem~\ref{thm:hall} in Appendix~\ref{app:equivalence proof}).} We can thus always apply the causal ordering algorithm to structural equations. Interestingly, since the output of the causal ordering algorithm is unique (see Theorem~\ref{thm:uniqueness}), the structure that is provided by the structural equations is actually redundant if the structural equations contain no cycles.
\paragraph{SCM for the bathtub.}

Recall that at equilibrium, the bathtub system can be described by the following structural equations:
\begin{alignat*}{3}
f_K:\qquad X_{v_K} &= U_{w_K}, &\qquad\qquad\qquad f_O:\qquad X_{v_O} &= U_{w_5} X_{v_K} X_{v_P}, \\
f_I:\qquad X_{v_I} &= U_{w_I}, &\qquad\qquad\qquad f_D:\qquad X_{v_D} &= X_{v_D} + U_{w_1} (X_{v_I} - X_{v_O}),\\
f_P:\qquad X_{v_P} &= g U_{w_3} X_{v_D}.
\end{alignat*}
The graph of this SCM is depicted in Figure \ref{fig:bathtub SCM}, and the descendants and non-descendants of vertices are given in Table \ref{tab:interventions SCM}. Can we use this table to read off generic causal effects of perfect interventions targeting $\{f_K, v_K\}$, $\{f_I, v_I\}$, $\{f_P, v_P\}$, $\{f_O, v_O\}$, and $\{f_D, v_D\}$? The graph of the SCM contains (self-)cycles and the SCM does not have a (unique) solution under each of these perfect interventions.\footnote{There is no (unique) solution if one fixes the outflow rate of the system $X_{v_O}$ to a value that is not equal to $X_{v_I}$ for the perfect interventions targeting $\{f_O,v_O\}$ and $\{f_P,v_P\}$. In the dynamical model for the bathtub, these perfect interventions would correspond with the water level becoming (plus or minus) infinity.} Therefore, the graph of this SCM may not have a straightforward causal interpretation. Indeed, \citet{Bongers2020} pointed out that for SCMs with cycles or self-cycles, the absence (presence) of directed edges and directed paths between vertices may not correspond one-to-one to the absence (generic presence) of direct and indirect causal effects, as it does in DAGs. (Self-)cycles may even lead to (in)direct causal effects without a corresponding directed edge or path being present in the graph of the SCM. For the bathtub example, that unusual behaviour does not occur, but instead it illustrates another behaviour: certain causal effects are absent, even though one would na\"ively expect these to be generically present based on the graph of the SCM.\footnote{Such behaviour is a characteristic of perfectly adaptive dynamical systems \citep{Blom2021}.}
For example, Table \ref{tab:interventions SCM} shows that $v_O$ is a descendant of $v_K$ in the graph of the SCM while the solution for the outflow rate $X_{v_O}$ does \emph{not} change after the perfect intervention $\mathrm{do}(f_K,v_K)$. That this causal relation is absent can actually be read off from the causal ordering graph in Figure \ref{fig:bathtub directed cluster graph}.

For the bathtub system, the causal ordering algorithm can exploit the fact that equation $f_D$ can be replaced by
$f_D' : 0 = U_{w_1} (X_{v_I} - X_{v_O})$, which does not involve $v_D$, whereas for the SCM this self-cycle cannot be
removed. This causes the following differences in the results of the two approaches:
\begin{enumerate}
  \item The d-separations in the Markov ordering graph in Figure \ref{fig:bathtub full markov ordering} imply more conditional independences than those implied by the $\sigma$-separations in the graph of the SCM in Figure~\ref{fig:bathtub SCM} (as was discussed in detail in Section~\ref{sec:application to bathtub}).
  \item The graph of the SCM and the causal ordering graph represent different perfect intervention targets. In the graph of the SCM, we have minimal perfect intervention targets of the form $\{f_i,v_i\}$ with $i\in \{K,I,P,O,D\}$, while the causal ordering graph represents minimal perfect interventions on clusters $\{f_K,v_K\}$, $\{f_I,v_I\}$, $\{f_P, v_D\}$, $\{f_O, v_P\}$, and $\{f_D, v_O\}$. In both cases, the set of all perfect intervention targets that are represented by the graph are obtained by taking unions of minimal perfect intervention targets.
  \item The causal ordering graph of the bathtub has a straightforward causal interpretation because the bathtub system still has a unique solution under interventions on clusters in the causal ordering graph. In contrast, the graph of the SCM for the bathtub system does not have a straightforward causal interpretation and the bathtub system does not have a solution under each perfect intervention on the SCM. 
\end{enumerate}
We conclude that the causal ordering approach yields a more ``faithful'' representation of the bathtub than the SCM framework.

\paragraph{Other frameworks.} Since the causal ordering algorithm can be applied to any set of equations, the results that we developed here are generally applicable to sets of equations in other modelling frameworks. For example, the recently introduced Causal Constraint Models (CCMs) do not yet have a graphical representation for the independence structure between the variables \citep{Blom2019}. The causal ordering algorithm can be directly applied to a set of active constraints to obtain a Markov ordering graph.

\begin{table}[!htb]
	\caption{The descendants and non-descendants of intervention targets in the graph of the SCM for the bathtub system in Figure \ref{fig:bathtub SCM}.}
	\label{tab:interventions SCM}
	\begin{center}
		\begin{tabular}{l  l  l}
			\toprule
			target & descendants & non-descendants \\
			\midrule
			$f_K, v_K$ & $v_K$, $v_P$, $v_O$, $v_D$ & $v_I$ \\
			$f_I, v_I$ & $v_I$, $v_P$, $v_O$, $v_D$ & $v_K$ \\
			$f_P, v_P$ & $v_P$, $v_O$, $v_D$ & $v_K$, $v_I$\\
			$f_O, v_O$ & $v_P$, $v_O$, $v_D$ & $v_K$, $v_I$\\
			$f_D, v_D$ & $v_P$, $v_O$, $v_D$ & $v_K$, $v_I$\\
			\bottomrule
		\end{tabular}
	\end{center}
\end{table}

\subsection{Equilibration in dynamical models}
\label{sec:equilibration in dynamical causal models}

In this subsection we will discuss in more detail the relation between our work and other closely related work, in particular to that of \citet{Dash2005}.

Dynamical models in terms of first order differential equations can be \emph{equilibrated to a set of equations} by equating each time-derivative to zero \citep{Mooij2013,Bongers2018}. They can be \emph{equilibrated and mapped to a causal ordering graph} by applying the causal ordering algorithm to the resulting set of equilibrium equations. They can also be \emph{equilibrated and mapped to a Markov ordering graph} by subsequently applying Definition \ref{def:declustering} to this causal ordering graph. The bathtub system provides an example of what \citet{Dash2005} calls a ``violation of the Equilibration Manipulation Commutability property''.\footnote{We argue that this is confusing terminology in two ways. First, what Dash calls ``equilibration'' is what we would call equilibration to a set of equations, composed with the mapping to the Markov ordering graph. Second, Dash follows \citet{Iwasaki1994} in referring to the Markov ordering graph as the ``causal graph''. We argued in Section \ref{sec:the causal graph a misnomer} that this is a misnomer, as in general there is no straightforward one-to-one correspondence between the Markov ordering graph and the causal semantics of the system. This terminological confusion explains the apparent contradiction with the result of \citet{Bongers2018}, who prove that equilibration to an SCM commutes with manipulation (for perfect interventions).} 

Consider the dynamical system version of the filling bathtub, with dynamical equations
\begin{align*}
f_{K}:&\qquad X_{v_K} = U_{w_K},\\
f_{I}:&\qquad X_{v_I} = U_{w_I},\\
f_{D}:&\qquad\dot{X}_{v_D}(t) = U_{w_1}(X_{v_I}(t)-X_{v_O}(t)),\\
f_{P}:&\qquad\dot{X}_{v_P}(t) = U_{w_2}(g\,U_{w_3}X_{v_D}(t)-X_{v_P}(t)),\\
f_{O}:&\qquad\dot{X}_{v_O}(t) = U_{w_4}(U_{w_5}X_{v_K}X_{v_P}(t)-X_{v_O}(t)).
\end{align*}
Equilibration yields the equilibrium equations $f_K$, $f_I$, $f_D$, $f_P$, and $f_O$ in equations \eqref{eq:dyn K} to \eqref{eq:dyn O}. It is clear in this particular case that any perfect intervention $\mathrm{do}(S_F,S_V,\B{\xi}_V)$ (where we extended Definition~\ref{def:perfect intervention} to dynamical equations) commutes with equilibration (substituting zeroes for all first-order derivatives).\footnote{Note that it is crucially important here to ensure that the \emph{labelling} of the equations is not changed by the equilibration operation.} This type of commutation relation actually holds also in more general settings (see \citet{Mooij2013} and \citet{Bongers2018}).

On the other hand, mapping a set of equations to the corresponding Markov ordering graph does not necessarily commute with perfect interventions. For example, for the perfect intervention $\mathrm{do}(f_D,v_D)$, the Markov ordering graphs $\mathrm{MO}(\BG)_{\mathrm{do}(v_D)}$ and $\mathrm{MO}(\BG_{\mathrm{do}(f_D,v_D)})$ are wildly different, as can be seen by comparing Figures \ref{fig:bathtub markov ordering observed intervened} and \ref{fig:bathtub markov ordering intervened} respectively. Since perfect interventions do commute with equilibration, one can conclude that also the composition of equilibration followed by mapping to the Markov ordering graph fails to commute with this perfect intervention. This is the phenomenon that \citet{Dash2005} pointed out.

This lack of commutability does not hold for \emph{all} perfect interventions. For example, one can easily check that the perfect intervention $\mathrm{do}(f_K,v_K)$ commutes with the composition of equilibration followed by mapping to the Markov ordering graph. More generally, Proposition \ref{prop:perfect interventions on clusters commute} tells us
that for the bathtub, the clusters in the causal ordering graph ($\{f_K,v_K\}$, $\{f_I,v_I\}$, $\{f_P, v_D\}$, $\{f_O, v_P\}$, and $\{f_D, v_O\}$) represent the minimal perfect interventions targets for which both operations do commute.
This means that of the perfect interventions that \citet{Dash2005} considers ($\mathrm{do}(\{v_K,f_K\})$, $\mathrm{do}(\{v_I,f_I\})$, $\mathrm{do}(\{v_D,f_D\})$, $\mathrm{do}(\{v_O,f_O\})$, $\mathrm{do}(\{v_P,f_P\})$, and combinations thereof), exactly three commute with the mapping to the Markov ordering graph (namely $\mathrm{do}(\{f_K,v_K\})$, $\mathrm{do}(\{f_I,v_I\})$, $\mathrm{do}(\{f_P,v_P,f_O,v_O,f_D,v_D\})$, and combinations thereof). Hence, these are also the three minimal perfect interventions in that set that commute with equilibration followed by mapping to the Markov ordering graph.

As pointed out by \citet{Dash2005}, this lack of commutability has important implications when one tries to discover causal relations through structure learning, which we will briefly discuss in the next subsection.

\subsection{Structure learning}
\label{sec:structure learning}

We have shown that, under a solvability assumption, d-separations in the Markov ordering graph (or $\sigma$-separations in the directed graph associated with a particular perfect matching) imply conditional independences between variables in a system of constraints (see Theorem \ref{thm:markov property} and Theorem \ref{theo:generalized markov property}). Constraint-based causal discovery algorithms relate conditional independences in data to graphs under the Markov condition and the corresponding d- or $\sigma$-faithfulness assumption. Roughly speaking, the equivalence class of the Markov ordering graph (or the directed graph associated with a particular perfect matching) can be learned from data under the assumption that all conditional independences in the data are implied by the graph. The bathtub system in Example \ref{ex:bathtub intro} is used by \citet{Dash2005}, who simulates data from the dynamical model until it reaches equilibrium, and then applies the PC-algorithm to learn the graphical structure of the system. It is no surprise that the learned structure is the Markov ordering graph in Figure \ref{fig:bathtub markov ordering observed}. The usual assumption is then that the Markov ordering graph equals the causal graph, where directed edges express direct causal relations between variables. In this work we have shown that this learned Markov ordering graph does \emph{not} have such a straightforward causal interpretation.

\section{Conclusion and future work}

In this work, we reformulated Simon's causal ordering algorithm and demonstrated that it is a convenient and scalable tool to study causal and probabilistic aspects of models consisting of equations. In particular, we showed how the technique of causal ordering can be used to construct a \emph{causal ordering graph} and a \emph{Markov ordering graph} from a set of equations, without calculating explicit global solutions to the system of equations. The novelties of this paper include an extension of the causal ordering algorithm for general bipartite graphs, and proving that the corresponding Markov ordering graph implies conditional independences between variables, whereas the corresponding causal ordering graph encodes the effects of soft and perfect interventions. 

To model causal relations between variables in sets of equations unambiguously, we generalized existing notions of perfect interventions on SCMs. The main idea is that a perfect intervention on a set of equations targets variables and specified equations, whereas a perfect intervention on a Structural Causal Model (SCM) targets variables and their associated structural equations. We considered a simple dynamical model with feedback and demonstrated that, contrary to claims in the literature, the Markov ordering graph does not generally have any obvious causal interpretation in terms of soft or perfect interventions. We showed that the causal ordering graph, on the other hand, does encode the effects of soft and certain perfect interventions. The main take-away is that we need to make a distinction between variables and equations in graphical representations of the probabilistic and causal aspects of models with feedback. By making this distinction, we clarified the correct interpretation of some existing results in the literature. Additionally, we shed new light on discussions in causal discovery about the justification of using a single directed graph with endogenous variables as vertices to simultaneously represent causal relations and conditional independences. We believe that the phenomenon where the Markov ordering graph does not encode causal semantics in the usual way manifests itself in certain biological or econometric models with feedback at equilibrium. In future work we plan to investigate these occurrences further \citep{Blom2020b,Blom2021}.


\acks{We thank Patrick Forr{\'e} and Stephan Bongers for fruitful discussions. This work was supported by the ERC under the European Union's Horizon 2020 research and innovation programme (grant agreement 639466) and by NWO (VIDI grant 639.072.410).}



\newpage

\appendix

\section{Preliminaries}
\label{sec:preliminaries}

\subsection{Graph terminology}
\label{sec:graph terminology}

A \emph{bipartite graph} is an ordered triple $\BG=\tuple{V,F,E}$ where $V$ and $F$ are disjoint sets of vertices and $E$ is a set of undirected edges $(v-f)$ between vertices $v\in V$ and $f\in F$. For a vertex $x\in V\cup F$ we write $\adj{\BG}{x}=\{y\in V\cup F: (x-y)\in E\}$ for its \emph{adjacencies}, and for $X\subseteq V\cup F$ we write $\adj{\BG}{X}=\bigcup_{x\in X}\adj{\BG}{x}$ to denote the adjacencies of $X$ in $\BG$. A \emph{matching} $\M\subseteq E$ for a bipartite graph $\BG=\tuple{V,F,E}$ is a subset of edges that have no common endpoints. We say that two vertices $x$ and $y$ are \emph{matched} when $(x-y)\in\M$. We let $\M(x)$ denote the set of vertices to which $x$ is matched. Note that if $(x-y)\in\M$ then $\M(x)=\{y\}$ and if $x$ is not matched then $\M(x)=\emptyset$. We let $\M(X)=\bigcup_{x\in X}\M(x)$ denote the set of vertices to which the set of vertices $X\subseteq V\cup F$ is matched. A matching is \emph{perfect} if all vertices $V\cup F$ are matched.

A \emph{directed graph} is an ordered pair $\G=\tuple{V,E}$ where $V$ is a set of vertices and $E$ is a set of directed edges $(v\to w)$ between distinct vertices $v,w\in V$. A \emph{directed mixed graph} is an ordered triple $\G=\tuple{V,E,B}$ where $\tuple{V,E}$ is a directed graph and $B$ is a set of bi-directed edges between vertices in $V$. If a directed mixed graph $\tuple{V,E,B}$ has an edge $(v\to v)\in E$ then we say that it has a \emph{self-cycle}. We say that a vertex $v$ is a \emph{parent} of $w$ if $(v\to w)\in E$ and write $v\in\pa{\G}{w}$. Similarly we say that $w$ is a \emph{child} of $v$ if $(v\to w)\in E$ and write $w\in\ch{\G}{v}$. A \emph{path} is a sequence of distinct vertices and edges $(v_1, e_1, v_2, e_2, \ldots, e_{n-1},v_n)$ where for $i=1,\ldots, n-1$ we have that $e_i=(v_i\to v_{i+1})$, $e_i=(v_i\leftarrow v_{i+1})$, or $e_i=(v_i\leftrightarrow v_{i+1})$. The path is called \emph{open} if there is no $v_i\in \{v_2,\ldots v_{n-1}\}$ such that there are two arrowheads at $v_i$ on the path (i.e.\ there is no collider on the path). A \emph{directed path} $(v\to \ldots \to w)$ from $v$ to $w$ is a path where all arrowheads point in the direction of $w$. We say that $v$ is an \emph{ancestor} of $w$ if there is a directed path from $v$ to $w$ and write $v\in\an{\G}{w}$. We say that $w$ is a \emph{descendant} of $v$ if there is a directed path from $v$ to $w$ and write $w\in\de{\G}{v}$.

Let $\G=\tuple{V,E,B}$ be a directed mixed graph and consider the relation:
\begin{equation*}
v\sim w \;\iff\; w\in\an{\G}{v}\cap\de{\G}{v} = \scc{\G}{v}.
\end{equation*}
Since the relation is reflexive, symmetric, and transitive this is an equivalence relation. The equivalence classes $\scc{\G}{v}$ are called the \emph{strongly connected components} of $\G$. A directed graph without self-cycles is \emph{acyclic} if and only if all of its strongly connected components are singletons. A directed graph with no directed cycles is called a Directed Acyclic Graph (DAG).

A \emph{perfect intervention} $\mathrm{do}(I)$ on a directed mixed graph $\G=\tuple{V,E,B}$ removes all edges with an arrowhead at any of the nodes $i\in I\subseteq V$. That is, $\G_{\mathrm{do}(I)}=\tuple{V,E',B'}$ where $E'=\{(x\to y)\in E: y\notin I\}$ and $B'=\{(x\leftrightarrow y)\in B: x\notin I, y\notin I\}$. \emph{Marginalizing} out a set of nodes $W \subseteq V$ from a directed mixed graph $\G=\tuple{V,E,B}$ results in a directed mixed graph $\G_{\mathrm{mar}(W)}=\tuple{V\setminus W, E_{\mathrm{mar}(W)},B_{\mathrm{mar}(W)}}$ (also known as the \emph{latent projection}) where:
\begin{enumerate}
	\item $E_{\mathrm{mar}(W)}$ consists of edges $(x\to y)$ such that $x,y\in V\setminus W$ and there exist $w_1,\ldots,w_k\in W$ such that the directed path $x\to w_1\to \ldots \to w_k \to y$ is in $\G$.
	\item $B_{\mathrm{mar}(W)}$ consists of edges $(x\leftrightarrow y)$ such that $x,y\in V\setminus W$ and there exist $w_1,\ldots, w_k\in W$ such that at least one of the following paths is in $\G$: (i) $x\leftrightarrow y$, or (ii) $x \leftarrow w_1 \leftarrow \ldots \leftarrow w_i \rightarrow \ldots \rightarrow w_k \rightarrow y$, or (iii) $x \leftarrow w_1 \leftarrow \ldots \leftarrow w_i \leftrightarrow w_{i+1} \rightarrow \ldots \rightarrow w_k \rightarrow y$.
\end{enumerate}
The operations of marginalization and intervention commute \citep{Forre2017}.

\subsection{Cyclic SCMs}
\label{sec:cyclic SCMs}

Structural causal models are a popular causal modelling framework that form the basis of many statistical methods for causal inference \citep{Pearl2000}. Their origins can be traced back to early work in genetics \citep{Wright1921}, econometrics \citep{Wright1928, Haavelmo1943}, and the social sciences \citep{Goldberger1973}. The properties of acyclic SCMs (i.e.\ recursive SEMs) have been widely studied and are well-understood, see for example \citet{Lauritzen1990, Spirtes2000, Pearl2000}. For systems that have causal cycles the class of cyclic SCMs has been proposed as an appropriate modelling class \citep{Spirtes1995, Mooij2013}. Recently, \citet{Forre2017, Bongers2020} showed that \emph{modular} and \emph{simple} SCMs retain many of the attractive properties of acyclic SCMs. Notably, they induce a unique distribution over variables, they obey a Markov property, and their graphs have an intuitive causal interpretation. Here, we will closely follow \citet{Bongers2020} for a succinct introduction to cyclic SCMs and their properties. We also discuss literature on how they may arise from equilibrating dynamical models.

The definition of an SCM in \citet{Bongers2020} slightly deviates from previous notions of (acyclic) SCMs because it separates the model from the (endogenous) random variables that solve it. Due to this change, interventions on SCMs are always well-defined, even if the resulting \emph{intervened} SCM does not have a (unique) solution. In Definition \ref{def:scm} below, we explicitly include exogenous random variables, which may be observed or unobserved, and the graph of the SCM. The endogenous random variables that solve an SCM are defined in Definition \ref{def:scm solution}.

\begin{definition}
\label{def:scm}
A structural causal model (SCM) is a tuple $\tuple{\B{\X}, \P_W, \B{f}, \G}$ where
\begin{enumerate}
	\item $\B{\X}=\bigotimes_{v\in V}\X_v$, where each $\X_v$ is a standard measurable space and the domain of a variable $X_v$,
  \item $\P_W=\prod_{w\in W} \P_w$ specifies the \emph{exogenous distribution}, a product probability measure on $\bigotimes_{w\in W}\X_w$, where each $\P_w$ is a probability measure on $\X_w$, with $W\subseteq V$ a set of indices corresponding to exogenous variables,\footnote{This means that the nodes $V\setminus W$ correspond to endogenous variables.}
	\item $\B{f}:\B{\X}_{V} \to \B{\X}_{V\setminus W}$ is a measurable function that specifies \emph{causal mechanisms}.\footnote{The structural equations of the model are given by $x_v=f_v(\B{x})$, $\B{x}\in\B{\X}$ for $v\in V\setminus W$.} 
	\item $\G=\tuple{V, E}$ is a directed graph with:
	\begin{enumerate}
		\item a set of nodes $V$ corresponding to variables,
		\item a set of edges $E=\{(v_i\to v_j): v_i \text{ is a parent of } v_j\}$.\footnote{We say that $v_i$ is a parent of $v_j$ if and only if $v_j \in V\setminus W$ and there does not exist a measurable function $\tilde{f}_j:\B{\X}_{V \setminus \{v_i\}}\to \X_j$ such that for $\P_W$-almost every $\B{x}_W\in\B{\X}_W$ and for all $\B{x}_{V\setminus W}\in\B{\X}_{V\setminus W}$ we have $x_j=f_j(\B{x}) \iff x_j=\tilde{f}_j(\B{x}_{V \setminus \{v_i\}})$, see Definition 2.7 in \citet{Bongers2020}.}
	\end{enumerate}
\end{enumerate}
\end{definition}

\begin{definition}
\label{def:scm solution}
We say that a random variable $\B{X}$ taking value in $\B{\X}$ is a \emph{solution to an SCM} $\tuple{\B{\X}, \P_W, \B{f}, \G}$ if 
$\P^{\B{X}_W} = \P_W$ (i.e., if the marginal distribution of $\B{X}$ on $\B{\X}_W$ equals the exogenous distribution specified
by the SCM), and
\begin{align}
\B{X}_{V\setminus W} = \B{f}(\B{X}) \quad \text{a.s.}
\end{align}
\end{definition}

The notion of \emph{unique solvability w.r.t.\ a subset} is given in Definition \ref{def:scm unique solvability} below.
\begin{definition}
\label{def:scm unique solvability}
  An SCM $\tuple{\B{\X}, \P_W, \B{f}, \G}$ is uniquely solvable w.r.t.\ $S\subseteq V\setminus W$ if there exists a measurable function $\B{g}_S:\B{\X}_{\mathrm{pa}_{\G}(S)\setminus S} \to \B{\X}_S$ such that 
  for $\P_W$-almost every $\B{x}_W\in\B{\X}_W$ and for all $\B{x}_{V\setminus W}\in\B{\X}_{V\setminus W}$
\begin{align}
\B{x}_S = \B{g}_S(\B{x}_{\mathrm{pa}_{\G}(S)\setminus S}) \iff \B{x}_S = \B{f}_S(\B{x}).
\end{align}
\end{definition}

SCMs that are uniquely solvable w.r.t.\ every subset of variables are called \emph{simple SCMs} \citep{Bongers2020}. It can be shown that SCMs with acyclic graphs are simple SCMs (Proposition 3.6 in \citet{Bongers2020}). Furthermore, SCMs are uniquely solvable w.r.t.\ a single variable if and only if there is no self-cycle at that variable (Proposition 3.9 in \citet{Bongers2020}). The notion of (perfect) interventions on an SCM is given in Definition \ref{def:scm intervention}.

\begin{definition}
\label{def:scm intervention}
Let $\mathcal{M}=\tuple{\B{\X}, \P_W, \B{f}, \G}$ be an SCM, $I\subseteq V$ an \emph{intervention target} and $\B{\xi}_I\in\B{\X}_I$ the \emph{intervention value}. A perfect intervention $\mathrm{do}(I,\B{\xi}_I)$ on the SCM maps it to an intervened SCM $\mathcal{M}_{\mathrm{do}(I,\B{\xi}_I)}=\tuple{\B{\X}, \P_W, \widetilde{\B{f}}, \G_{\mathrm{do}(I)}}$ with
\begin{align}
\widetilde{f}_v(\B{x}) := \begin{cases}
\xi_v &\qquad v\in I \\
f_v(\B{x}) &\qquad v\in V\setminus I.
\end{cases}
\end{align}
\end{definition}

Cyclic SCMs may have no solution, multiple solutions with different distributions, or all solutions may have the same distribution. This may even change as a result of a perfect intervention. Because changes in the solution after an intervention may not be compatible with the structure of the functional relations between variables, the causal interpretation of cyclic SCMs may not be intuitive \citep{Bongers2020}. It can be shown that the graph of a simple SCMs, whose unique solvability is preserved under intervention (Proposition 8.2 in \citet{Bongers2020}), has an intuitive causal interpretation; direct and indirect causal effects can be read off from the graph of the SCM by checking for the presence of directed edges and directed paths between variables \citep{Bongers2020}. For general cyclic models, \citet{Bongers2020} give a sufficient condition for detecting direct and indirect causes in an SCM with cycles. Roughly speaking, an indirect cause $v_i$ of $v_j$ can be detected if by controlling $v_i$ we can bring about a change in the distribution of $v_j$ and a direct cause $v_d$ of $v_j$ can be detected if by controlling $v_d$ and keeping all other variables constant we can bring about a change in the distribution of $v_j$. For the exact formulation we refer to Proposition 7.1 in \citet{Bongers2020}.

Cyclic SCMs have been used to represent the equilibrium distribution of dynamical models \citep{Fisher1970,Spirtes1995, Richardson1996, Lauritzen2002, Mooij2013, Bongers2018}. Under certain stability assumptions, an SCM can be obtained by \emph{equilibrating} a dynamical model \citep{Dash2005, Mooij2013, Bongers2018}. In the deterministic setting, \citet{Mooij2013} showed that a set of first-order differential equations in a globally asymptotically stable system can be mapped to a set of \emph{labelled equilibrium equations} by setting the time derivatives of variables equal to zero and labelling them as belonging to the time derivative of particular variables. If each labelled equilibrium equation can be solved for the corresponding variable then the labelled equilibrium equations can be mapped to a cyclic SCM without self-cycles. The idea that a dynamical model can be \emph{equilibrated to an SCM} was formalized in a general stochastic setting with zeroth and higher order differential equations by \citet{Bongers2018}, who also show how to equilibrate the  \emph{causal dynamics model} of the bathtub system that we discussed in Example \ref{ex:bathtub intro} to an SCM with self-cycles.

\subsection{Graph separation and Markov properties}
\label{sec:graph sep and mp}

In the literature, several versions of Markov properties for graphical models and corresponding probability distributions have been put forward, see e.g.\ \citet{Lauritzen1990, Pearl2000, Spirtes2000, Forre2017}. For DAGs and Acyclic Directed Mixed Graphs (ADMGs), the d-separation criterion is often used to relate conditional independences between variables in a model to the underlying (acyclic) graphical structure of the model \citep{Pearl2000}. For graphs that contain cycles the `collapsed graph' representation of \citet{Spirtes1995} inspired \citet{Forre2017} to introduce the $\sigma$-separation criterion.

\begin{definition}
	\label{def:separation criteria}
	For a directed mixed graph $\G=\tuple{V,E,B}$ we say that a path $(v_1,\ldots,v_n)$ \emph{is $\sigma$-blocked by $Z\subseteq V$} if
	\begin{enumerate}
		\item $v_1 \in Z$ and/or $v_n\in Z$, or \label{cond:sigma-block1}
		\item there is a vertex $v_i\notin \an{\G}{Z}$ on the path such that the adjacent edges both have an arrowhead at $v_i$, or \label{cond:sigma-block2}
		\item there is a vertex $v_i\in Z$ on the path such that: $v_i\to v_{i+1}$ with $v_{i+1}\notin \scc{\G}{v_i}$, or $v_{i-1}\leftarrow v_i$ with $v_{i-1}\notin \scc{\G}{v_i}$, or both.
	\end{enumerate}
	The path is \emph{$d$-blocked by $Z$} if it is $\sigma$-blocked or if there is a vertex $v_i\in Z$ on the path such that at least one of the adjacent edges does not have an arrowhead at $v_i$. We say that $X\subseteq V$ and $Y\subseteq V$ are \emph{$\sigma$-separated} by $Z\subseteq V$ if every path in $\G$ with one end-vertex in $X$ and one end-vertex in $Y$ is $\sigma$-blocked by $Z$, and write
	\begin{equation*}
	X \sigsep{\G} Y \given Z.
	\end{equation*}
	If every such path is $d$-blocked by $Z$ then we say that \emph{$X$ and $Y$ are $d$-separated by $Z$}, and write
	\begin{equation*}
	X \dsep{\G} Y \given Z.
	\end{equation*}
\end{definition}

It can be shown that $\sigma$-separation implies d-separation and that the two are equivalent for acyclic graphs \citep{Forre2017}. In general, d-separation does not imply $\sigma$-separation. The d-separations or $\sigma$-separations in a probabilistic graphical model may imply conditional independences via the Markov properties in Definition \ref{def:markov property} below.

\begin{definition}
	\label{def:markov property}
	For a directed graph $\G=\tuple{V,E}$ and a probability distribution $\P_{\B{X}}$ on a product $\B{\X}=\otimes_{v\in V}\X_v$ of standard measurable spaces $\X_v$, we say that the pair $(\G,\P_{\B{X}})$ satisfies the \emph{directed global Markov property} if for all subsets $W,Y,Z\subseteq V$:
	\begin{equation*}
	W \dsep{\G} Y \given Z \implies \B{X}_W \ind{\P_{\B{X}}} \B{X}_Y \given \B{X}_Z.
	\end{equation*}
	The pair $(\G,\P_{\B{X}})$ satisfies the \emph{generalized directed global Markov property} if for all subsets $W,Y,Z\subseteq V$:
	\begin{equation*}
	W \sigsep{\G} Y \given Z \implies \B{X}_W \ind{\P_{\B{X}}} \B{X}_Y \given \B{X}_Z.
	\end{equation*}
\end{definition}

Since $\sigma$-separations imply d-separations but not the other way around, the generalized directed global Markov property is strictly weaker than the directed global Markov property \citep{Bongers2020}. For acyclic SCMs the induced probability distribution on endogenous variables and the corresponding DAG satisfy the directed global Markov property \citep{Lauritzen1990}. The variables that solve a simple SCM obey the generalized directed global Markov property relative to the graph of the SCM \citep{Bongers2020}, while d-separation is limited to more specific settings such as acyclic models, discrete variables, or continuous variables with linear relations \citep{Forre2017}. A comprehensive account of different Markov properties for graphical models is provided by \citet{Forre2017}.

Constraint-based causal discovery algorithms require some additional faithfulness assumption. A probability distribution is \emph{d-faithful} to a directed mixed graph when each conditional independence implies a d-separation in that graph. Similarly, a probability distribution is \emph{$\sigma$-faithful} to a directed mixed graph when each conditional independence implies a $\sigma$-separation in that graph. In non-linear, non-discrete, cyclic settings the $\sigma$-faithfulness assumption is a natural extension of the notion of the common $d$-faithfulness assumption with $\sigma$-separation replacing $d$-separation. Under the additional assumption of causal sufficiency (i.e., no latent confounding variables), the NL-CCD algorithm was shown to be sound under the generalized directed Markov property and the weaker d-faithfulness assumption (Chapter 4 in \citet{Richardson1996}). Recently, \citet{Forre2018, Mooij2020a, Mooij2020b} proved soundness for a variety of causal discovery algorithms under the generalized directed Markov property and the $\sigma$-faithfulness assumption. \citet{Strobl2018} proved soundness of a causal discovery algorithm under the directed Markov property and the d-faithfulness assumption, allowing for latent confounding and selection bias.

\section{Proofs}
\label{app:proofs}

In this section of the appendix, all proofs are provided.

\subsection{Causal ordering via minimal self-contained sets}
\label{app:uniqueness proof}

In this section we prove Theorem \ref{thm:uniqueness} below.

\uniqueness*

Lemma \ref{lemma:disjoint} below shows that the minimal self-contained sets in a self-contained bipartite graph are disjoint. Lemma \ref{lemma:ind-selfcont} shows that the induced subgraph after one iteration of Algorithm \ref{alg:causal ordering minimal self-contained exo}, with a self-contained bipartite graph as input, is self-contained. The minimal self-contained sets in the graph which are not used in the iteration are minimal self-contained sets of the induced subgraph. This shows that the output of Algorithm \ref{alg:causal ordering minimal self-contained exo} is well-defined. We then use Lemma \ref{lemma:disjoint} and \ref{lemma:ind-selfcont} to prove Lemma \ref{lemma:unique} which states that the output of Algorithm \ref{alg:causal ordering minimal self-contained exo}, with a self-contained bipartite graph as input, is unique. This implies that the output of Algorithm \ref{alg:causal ordering minimal self-contained exo}, which has an initialization that is uniquely determined by the specification of exogenous variables $W$, must also be unique.

\begin{lemma}\label{lemma:disjoint}
	Let $\BG=\tuple{V,F,E}$ be a self-contained bipartite graph. Let $\S_F$ be the set of minimal self-contained sets in $\BG$. The sets in $\S_F$ are pairwise disjoint, and, likewise, the sets of adjacent nodes 
	\begin{equation*}
	\S_V=\{\adj{\BG}{S} : S\in\S_F\},
	\end{equation*}
	of the minimal self-contained sets in $\S_F$ are pairwise disjoint.
\end{lemma}

\begin{proof}
	Let $S_1\subseteq F$ and $S_2\subseteq F$ be non-empty distinct minimal self-contained sets in $\S_F$. For the sake of contradiction, assume that $S_1\cap S_2 \neq \emptyset$. Since $S_1$ is minimal self-contained, we know that $S_1\cap S_2\subset S_1$ is not self-contained. Hence, by Definition \ref{def:self-contained}, we have that
	\begin{align}
	\label{eq:strict inequality}
	|S_1\cap S_2| < |\adj{\BG}{S_1 \cap S_2}|.
	\end{align}
	Consider the following equations:
	\begin{align}
	\hspace{70pt}&\hspace{-70pt}
	|\adj{\BG}{S_1}|+|\adj{\BG}{S_2}| - |S_1\cap S_2| \\
  &= |S_1|+|S_2| - |S_1 \cap S_2| \label{eq:msc} \\
	&= |S_1\cup S_2| \nonumber \\
	&\le |\adj{\BG}{S_1\cup S_2}| \label{eq:sc} \\
	&= |\adj{\BG}{S_1}\cup\adj{\BG}{S_2}| \nonumber \\
	&= |\adj{\BG}{S_1}|+|\adj{\BG}{S_2}| - |\adj{\BG}{S_1}\cap\adj{\BG}{S_2}|\nonumber \\
	&\le |\adj{\BG}{S_1}|+|\adj{\BG}{S_2}| - |\adj{\BG}{S_1\cap S_2}| \label{eq:dj},
	\end{align}
	where equality \eqref{eq:msc} holds by condition \ref{SC1} of Definition \ref{def:self-contained}, since $\BG$ is self-contained inequality \eqref{eq:sc} holds by condition \ref{SC2} of Definition \ref{def:self-contained}, and inequality \eqref{eq:dj} holds because $\adj{\BG}{S_1\cap S_2} \subseteq \adj{\BG}{S_1}\cap\adj{\BG}{S_2}$. It follows that
	\begin{equation*}
	|S_1\cap S_2| \ge |\adj{\BG}{S_1}\cap \adj{\BG}{S_2}| \ge |\adj{\BG}{S_1\cap S_2}|\geq 0.
	\end{equation*}
	This is in contradiction with equation \eqref{eq:strict inequality}, and hence $S_1\cap S_2=\emptyset$. This implies that $|S_1 \cap S_2|=0$ and therefore by the inequalities above we have that $|\adj{\BG}{S_1}\cap \adj{\BG}{S_2}|=0$. Thus $\adj{\BG}{S_1}\cap\adj{\BG}{S_2}=\emptyset$.
\end{proof}

\begin{lemma}\label{lemma:ind-selfcont}
	Let $\BG=\tuple{V,F,E}$ be a self-contained bipartite graph. Suppose that $F$ has minimal self-contained sets $\S_F$. Let $\BG'$ be the subgraph of $\BG$ induced by
	\begin{equation*}
	V':=V \backslash \adj{\BG}{S}, \quad \text{and } \quad F':=F\backslash S,
	\end{equation*}
	with $S\in\S_F$. Then the following two properties hold:
	\begin{enumerate}

		\item $\BG'$ is self-contained, and

		\item the sets in $\S_F\backslash\{S\}$ are minimal self-contained in $\BG'$.

	\end{enumerate}

\end{lemma}

\begin{proof}

	Let $S\in\S_F$ be a minimal self-contained subset in $\BG$. Since $\BG$ and $S$ are self-contained we have that $|V|=|F|$ and $|S|=|\adj{\BG}{S}|$ respectively. Therefore
	\begin{align*}
	|V'|=|V\setminus \adj{\BG}{S}|=|V|-|\adj{\BG}{S}|=|F|-|S| = |F\setminus S|=|F'|.
	\end{align*}

	This shows that condition \ref{SC1} of Definition \ref{def:self-contained} is satisfied for $\BG'$. Assume, for the sake of contradiction, that $F'$ does not satisfy condition \ref{SC2} of Definition \ref{def:self-contained} in the induced subgraph $\BG'$. Then there exists $S'\subseteq F'$ such that $|S'|>|\adj{\BG'}{S'}|$. Consider the following equations:
	\begin{align*}
	|S\cup S'|&=|S|+|S'|\\
	&>|\adj{\BG}{S}| + |\adj{\BG'}{S'}|\\
	&=|\adj{\BG}{S}| + |\adj{\BG}{S'}| - |\adj{\BG}{S}\cap \adj{\BG}{S'}| \\
	&= |\adj{\BG}{S}\cup\adj{\BG}{S'}|\\
	&= |\adj{\BG}{S\cup S'}|\\
	&\geq |S\cup S'|,
	\end{align*}
	where the last inequality holds because $\BG$ is self-contained by assumption. This is a contradiction, and we conclude that both conditions of Definition \ref{def:self-contained} are satisfied for $\BG'$. This shows that $\BG'$ is self-contained.

	Let $S_1\in \S_F$ and $S_2\in S_F$ be two distinct minimal self-contained sets in $\BG$. Suppose that $\BG_1$ is a subgraph of $\BG$ induced by $V\setminus\adj{\BG}{S_1}$ and $F\setminus S_1$. By Lemma \ref{lemma:disjoint} we know that $S_1\cap S_2=\emptyset$ and $\adj{\BG}{S_1}\cap \adj{\BG}{S_2}=\emptyset$. It follows that for all $S'\subseteq S_2$ we have that $\adj{\BG}{S'}=\adj{\BG_1}{S'}$. We find that
	\begin{align*}
	|S_2| &= |\adj{\BG}{S_2}| = |\adj{\BG_1}{S_2}|,\\
	|S'| &\leq |\adj{\BG}{S'}| = |\adj{\BG_1}{S'}|,
	\end{align*}
	for all $S'\subseteq S_2$. This shows that $S_2$ satisfies the conditions of Definition \ref{def:self-contained} in the bipartite graph $\BG_1$. Since $S_2$ has no non-empty strict subsets that are self-contained in $\BG$ we have that $S_2$ has no non-empty strict subsets that are self-contained in $\BG_1$. We conclude that $S_2$ is a minimal self-contained subset in $\BG_1$. This shows that the sets $\S_F\setminus\{S\}$ are minimal self-contained in $\BG'$.

\end{proof}

\begin{lemma}
	\label{lemma:unique}

	Let $\BG=\tuple{V,F,E}$ be a self-contained bipartite graph. The output $\mathrm{CO}(\BG)$ of Algorithm \ref{alg:causal ordering minimal self-contained exo} is unique.

\end{lemma}

\begin{proof}

	Suppose $\G_1=\tuple{\mathcal{V}_1,\mathcal{E}_2}$ and $\G_2=\tuple{\mathcal{V}_2,\mathcal{E}_2}$ are directed cluster graphs that are obtained by running Algorithm \ref{alg:causal ordering minimal self-contained exo}. Let $A=(1,2,\ldots,|\mathcal{V}_1|)$ be an ordered set that indicates the order in which clusters $S^{(a)}$ (with $a\in A$) are added to $\mathcal{V}_1$ in the first run of the algorithm. Similarly,  $B=(1,2,\ldots,|\mathcal{V}_2|)$ is an ordered set that indicates the order in which clusters $T^{(b)}$ (with $b\in B$) are added to $\mathcal{V}_2$ in the second run of the algorithm. With a slight abuse of notation we define $\BG\setminus (S^{(k)})_{k<i}$ as the subgraph of $\BG$ induced by the nodes $(S^{(k)})_{k\geq i}$. Similarly, $\BG\setminus (T^{(k)})_{k<i}$ denotes the subgraph of $\BG$ induced by the nodes $(T^{(k)})_{k\geq i}$.

	\paragraph{\normalfont{}\emph{Intermediate result:}} We will prove that for $i \in (1,2,\ldots,|\mathcal{V}_1|)$ there exists $b_i\in B$ such that $S^{(i)}=T^{(b_i)}$ by induction.

	\paragraph{\normalfont{}\emph{Base case:}} The algorithm adds the cluster $S^{(1)}$ to $\V_1$ in the first step of the first run. Therefore, we know that the set of nodes $F \cap S^{(1)}$ must be minimal self-contained in $\BG$. Let $1 \leq k \leq |\V_2|$ be arbitrary. By Lemma \ref{lemma:ind-selfcont} it follows that $F\cap S^{(1)}$ is minimal self-contained in $\BG\setminus (T^{(j)})_{j<k}$ provided $S^{(1)} \neq T^{(j)}$ for all $j<k$. Since $\BG$ is finite, the minimal self-contained set $S^{(1)}$ must be chosen eventually, and hence there exists $b_1\in B$ such that $S^{(1)}=T^{(b_1)}$.

	\paragraph{\normalfont{}\emph{Induction hypothesis:}}  Let $1 \leq i < |\V_1|$ be arbitrary and assume that for all $j\leq i$ there exists $b_j\in B$ such that $S^{(j)}=T^{(b_j)}$. We want to show that there exists $b_{i+1}\in B$ such that $S^{(i+1)}=T^{(b_{i+1})}$.

	\paragraph{\normalfont{}\emph{Induction step:}} Let $B'=B \setminus (b_1,\ldots, b_i) = (b'_1, \ldots, b'_{|\mathcal{V}_2| - i})$ be an ordered set such that $b'_{j} \prec b'_{j+1}$ for all $j=1,\ldots, |\mathcal{V}_2|-(i+1)$. 

	\begin{enumerate}

		\item In the second run of the algorithm, the cluster $T^{(b'_1)}$ is added to $\V_2$ right after the clusters $T^{(b_j)}$ with $b_j\prec b'_1$ are added to $\V_2$ and removed from the bipartite graph. Therefore, the set $F\cap T^{(b'_1)}$ is minimal self-contained in $\BG\setminus (T^{(b_j)})_{j\leq i, b_j \prec b'_1}$. In the first run of the algorithm, the clusters $S^{(1)}=T^{(b_1)},\ldots, S^{(i)}=T^{(b_i)}$ are subsequently added to $\V_1$ and removed from the bipartite graph. Therefore, by Lemma \ref{lemma:disjoint} and Lemma \ref{lemma:ind-selfcont}, we have that $F\cap T^{(b'_1)}$ is minimal self-contained in $\BG'=\BG\setminus (T^{(b_j)})_{j\leq i} =\BG\setminus (S^{(k)})_{k\leq i}$. Hence, both $F\cap T^{(b'_1)}$ and $F\cap S^{(i+1)}$ are minimal self-contained in $\BG'$. Therefore, by Lemma \ref{lemma:disjoint} and Lemma \ref{lemma:ind-selfcont}, either $T^{(b'_1)}=S^{(i+1)}$ (in which case we are done) or $F\cap S^{(i+1)}$ is minimal self-contained in $\BG'\setminus T^{(b'_1)}$.

		\item Let $k\leq |\mathcal{V}_2|-i$ be arbitrary. By iteration of the argument in the previous step we find that $F\cap T^{(b'_k)}$ is minimal self-contained in $(\BG\setminus (T^{(b_j)})_{j\leq i, b_j \prec b'_k})\setminus (T^{(b'_j)})_{j<k}$ and hence in $\BG'\setminus (T^{(b'_j)})_{j<k}$, so that either $T^{(b'_k)}=S^{(i+1)}$ or $F\cap S^{(i+1)}$ is minimal self-contained in $\BG'\setminus (T^{(b'_j)})_{j\leq k}$. Since the bipartite graph is finite, there exists $m\in 1,\ldots, |\V_2|-i$ such that $T^{(b'_m)}=S^{(i+1)}$. By definition of $B'$ there exists $b_{i+1}\in B$ such that $S^{(i+1)}=T^{(b_{i+1})}$.		

	\end{enumerate}

	This proves that the clusters in $\V_1$ are also clusters in $\V_2$. By symmetry we find that the clusters $S^{(a)}$ in $\mathcal{V}_1$ and the clusters $T^{(b)}$ in $\mathcal{V}_2$ coincide. Since $\mathcal{V}_1=\mathcal{V}_2$ it follows immediately from the construction of edges in the algorithm that $\E_1=\E_2$ and hence $\G_1=\G_2$.

\end{proof}

\subsection{Coarse decomposition}
\label{appendix:coarse decomposition proofs}

For completeness, we include the proofs of the results in \citet{Pothen1990} that are necessary to show that the output of the extended causal ordering algorithm (Algorithm \ref{alg:causal ordering coarse decomposition}) is unique. The presentation in this section is based on the exposition of \citet{Diepen2019}. In order to prove the statements in Lemma \ref{lemma:impossible edges} and Proposition \ref{prop:coarse decomposition unique}, we require additional results. Lemma \ref{lemma:berge} and \ref{lemma:disjoint coarse decomposition} show that the incomplete, complete, and over-complete set are disjoint. The former uses the notion of an \emph{augmented path} for a bipartite graph $\BG$ and a matching $M$, which is an alternating path for $M$ that starts and ends with an unmatched vertex.
\begin{lemma}
\label{lemma:berge}
[\citet{Berge1957}] $M$ is a maximum matching for a bipartite graph $\BG$ if and only if $\BG$ does not contain any augmenting paths for $M$.
\end{lemma}
\begin{proof}
The proof can be found in \citet{Berge1957}.
\end{proof}

\begin{lemma}
\label{lemma:disjoint coarse decomposition}
[\citet{Pothen1985}] Let $\BG=\tuple{V,F,E}$ be a bipartite graph with a maximum matching $M$. The incomplete set $T_I$ and the overcomplete set $T_O$ in Definition \ref{def:coarse decomposition} are disjoint.
\end{lemma}
\begin{proof}
For the sake of contradiction, assume that there is a vertex $v\in T_I\cap T_O$. Then there is an alternating path from an unmatched vertex in $V$ to $v$ and there is also an alternating path from an unmatched vertex in $F$ to $v$. By sticking these two paths together we obtain an augmented path. It follows from Lemma \ref{lemma:berge} that $M$ is not maximum. This is a contradiction and therefore $T_I$ and $T_O$ must be disjoint.
\end{proof}

Lemma \ref{lemma:matched vertices} and Lemma \ref{lemma:complete graph self-contained} show that for a bipartite graph and a maximum matching with coarse decomposition $\mathrm{CD}(\BG,M)$, the vertices in $T_I,T_C,T_O$ are matched to vertices in $T_I,T_C,T_O$ respectively. Furthermore the subgraph of $\BG$ induced by $T_C$ is self-contained, so that Algorithm \ref{alg:causal ordering minimal self-contained exo} can be applied.

\begin{lemma}
\label{lemma:matched vertices}
[\citet{Pothen1985}] Let $\BG=\tuple{V,F,E}$ be a bipartite graph with a maximum matching $M$. Let $\mathrm{CD}(\BG,M)=\tuple{T_I,T_C,T_O}$ be the associated coarse decomposition. A matched vertex in $T_I$ is matched to a vertex in $T_I$ and a matched vertex in $T_O$ is matched to a vertex in $T_O$.
\end{lemma}

\begin{proof}
For a matched vertex $x\in T_I$ there is an alternating path starting from an unmatched vertex $u_v\in V$ to $x$. When $x\in V$, this alternating path ends with a matched edge and hence $x$ is matched to a vertex in $T_I$. When $x\in F$ the alternating path ends with an unmatched edge. We may extend the alternating path with the edge adjacent to $x$ that is in $M$, and hence is matched to a vertex in $T_I$. For a matched vertex $x\in T_O$ there is an alternating path starting from an unmatched vertex $u_f\in F$ to $x$. When $x\in F$, this alternating path ends with a matched edge and hence $x$ is matched to a vertex in $T_O$. When $x\in V$, the alternating path ends with an unmatched edge. The alternating path may be extended with the edge adjacent to $x$ that is in $M$, and hence $x$ is matched to a vertex in $T_O$.
\end{proof}

\completegraphselfcontained*

\begin{proof}
By Lemma \ref{lemma:matched vertices} we know that vertices in $T_I$ and $T_O$ can only be matched to a vertex in $T_I$ and $T_O$, respectively. There are no unmatched vertices in $T_C$, so vertices in $T_C\cap V$ are perfectly matched to vertices in $T_C\cap F$. It follows from Hall's marriage theorem that $\BG_C$ is self-contained \citep{Hall1986}.
\end{proof}

The following lemma restricts edges that can be present between the incomplete, complete and overcomplete sets. This shows that clusters of the causal ordering graph that are in the overcomplete set are never descendants of clusters in the incomplete or complete set. Similarly, it also shows that clusters in the incomplete set are never ancestors of the complete or overcomplete sets. Lemma \ref{lemma:impossible edges} is then used to prove Proposition \ref{prop:coarse decomposition unique}.

\impossibleedges*

\begin{proof}
Suppose that there is an edge $e=(v-f)$ between a vertex $v\in T_I\cap V$ to a vertex $f\in(T_C\cup T_O)\cap F$. By Lemma \ref{lemma:matched vertices} the edge is not in the maximum matching. Note that there is an alternating path from an unmatched vertex in $T_I\cap V$ to $v$ that starts with an unmatched edge and ends with a matched edge. By adding the edge $(v-f)$, we obtain again an alternating path so that $f\in T_I$. This is a contradiction, and hence there is no edge between $(v-f)$. The second part of the lemma follows by symmetry.
\end{proof}

\coarsedecompositionunique*

\begin{proof}
Let $M$ be an arbitrary matching and let $\mathrm{CD}(\BG,M)=\tuple{T_I,T_C,T_O}$. Note that all vertices in $(T_I\cap V)\setminus U_V$ are $M$-matched to vertices in $T_I \cap F$ (by construction and Lemma~\ref{lemma:matched vertices}). Also, all vertices in $(T_O\cap F)\setminus U_F$ are $M$-matched with vertices in $T_O\cap V$. Finally, all vertices in $T_C\cap V$ are $M$-matched with vertices in $T_C\cap F$ and vice versa by Lemma~\ref{lemma:complete graph self-contained}. By Lemma~\ref{lemma:impossible edges} we have $\adj{\BG}{T_I\cap V} = T_I\cap F$ and $\adj{\BG}{T_O\cap F} = T_O\cap V$, so \emph{any} matching for $\BG$ can only match vertices in $T_I\cap V$ with vertices in $T_I\cap F$ and vertices in $T_O\cap F$ with vertices in $T_O\cap V$.


For the sake of contradiction, assume that there exists a maximum matching $M'$ that matches a vertex in $T_I\cap F$ with a vertex in $(T_C \cup T_O)\cap V$. Write:
\begin{align*}
M_V = \{v \in V : \exists f \in F : v-f \in M\}, \qquad M_V' = \{v \in V : \exists f \in F : v-f \in M'\},\\
M_F = \{f \in F : \exists v \in V : v-f \in M\}, \qquad M_F' = \{f \in F : \exists v \in V : v-f \in M'\}.
\end{align*}
Note that the number of edges in matching $M'$ is bounded by
\begin{equation*}\begin{split}
|M'| &= |M_V'| \\
     &= |M_V' \cap T_I| + |M_V' \cap T_C| + |M_V' \cap T_O| \\
     &\le (|F \cap T_I| - 1) + |V \cap T_C| + |V \cap T_O| \\
     &= (|M_V \cap T_I| - 1) + |M_V \cap T_C| + |M_V \cap T_O| \\
     &= |M_V| - 1 = |M| - 1,
\end{split}\end{equation*}
where we used that (i) vertices in $T_I \cap V$ can only be matched with vertices in $T_I \cap F$, (ii) all nodes in $T_I \cap F$ are $M$-matched with vertices in $M_V \cap T_I$, (iii) all variable vertices in $T_C$ are $M$-matched, and (iv) all vertices in $T_O \cap V$ are $M$-matched. This contradicts the assumption that $M'$ is a maximum matching.

In a similar way, one obtains a contradiction when assuming the existence of a maximum matching $M''$ that matches a vertex in $T_O\cap V$ with a vertex in $(T_I\cup T_C)\cap F$. Hence any maximum matching of $\BG$ must match all vertices in $T_I\cap F$ with vertices in $T_I \cap V$, and all vertices in $T_O\cap V$ with vertices in $T_O\cap F$. We conclude that $T_O$ and $T_I$ do not depend on the choice of maximum matching. By definition $T_C$ is uniquely determined by $T_O$ and $T_I$. Therefore the coarse decomposition is independent of the choice of maximum matching.


\end{proof}

\subsection{Markov property via $d$-separation}
\label{app:markov property proof}

In this section we prove Theorem \ref{thm:markov property} below.

\markovproperty*

\begin{proof}
Let $v\in (T_C\cup T_O) \cap (V\setminus W)$ be arbitrary and define $S_V=\mathrm{cl}(v)\cap V$ and $S_F=\mathrm{cl}(v)\cap F$. First, we will show that $V(S_F)\setminus S_V = \mathrm{pa}_{\mathrm{MO}(\BG)}(v)$. The following equivalences hold for $x \in V$:
\begin{align*}
x \in V(S_F)\setminus S_V & \iff x\in \mathrm{adj}_{\BG}(S_F)\setminus S_V
\tag*{\small(by Definition \ref{def:system of constraints})}\\
& \iff (x\to \mathrm{cl}(v)) \text{ in } \mathrm{CO}(\BG)
\tag*{\small(by definition of Algorithm \ref{alg:causal ordering coarse decomposition})} \\
& \iff (x\to v) \text{ in } D(\mathrm{CO}(\BG))
\tag*{\small(by Definition \ref{def:declustering})}\\
& \iff (x\to v) \text{ in } D(\mathrm{CO}(\BG))_{\mathrm{mar}(F)} \\
& \iff (x\to v) \text{ in } \mathrm{MO}(\BG)
\tag*{\small(by Definition \ref{def:declustering})}\\
& \iff x \in \pa{\mathrm{MO}(\BG)}{v}.
\end{align*}

By assumption, the system of constraints is maximally uniquely solvable w.r.t.\ $\mathrm{CO}(\BG)$. Note that $S_V\subseteq V(S_F)$. Hence, there exist measurable functions $g_i:\B{\X}_{\mathrm{pa}_{\mathrm{MO}(\BG)}(v)}\to \X_i$ for all $i\in S_V$ such that
$\P_{\B{X}_W}$-a.s., for all $\B{x}_{V(S_F)\setminus W}\in \B{\X}_{V(S_F)\setminus W}$:
\begin{align*}
  \forall\, f\in S_F: \,\, \phi_f(\B{x}_{V(f)\setminus W},\B{X}_{V(f)\cap W})=c_f \;&\iff\; \\
  &\forall \, i\in S_V: \,\, x_i=g_i(\B{x}_{\mathrm{pa}_{\mathrm{MO}(\BG)}(v) \setminus W},\B{X}_{\mathrm{pa}_{\mathrm{MO}(\BG)}(v)\cap W}).
\end{align*}
  Since $v\in (T_C\cup T_O) \cap (V\setminus W)$ was chosen arbitrarily and $\B{X}^* = \B{h}(\B{X}_W)$ with $\B{h}$ a solution of $\mathcal{M}$, it follows that
\begin{align*}
X^*_v = g_v(\B{X}^*_{\mathrm{pa}_{\mathrm{MO}(\BG)} (v)}) \quad \text{$\P_{\B{X}_W}$-a.s.},
\end{align*}
for all $v\in (T_C\cup T_O) \cap (V\setminus W)$. The directed global Markov property was already shown to hold for pairs $(\G,\mathbb{P}_{\B{X}})$ where $\G$ is a DAG and $\B{X}$ is a solution to a set of structural equations with functional dependences corresponding to the DAG \citep{Pearl2000, Lauritzen1996}. Because the Markov ordering graphs $\mathrm{MO}(\BG)$ and $\mathrm{MO}_{\mathrm{CO}}(\BG)$ are acyclic by construction, and $\mathrm{MO}_{\mathrm{CO}}(\BG)$ is the graph corresponding to this set of structural equations, this completes the proof.
\end{proof}

\subsection{Causal ordering via perfect matchings}
\label{app:equivalence proof}

In this section we prove Theorem \ref{thm:equivalence} below.

\equivalence*

The following result gives a necessary and sufficient condition for the existence of a perfect matching for a bipartite graph and can be found in \citet{Hall1986}.

\begin{theorem}[Hall's Marriage Theorem]
\label{thm:hall}
Let $\BG=\tuple{V,F,E}$ be a bipartite graph with $|V|=|F|$. Then $\BG$ has a perfect matching if and only if $|F'|\le|\adj{\BG}{F}|$ for all $F'\subseteq F$.
\end{theorem}

From Hall's Marriage Theorem it trivially follows that a bipartite graph has a perfect matching if and only if it is self-contained.

\begin{corollary}
\label{cor:perfect matching self contained}
Let $\BG=\tuple{V,F,E}$ be a bipartite graph. Then $\BG$ has a perfect matching if and only if $\BG$ is self-contained.
\end{corollary}

\begin{proof}
If $\BG$ has a perfect matching then $|V|=|F|$. By Definition \ref{def:self-contained} we know that if $\BG$ is self-contained then $|V|=|F|$. Hence, the statement follows from Definition \ref{def:self-contained} and Theorem \ref{thm:hall}.
\end{proof}

The following technical lemma is used to prove Lemma \ref{lemma:output coincides}, which shows that the output of Algorithm \ref{alg:causal ordering minimal self-contained exo} coincides with that of Algorithm \ref{alg:causal ordering perfect matchings} in the case that the input of the algorithm is a self-contained bipartite graph and $W=\emptyset$.

\begin{lemma}
\label{lemma:sccs are msc}
Let $\M$ be a perfect matching for a self-contained bipartite graph $\BG=\tuple{V,F,E}$. Let $S_V^{(1)},\ldots,S_V^{(n)}$ be a topological ordering of the strongly connected components in the graph $\G(\BG,\M)_{\mathrm{mar}(F)}$. Let $\BG^{(i)}$ be the subgraph of $\BG$ induced by $\bigcup_{j=i}^{n}(S_V^{(j)}\cup \M(S_V^{(j)}))$. Then $\BG^{(i)}$ is self-contained and $\M(S_V^{(i)})$ is a minimal self-contained set in $\BG^{(i)}$.
\end{lemma}

\begin{proof}
We use the notation $\G^{(k)}:=\G(\BG^{(k)},\M^{(k)})$ and $S_F^{(k)} := \M^{(k)}(S_V^{(k)})$, where $\M^{(1)}=\M$ (we will define $\M^{(i)}$ with $i>1$ later). First we show that $S_F^{(1)}$ is self-contained in $\BG^{(1)}$. We proceed by proving that $S_F^{(1)}$ is minimal self-contained in $\BG^{(1)}$ and that $\BG^{(2)}$ is a self-contained bipartite graph. Finally, we consider how these arguments can be iterated to prove the lemma.

By definition of a perfect matching and the fact that $\BG^{(1)}=\BG$ is self-contained, we know that:
\begin{align}
|S_V^{(1)}| = |S_F^{(1)}| \leq |\adj{\BG^{(1)}}{S_F^{(1)}}|.
\end{align}
By definition of topological ordering and the orientation step in Definition \ref{def:orient, cluster, merge} we know that:
\begin{align*}
\adj{\BG^{(1)}}{S_F^{(1)}} \subseteq S_V^{(1)}.
\end{align*}
Together, these two inequalities show that $|S_F^{(1)}|=|\adj{\BG^{(1)}}{S_F^{(1)}}|$. Because $\BG^{(1)}$ is self-contained, the set $S_F^{(1)}$ satisfies both conditions of Definition \ref{def:self-contained}. We conclude that $S_F^{(1)}$ is self-contained in $\BG^{(1)}$.

Assume, for the sake of contradiction, that $S_F^{(1)}$ is not \emph{minimal} self-contained. Then there exists a non-empty strict subset $F'\subset S_F^{(1)}$ that is self-contained in $\BG^{(1)}$. First note that, by Definition \ref{def:self-contained}, we have that $|F'|=|\adj{\BG^{(1)}}{F'}|$ and $|S_V^{(1)}|=|S_F^{(1)}$| so that $S_V^{(1)}\setminus \adj{\BG^{(1)}}{F'}\neq \emptyset$ and $\adj{\BG^{(1)}}{F'}\neq \emptyset$. Furthermore, by Definition \ref{def:orient, cluster, merge} (orientation step), we must have that:
\begin{align}
\pa{\G^{(1)}}{\adj{\BG^{(1)}}{F'}} = \M^{(1)}(\adj{\BG^{(1)}}{F'}) =  F'.
\end{align}
Therefore there is no directed edge from any vertex in $F\setminus F'$ to any vertex in $\adj{\BG^{(1)}}{F'}$. Clearly, there can be no edge in $\G^{(1)}$ between any vertex $v\in S_V^{(1)}\setminus \adj{\BG^{(1)}}{F'}$ and any vertex $f'\in F'$ and hence
\begin{align}
\pa{\G^{(1)}}{S_V^{(1)}\setminus \adj{\BG^{(1)}}{F'}} = \M^{(1)}(S_V^{(1)}\setminus \adj{\BG^{(1)}}{F'}) = F\setminus F'.
\end{align}
Therefore, there can be no directed path from any $v\in S_V^{(1)}\setminus \adj{\BG^{(1)}}{F'}$ to any $f\in F'$ in $\G^{(1)}$. This contradicts the assumption that $S_V^{(1)}$ is a strongly connected component in $\G^{(1)}_{\mathrm{mar}(F)}$. We conclude that $S_F^{(1)}$ is minimal self-contained in $\BG^{(1)}$.

Clearly, the set $\M^{(2)} := \{(i-j) \in \M^{(1)}: i,j \notin S_V^{(1)}\cup S_F^{(1)}\}$ is a perfect matching for $\BG^{(2)}$. By Corollary \ref{cor:perfect matching self contained} we therefore know that $\BG^{(2)}$ is self-contained. Since $S_V^{(2)},\ldots, S_V^{(n)}$ is a topological ordering for the strongly connected components in $\G^{(2)}_{\mathrm{mar}(F)}$ the above argument can be repeated to show that $S_F^{(2)}$ is minimal self-contained in $\BG^{(2)}$. For arbitrary $i\in \{1,\ldots, n\}$ this entire argument can be iterated to show that $S_F^{(i)}$ is minimal self-contained in the self-contained bipartite graph $\BG^{(i)}$.
\end{proof}

\begin{lemma}
\label{lemma:output coincides}
Let $\M$ be an arbitrary perfect matching for a self-contained bipartite graph $\BG=\tuple{V,F,E}$. The directed cluster graph $\G_1=\tuple{\mathcal{V}_1,\mathcal{E}_1}$ that is obtained by application of Definition \ref{def:orient, cluster, merge} coincides with the output $\G_2=\tuple{\mathcal{V}_2,\mathcal{E}_2}$ of Algorithm \ref{alg:causal ordering minimal self-contained exo}.
\end{lemma}

\begin{proof}
Let $S^{(1)},\ldots, S^{(n)}$ be a topological ordering of the strongly connected components in $\G(\M,\BG)_{\mathrm{mar}(F)}$. By Definition \ref{def:orient, cluster, merge} the cluster set $\V_1$ consists of clusters $S^{(i)}\cup \M(S^{(i)})$ with $i\in\{1,\ldots,n\}$. By Lemma \ref{lemma:sccs are msc}, Algorithm \ref{alg:causal ordering minimal self-contained exo} can be run in such a way that the clusters $S^{(i)}\cup \M(S^{(i)})$ are added to $\V_2$ in the order specified by the topological ordering. By Theorem \ref{thm:uniqueness} the output of Algorithm \ref{alg:causal ordering minimal self-contained exo} is unique and therefore $\V_1 = \V_2$. By Definition \ref{def:orient, cluster, merge} the following equivalences hold for $C\in\V_1=\V_2$ and $v\in V\setminus C$:

\begin{align*}
(v\to C) \in \E_1 &\iff \exists w\in C \text{ s.t. } (v\to w) \text{ in } \G(\M,\BG) \\
&\iff \exists w\in C \text{ s.t. } (v-w) \in E \text{ and } (v-w)\notin \M \\
&\iff v \in \adj{\BG}{C\cap F} \setminus \M(C\cap F)\\
&\iff v \in \adj{\BG}{C\cap F} \setminus (C\cap V)\\
&\iff (v\to C) \in \E_2.
\end{align*}

Let $C\in \V_1=\V_2$ and $f\in F\cap(\adj{\BG}{C}\setminus C)$. By definition of Algorithm \ref{alg:causal ordering minimal self-contained exo} we know that $(f\to C)\notin \E_2$. Note that $\M(C\cap F)=C\cap V$. By Definition \ref{def:orient, cluster, merge} there is no edge $(f\to v)$ with $v\in C\cap V$ in $\G(\BG,\M)$ and hence by Definition we know that $(f\to C)\notin \E_2$. By construction, edges $(x\to C)$ with $x\in C$ are neither in $\E_1$ nor in $\E_2$. We conclude that $\E_1=\E_2$ and consequently $\G_1$ coincides with $\G_2$.
\end{proof}

Lemma \ref{lemma:output coincides} shows that the output of Algorithm \ref{alg:causal ordering minimal self-contained exo} coincides with the output of Algorithm \ref{alg:causal ordering perfect matchings} if the input is a self-contained bipartite graph. Otherwise, both Algorithm \ref{alg:causal ordering minimal self-contained exo} and \ref{alg:causal ordering perfect matchings} have an initialization that is determined by the specification of exogenous variables. The exogenous variables are placed into separate clusters and there are directed edges from each exogenous variable to the clusters of its adjacencies for both algorithms. The output of the two algorithms coincides for any valid input.

\subsection{Markov property via $\sigma$-separation}
\label{app:acyclification proof}

Here, we prove the following theorem.

\gdgmp*

The proof of this theorem relies on results by \citet{Forre2017}, who define the notion of an \emph{acyclic augmentation} for a class of graphical models that they call \emph{HEDGes}. They define the \emph{augmentation} of a HEDG as a directed graph where hyperedges are represented by vertices with additional edges. The acyclic augmentation of a HEDG is obtained by \emph{acyclification} of the edge set of it augmentation \citep{Forre2017}. The acyclification of a directed graph is given in Definition \ref{def:acyclification}.

\begin{definition}
\label{def:acyclification}
Let $\G=\tuple{V,E}$ be a directed graph. The \emph{acyclification} of $E$, denoted by $E^{\mathrm{acy}}$, has edges $(i\to j)\in E^{\mathrm{acy}}$ if and only if $i\notin\scc{\G}{j}$ and there exists $k\in\scc{\G}{j}$ such that $(i\to k)\in E$. 
\end{definition}

Lemma \ref{lemma:cluster decluster} shows that the clustering operation in Definition \ref{def:orient, cluster, merge} on directed graphs, followed by the declustering operation in Definition \ref{def:declustering}, results in the same directed graph as the one that is obtained by applying the acyclification operation to its edge set.

\begin{lemma}
\label{lemma:cluster decluster}
Let $\G=\tuple{V,E}$ be a directed graph. It holds that $\G^{\mathrm{acy}}=\tuple{V,E^{\mathrm{acy}}} = D(\mathrm{clust}(\G)))$.
\end{lemma}

\begin{proof}
This follows from Definitions \ref{def:declustering}, \ref{def:orient, cluster, merge}, and \ref{def:acyclification}.
\end{proof}

The following proposition shows that $\sigma$-separations in a directed graph coincide with $d$-separations in the graph that is obtained by clustering and subsequently declustering that directed graph.

\begin{proposition}
\label{prop:acyclification dsep}	
Let $\G=\tuple{V,E}$ be a directed graph with nodes $V$ and $\G^{\mathrm{acy}}=\tuple{V,E^{\mathrm{acy}}}$. Then for all subsets $A,B,C\subseteq V$:
\begin{align*}
A\sigsep{\G} B\given C \;\iff\; A\dsep{\G^{\mathrm{acy}}} B\given C \;\iff\; A\dsep{D(\mathrm{clust}(\G))} B\given C.
\end{align*}
\end{proposition}

\begin{proof}
The first equivalence is Proposition A.19 in \citet{Bongers2020}. The second equivalence follows directly from Lemma \ref{lemma:cluster decluster}.
\end{proof}

We now have all ingredients to finish the proof of Theorem~\ref{theo:generalized markov property}. First note that, since the subgraph of $\BG=\tuple{V,F,E}$ induced by $(V\cup F)\setminus W$ has a perfect matching, $\mathrm{CO}(\BG)=\tuple{\V,\E}$ is well-defined by Corollary \ref{cor:perfect matching self contained}. Let $S_V^{(1)},\ldots, S_V^{(n)}$ be the strongly connected components in $\G_{\mathrm{dir}}$, where $\G_{\mathrm{dir}}:= \G(\BG,\M)_{\mathrm{mar}(F)}$. By Lemma \ref{lemma:sccs are msc} and the definition of Algorithm \ref{alg:causal ordering minimal self-contained exo} we know that $\V$ consists of the clusters $S_V^{(i)}\cup\M(S_V^{(i)})$ with $i=1,\ldots,n$. Therefore, $\mathcal{M}$ is uniquely solvable with respect to $\mathrm{CO}(\BG)$. By Theorem \ref{thm:markov property} we have that for subsets $A,B,C\subseteq V\setminus W$:
\begin{align}
\label{eq:markov property}
A \dsep{\mathrm{MO}(\BG)} B \given C \;\implies\; \B{X}_A \ind{\P_{\B{X}}} \B{X}_B \given \B{X}_C.
\end{align}
By Proposition \ref{prop:acyclification dsep} we have that:
\begin{align}
\label{eq:acyclification}
A \sigsep{\G_{\mathrm{dir}}} B \given C \;\iff\; A \dsep{\G_{\mathrm{dir}}^{\mathrm{acy}}} B\given C
\;\iff\; A \dsep{D(\mathrm{clust}(\G_{\mathrm{dir}}))} B \given C. 
\end{align}
The desired result follows from implications \eqref{eq:markov property} and \eqref{eq:acyclification} when $D(\mathrm{clust}(\G_{\mathrm{dir}}))=\mathrm{MO}(\BG)$. Consider the cluster set $\V_{\mathrm{mar}(F)} = \{S\cap V: S\in \V\}$ and note that edges in $\mathrm{CO}(\BG)$ go from vertices in $V$ to clusters in $\V$. By Definition \ref{def:declustering} and \ref{def:orient, cluster, merge} we have that:
\begin{align}
D(\tuple{\V_{\mathrm{mar}(F)},\E}) = D(\tuple{\V,\E})_{\mathrm{mar}(F)} \quad \text{ and } \quad \mathrm{clust}(\G_{\mathrm{dir}}) = \tuple{\V_{\mathrm{mar}(F)},\E},
\end{align}
respectively. It follows that
\begin{align}
D(\mathrm{clust}(\G_{\mathrm{dir}})) = D(\mathrm{CO}(\BG))_{\mathrm{mar}(F)} = \mathrm{MO}(\BG).
\end{align}
Note that both d-separations and $\sigma$-separations are preserved under marginalization of exogenous vertices $W$ \citep{Forre2017, Bongers2020}. This finishes the proof.

\subsection{Effects of interventions}
\label{app:interventions proofs}

This section is devoted to the proofs of the results that were presented in Section \ref{sec:causal implications for sets of equations}.

\softinterventions*

\begin{proof}
The directed cluster graph $\mathrm{CO}(\BG)$ is acyclic by construction and therefore there exists a topological ordering of its clusters. When there is no directed path from $f$ to $v$ in $\mathrm{CO}(\BG)$ then there exists a topological ordering $V^{(1)},\ldots,V^{(n)}$ of the clusters such that $\mathrm{cl}(v)$ comes before $\mathrm{cl}(f)$. Note that clusters of vertices in the incomplete set $T_I$ are never ancestors of clusters in $T_C\cup T_O$ by Lemma \ref{lemma:impossible edges} (the proof of this lemma will be given in Appendix~\ref{appendix:coarse decomposition proofs}). Therefore there exists a topological ordering of clusters so that no cluster in $T_I$ precedes a cluster in $T_C\cup T_O$. By the assumption of unique solvability w.r.t.\ the clusters $T_C\cup T_O$ in $\mathrm{CO}(\BG)$ we know that the solution component for any variable $v\in V^{(i)}\subseteq T_C\cup T_O$ can be solved from the constraints in $V^{(i)}$ after plugging in the relevant solution components $\bigcup_{j=1}^{i-1} V^{(j)}$. By the solvability assumption, the solution components $X_v^*$ and $X_v'$ are equal almost surely.

By assumption, the variables in $\mathrm{cl}(f)$ can be solved from the constraints in $\mathrm{cl}(f)$. Hence, a soft intervention on a constraint in $\mathrm{cl}(f)$ may change the distribution of the solution components $\B{X}^*_{\mathrm{cl}(f)\cap V}$ that correspond to the variable vertices in $\mathrm{cl}(f)$. Suppose that there exists a sequence of clusters $V_1=\mathrm{cl}(f),V_2,\ldots, V_{k-1}, V_k=\mathrm{cl}(v)$ such that for all $V_i\in\{V_1,\ldots, V_{k-1}\}$ there is a vertex $z_i\in V_i$ such that $(z_i\to V_{i+1})$ in $\mathrm{CO}(\BG)$. In that case we know that $V_i\cup T_I = \emptyset$ for $i=1,\ldots, k$. By the assumption of maximal unique solvability w.r.t.\ $\mathrm{CO}(\BG)$ the solution components for the variables in $V_2,\ldots V_k$ may depend on the distribution of the unique solution components $\B{X}^*_{\mathrm{cl}(f)\cap V}$ that correspond to the variable vertices in $\mathrm{cl}(f)$. It follows that the solution $\B{X}^*_v$ may be different from that of $\B{X}'_v$, if there is a directed path from $f$ to $v$ in $\mathrm{CO}(\BG)$.
\end{proof}

\selfcontainedafterintervention*

\begin{proof}
By definition of Algorithm \ref{alg:causal ordering perfect matchings} we know that the subgraph of $\BG$ induced by $(V\cup F)\setminus W$ has a perfect matching $\M$ such that $\M(S_F)=S_V$. By definition of a perfect intervention on the bipartite graph we know that $\M$ is also a perfect matching for the subgraph of $\BG_{\mathrm{do}(S_F,S_V)}$ induced by $(V\cup F)\setminus W$. The result follows from Corollary \ref{cor:perfect matching self contained}.
\end{proof}

\commute*

\begin{proof}
Let $S_V=\tuple{s_v^1,\ldots, s_v^m}$ and $S_F=\tuple{s_f^1,\ldots s_f^m}$ denote the targeted variables and constraints. We consider the output $\mathrm{CO}(\BG)=\tuple{\V,\E}$ of the causal ordering algorithm. Suppose that the order in which clusters $V^{(i)}$ are added to $\V$ is given by
\begin{align}
V^{(1)}, \ldots, V^{(k)}=(S_F\cup S_V),\ldots, V^{(n)}.
\end{align}
Consider $\mathrm{CO}(\BG_{\mathrm{do}(S_F,S_V)})=\tuple{\V',\E'}$. It follows from Definition \ref{def:perfect intervention}, Lemma \ref{lemma:disjoint}, Lemma \ref{lemma:ind-selfcont}, and the definition of Algorithm \ref{alg:causal ordering coarse decomposition} (i.e.\ the extended causal ordering algorithm) that
\begin{align}
V^{(1)},\ldots, V^{(k-1)}, \{s_f^1,s_v^1\}, \ldots, \{s_f^m,s_v^m\}, V^{(k+1)},\ldots V^{(n)}
\end{align}
is an order in which clusters could be added to $\V'$. This shows that there are two differences between $\mathrm{CO}(\BG)=\tuple{\V,\E}$ and $\mathrm{CO}(\BG_{\mathrm{do}(S_F,S_V)})=\tuple{\V',\E'}$: first $(S_F\cup S_V)\in\V$ whereas $\{\{s_f^i, s_v^i\}:i=1,\ldots,m\}\subseteq \V'$ and second the clusters $(S_F\cup S_V)$ may have parents in $\mathrm{CO}(\BG)$ but the clusters $\{s_f^i, s_v^i\}$ (with $i\in\{1,\ldots,m\}$) have no parents in $\mathrm{CO}(\BG_{\mathrm{do}(S_F,S_V)})$. The result follows directly from Definition \ref{def:perfect intervention on a directed cluster graph}.	
\end{proof}

\perfectinterventions*

\begin{proof}
First note that $T_C\cup T_O = T'_C\cup T'_O$ by Definition \ref{def:perfect intervention}. Let $v\in S_V$. Since the variable vertices $S_V$ are targeted by the perfect intervention, we have that $X'_v=\xi_v$, which may be different from the solution component $X^*_v$. Consider $v\in V\setminus S_V$ and its cluster $\mathrm{cl}(v)$ in $\mathrm{CO}(\BG)$. Since the causal ordering graph is acyclic by construction, there exists a topological ordering $V^{(1)},\ldots, V^{(i)}=\mathrm{cl}(v),\ldots V^{(n)}$ of the clusters in $\mathrm{CO}(\BG)$ (where $n$ is the amount of clusters in $\mathrm{CO}(\BG)$) such that $V^{(j)}\prec \mathrm{cl}(v)$ implies that there is a directed path from some vertex in $V^{(j)}$ to the cluster $\mathrm{cl}(v)$ in $\mathrm{CO}(\BG)$. Note that clusters in $T_I$ are never ancestors of clusters in $T_C\cup T_O$ and that the ordering $V^{(1)},\ldots, V^{(n)}$ is such that no cluster in $T_I$ precedes a cluster in $T_C\cup T_O$. By assumption, the solution component $X^*_v$ can be solved from the constraints and variables in $V^{(i)}=\mathrm{cl}(v)$ by plugging in the solution for variables in $V^{(1)},\ldots, V^{(i-1)}$. Let $s_f^1,\ldots s_f^m$ and $s_v^1, \ldots s_v^m$ denote the ordered vertices in $S_F$ and $S_V$ respectively and suppose that $S_V\cup S_F=V^{(k)}$ for some $k\in\{1,\ldots,n\}$. By definition of a perfect intervention on a cluster we know that $V^{(1)},\ldots, V^{(k-1)}, \{s_f^1,s_v^1\}, \ldots, \{s_f^m,s_v^m\}, V^{(k+1)},\ldots V^{(n)}$ is a topological ordering of clusters in $\mathrm{CO}(\BG)_{\mathrm{do}(S_F,S_V)} = \mathrm{CO}(\BG_{\mathrm{do}(S_F,S_V)})$ (by Proposition \ref{prop:perfect interventions on clusters commute}). Furthermore, maximal unique solvability w.r.t.\ $\mathrm{CO}(\BG)$ implies maximal unique solvability w.r.t.\ $\mathrm{CO}(\BG_{\mathrm{do}(S_F,S_V)})$.

Suppose that $V^{(k)} \succ \mathrm{cl}(v)$ in the topological ordering for $\mathrm{CO}(\BG)$. By maximal unique solvability w.r.t.\ $\mathrm{CO}(\BG)_{\mathrm{do}(S_F,S_V)}$, $X'_v$ can be solved from the constraints and variables in $\mathrm{cl}(v)$ by plugging in the solution for variables in $V^{(1)},\ldots, V^{(i-1)}$. It follows that $X^*_v=X'_v$ almost surely and by construction of the topological ordering there is no directed path from any $x\in S_V$ to $v$ in $\mathrm{CO}(\BG)$. Suppose that $V^{(k)} \prec \mathrm{cl}(v)$ in the topological ordering for $\mathrm{CO}(\BG)$. By maximal unique solvability w.r.t.\ $\mathrm{CO}(\BG)_{\mathrm{do}(S_F,S_V)}$, we know that $X'_v$ can be solved from the constraints and variables in $V^{(i)}$ by plugging in the solution for variables in $V^{(1)},\ldots, V^{(k-1)}, \{s_f^1,s_v^1\}, \ldots, \{s_f^m,s_v^m\}, V^{(k+1)},\ldots V^{(i-1)}$. It follows that $X^*_v$ and $X'_v$ may have a different distribution, and by construction of the topological ordering there is a directed path from a vertex in $S_V$ to the cluster $\mathrm{cl}(v)$ in $\mathrm{CO}(\BG)$.
\end{proof}

\bibliography{library}

\end{document}